%% file: tofu-arxiv-version.tex
\documentclass[11pt]{article}
\pdfoutput=1

\usepackage[margin=1in]{geometry}
\usepackage{parskip}
\usepackage[numbers,compress,sort]{natbib}

\usepackage[utf8]{inputenc}
\usepackage[T1]{fontenc}
\usepackage[colorlinks = true]{hyperref}
\usepackage{url}
\usepackage{booktabs}
\usepackage{amsfonts}
\usepackage{nicefrac}
\usepackage{microtype}
\usepackage{xcolor}
\usepackage{fullpage}
\usepackage{enumitem}
\usepackage{varwidth}
\usepackage{etoolbox}
\usepackage{wrapfig}
\input{input-commands}

\title{Stochastic Linear Bandits with Partially Observed Actions\thanks{This work was primarily performed while Vineet Gattani was a PhD student at Arizona State University, and was partially supported by the National Science Foundation award CCF-2048223.}}
\vspace{1.5em}
\author{%
  \normalsize
  \begin{tabular*}{\textwidth}{@{}l@{\extracolsep{\fill}}r@{}}
  \textbf{Gautam Dasarathy} & \textsc{gautamd@asu.edu}\\
  \multicolumn{2}{@{}l@{}}{\textit{Arizona State University}}\\[8pt]
  \textbf{Vineet Gattani} & \textsc{gattanivineet@gmail.com}\\
  \multicolumn{2}{@{}l@{}}{\textit{GE Vernova}}\\[8pt]
  \textbf{Lalit Jain} & \textsc{lalitkumarj@gmail.com}\\
  \multicolumn{2}{@{}l@{}}{\textit{Google}}
  \end{tabular*}%
}
\date{}

\newcommand{\algname}{\texttt{TOFU-POV}}
\begin{document}

\maketitle

\begin{abstract}
The stochastic linear bandit, where actions are represented as vectors and rewards are linear, is a central paradigm for sequential decision making. We study a partially observed variant of this problem in which the learning agent only sees a random subset of coordinates for each action. Such partial observability arises naturally in settings like recommendation and healthcare, where full action descriptions can be expensive or even impossible to obtain. In general, this makes sublinear regret information-theoretically impossible. However, we show that this barrier can be overcome when the action vectors have low intrinsic dimension. We propose an algorithm, \algname{}, that estimates the latent action subspace using the masked actions, imputes current actions using an epoch-wise frozen representation, and runs OFUL in the resulting low-dimensional coordinates. Our theory shows that \algname{} enjoys a \(\sqrt{T}\) regret that scales with the intrinsic action subspace dimension as opposed to the ambient dimension and quantifies the interaction between these quantities and the missingness, decision set size, and subspace conditioning. We also devise a rank-adaptive algorithm that does not require the knowledge of the intrinsic dimension. We complement these guarantees with a lower bound based on a novel product construction that separates usual reward-learning uncertainty from a missingness-dependent cost intrinsic to partial observation. Synthetic and real data experiments support our theory and show that \algname{} can substantially improve upon natural baselines in this challenging problem.
\end{abstract}

\input{intro.tex}

\input{related.tex}

\input{setup.tex}

\input{algorithm.tex}

\input{subspace-est.tex}

\input{model-conf-set.tex}

\input{regret.tex}

\input{lower-bound.tex}

\input{experiments.tex}

\input{discussion.tex}

\bibliography{contextual_bandits}
\newpage
\onecolumn

\appendix
\input{appendices.tex}

\end{document}

%% file: input-commands.tex
\usepackage{amsfonts}
\usepackage[utf8]{inputenc}
\usepackage[T1]{fontenc}
\usepackage[most]{tcolorbox}
\usepackage{mlmodern}
\usepackage[colorlinks = true]{hyperref}
\usepackage{nicefrac}
\usepackage{microtype}
\usepackage{booktabs}
\usepackage{hyperref}
\usepackage{algorithm}[1]
\usepackage{algpseudocode}
\usepackage{bbm}
\usepackage{amsmath,amssymb}

\usepackage{subcaption}
\newtheorem{theorem}{Theorem}
\newtheorem{lemma}{Lemma}

\newtheorem{corollary}{Corollary}
\newtheorem{definition}{Definition}
\newtheorem{assumption}{Assumption}

\newenvironment{proof}{{\bf Proof:}}{\hfill\rule{2mm}{2mm}}
\definecolor{myblue}{RGB}{80,80,160}
\definecolor{mygreen}{RGB}{80,160,80}
\definecolor{light-gray}{gray}{0.4}

\newcommand{\E}{\mathbb{E}}
\newcommand{\R}{\mathbb{R}}

\newcommand{\U}{\mathbf{U}}
\newcommand{\hU}{\hat{\mathbf{U}}}
\newcommand{\Pb}{\mathbf{P}}

\newcommand{\hP}{\hat{\mathbf{P}}}

\newcommand{\htheta}{\hat{\theta}}
\newcommand{\dX}{\dot{X}}
\newcommand{\hX}{\hat{X}}

\newcommand{\dD}{\dot{D}}

%% file: intro.tex
\section{Introduction}
\label{sec:introduction}
The stochastic linear bandit (SLB) is an important framework for sequential decision-making under uncertainty, where the expected reward of a vector-valued action is assumed to be a linear function of its features~\cite{abbasi2011improved, rusmevichientong2010linearly, lattimore2020bandit}. At each round $t$, the learning agent is presented with a decision set \(D_t = \{X_{t,1}, X_{t,2},\ldots \}\) and chooses an action {\(X_t \in D_t\), which results in a reward \(r_t = \langle X_t, \theta^\star\rangle + \eta_t\)}, where $\theta^\star \in \mathbb{R}^d$ is unknown and $\eta_t$ is conditionally zero-mean random noise. The goal here is to minimize the cumulative regret relative to an oracle that (a) knows $\theta^\star$ and, therefore, (b) chooses the action in \(D_t\) that maximizes the expected reward each round. A widely studied algorithm in this setting is OFUL~\cite{abbasi2011improved}, whose regret is known to be bounded above by $\widetilde{O}(d\sqrt{T})$, matching known lower bounds up to logarithmic factors. SLBs have found far-ranging applications in recommendation systems, advertising, and treatment allocation~\cite{li2010contextual, chu2011contextual, bastani2020online}.

In many modern applications, however, observing the full feature vector of each action is prohibitively expensive, infeasible, or impossible. In recommendation systems~\cite{li2010contextual}, due to privacy, storage, or computational constraints, only a sparse subset of item features may be accessible. Similarly, in scientific or healthcare applications, constraints on sensing or data collection may naturally lead to missing observations. This motivates a more challenging variant of the SLB problem, where at each round, the agent only observes a subset of entries from each action vector---what we refer to as \emph{partial observability}. Formally, for each action vector $X_t \in \mathbb{R}^d$, each coordinate is revealed independently with probability \(p\in (0,1]\). Without any further structure, reward-learning is information-theoretically impossible here: the agent cannot infer the full linear reward model, and pays a suboptimality price (i.e., regret) that is linear in time. Fortunately, many real-world problems exhibit low-dimensional structure. In recommendation systems, for example, user-item interactions are often governed by a small number of latent factors, implying that the true feature vectors lie near a low-dimensional subspace. Motivated by this, we study the SLB problem under \emph{limited observability} and \emph{low-rank structure}: we assume that the ideal (fully observed) action vectors $X \in \mathbb{R}^d$ lie in an unknown $m$-dimensional subspace, with $m$ possibly being significantly smaller than $d$.

There are two lines of work that our setting sits between, but is not covered by. Recent work on bandits with partially observable features~\cite{kim2025linear, park2022regret} studies different observation models in which the missing or latent components enter the reward problem in a prescribed way. In contrast, here the actions themselves lie in a low-dimensional subspace, and the learner sees (changing) random coordinate masks of every offered action. \cite{lale2019stochastic} studies the low-rank action-representation bandit problem where the action vectors are \emph{fully observed}, and pursue a projected-OFUL style analysis to get regret guarantees for their algorithm, PSLB. This is distinct from our approach which creates estimates of the latent action subspace and imputes every slate of actions using an epoch-wise frozen subspace representation. We then show that inside each epoch, we have an approximate linear bandit whose approximation error can be carefully controlled to get our regret guarantees.

There is also a technical reason to be careful about importing projected-OFUL analyses directly. If the estimated projection is updated over time and applied to past data, the projected noise process is not the standard predictable martingale transform used in OFUL, and the projected design no longer evolves by the usual rank-one updates. Moreover, lower-eigenvalue growth of the covariance of played arms cannot be inferred from the population covariance of offered arms when the played arms are selected by an OFU rule. This is why \algname{} freezes the representation within epochs: conditional on the epoch-start estimate, reward learning is an ordinary fixed-coordinate OFUL problem, while subspace and imputation errors are controlled separately. We give the detailed comparison in Appendix~\ref{sec:comparison-with-literature}.

We begin by outlining the {\bf main contributions} of our paper.
\begin{itemize}[leftmargin=1.35em]
    \item {{\bf Problem formulation.}} We formulate the partially observed low-rank SLB problem, where each offered action lies in an unknown \(m\)-dimensional subspace of \(\mathbb R^d\), but the learner observes only {a random subset of its features (determined by independent coordinate masks)}.

    \item {{\bf Algorithm.}} We introduce \algname{}, an epoch-wise algorithm that estimates the latent subspace from all offered masked actions, freezes the representation within each epoch, imputes the current action set, and runs OFUL in the resulting low-dimensional coordinates.

    \item {{\bf Regret guarantee.}} When \(m\) is known, {under a standard incoherence condition and the standard normalization \(B_X, \|\theta^\star\| = O(1)\),} we show that, with high probability, the regret of \algname{} scales as
    {\[ \widetilde{O}\!\left(
        \frac{\kappa^2 m}{p^4K}+
        m\sqrt T + \frac{\kappa m\sqrt T}{p^2\sqrt K}\right),
    \]
    where \(\kappa\) measures the conditioning of the action covariance matrix (see Section~\ref{sec:problem_setup}).} Equivalently, suppressing the additive burn-in cost that does not grow \(T\), the effective scaling is \(\widetilde O(m\sqrt T+\kappa m\sqrt T/(p^2\sqrt K))\). This bound replaces the ambient dimension \(d\) by the usually much smaller \(m\), while exposing how missingness \(p\), decision set size \(K\), and subspace conditioning \(\kappa\) affect the regret. To achieve this, our epoch-based argument isolates the statistical difficulty of OFUL from a controlled misspecification term. As we show in Appendix~\ref{app:pslb-comparison}, this resolves key obstacles in projected low-rank bandit arguments.

    \item {{\bf Rank adaptivity.}} {We also devise a rank-adaptive variant of \algname{} that requires no knowledge of \(m\) (nor an upper bound on it) and enjoys the same regret scaling; the rank-identification cost is absorbed, up to constants, into the imputation burn-in.}

    \item {{\bf Lower bound.}} We complement the upper bound with a lower bound via a novel argument that separates ordinary reward uncertainty from a missingness-discovery cost.

    \item {{\bf Experiments.} We evaluate \algname{} and its natural variants on synthetic and real data. The results corroborate our theory, with the largest gains over the baselines appearing under heavier missingness.}
\end{itemize}

%% file: related.tex
\section{Related Work}
\label{sec:related_work}

\textbf{Structured Linear Bandits.}
Linear stochastic bandits are usually analyzed in the ambient feature dimension
\(d\)~\cite{abbasi2011improved,lattimore2020bandit}. A large literature reduces
this dependence by imposing structure, including sparsity of \(\theta^\star\)
\cite{abbasiyadkori2012online,carpentier2012bandit,kwon2017sparse,jang2022popart}
and low-rank matrix or bilinear reward structure
\cite{jun2019bilinear,lu2021low,jang2021bilinear,kang2022efficient}. These
exploit parameter structure, whereas our setting exploits low-rank
structure in the action vectors themselves. {Other kinds of structure on the action space itself, e.g., spectral structure, are likewise known to aid regret and pure-exploration performance~\citep{valko2014spectral,lejeune2020thresholding,thaker2022maximizing}.}

\textbf{Low-Rank Action Representations.}
The closest predecessor is \citet{lale2019stochastic}, who study linear bandits
with fully observed, approximately low-rank action representations. Our paradigm
adds coordinate-wise missingness, but the distinction goes beyond just modeling:
a direct projected-OFUL proof with a continually updated projection does not
automatically inherit the standard self-normalized or elliptical-potential
arguments. Our epoch-wise construction avoids this by freezing the representation
inside each epoch and controlling the resulting representation bias carefully;
see Appendix~\ref{sec:comparison-with-literature} for more on this comparison.

\textbf{Bandits with Partially Observable Features.}
Recent work also studies bandits with partially observable or latent features
\cite{kim2025linear,park2022regret}, but under different observation models. In
these settings, the missing or latent components enter the reward problem through
a prescribed structure or a known sensing channel. In contrast, \algname{} assumes
that the action vectors themselves lie in an unknown low-rank subspace, while
each offered arm is revealed through random arm- and time-dependent coordinate
masks.

\textbf{Subspace Estimation with Missing Data.} Our work is also related to subspace estimation from incomplete observations, which has been widely studied, both in offline settings such as matrix completion~\cite{candes2009exact} and robust PCA~\cite{candes2011robust}, and in online or streaming settings through methods such as Oja's algorithm~\cite{oja1982simplified}, GROUSE~\cite{balzano2010high}, and PETRELS~\cite{chi2013petrels}. Recent high-dimensional analyses also give a unified view of several such online updates~\cite{liang2019exponential}. We use some similar techniques, but our focus is on controlling the bias induced by approximate subspace estimation and its interaction with the regret of an online algorithm.

%% file: setup.tex
\section{Problem Setup}
\label{sec:problem_setup}

{\bf Low-dimensional action vector model.}
For a natural number $n \in \mathbb{N}$, let $[n] := \{1,2,\ldots,n\}$. We consider a $d$-dimensional stochastic linear bandit over a horizon of $T$ rounds, and assume that the ideal action vectors lie in an unknown $m$-dimensional subspace of $\mathbb{R}^d$. Let $\U \in \mathbb{R}^{d \times m}$ be an orthonormal basis for this unknown subspace. At each round $t \in [T]$, an ideal decision set $D_t := \{X_{t,1}, \ldots, X_{t,K}\}$ is generated, where the ideal action vectors $X_{t,i}$ are drawn i.i.d.\ (across arms and rounds) from a fixed distribution supported on $\mathrm{span}(\U)$. We make the following two assumptions on this distribution.

\begin{assumption}[Bounded actions]
\label{ass:bounded_actions}
There is a known constant $B_X$ such that $\|X_{t,i}\|_2 \le B_X$ almost surely for all $t \in [T]$ and $i \in [K]$.
\end{assumption}

\begin{assumption}[Action covariance rank]
\label{ass:spectrum}
The covariance matrix $\Sigma := \mathbb{E}[X_{t,i} X_{t,i}^\top]$ has rank $m$, and its nonzero eigenvalues satisfy
\[
\bar\lambda \;\ge\; \lambda_1 \;\ge\; \cdots \;\ge\; \lambda_m \;>\; 0, 
\]
for some constant \(\bar\lambda >0\). 
\end{assumption}

As is standard in the bandit literature~\cite{chu2011contextual,agrawal2013thompson,abeille2017linear,krishnamurthy2018semiparametric, abbasi2011improved}, we will suppose that \(B_X = O(1)\) and is known by the algorithm. 
Assumption~\ref{ass:spectrum} says that the action distribution excites every direction of the latent subspace, and provides an envelope on its energy. We note that \(\bar\lambda\) need not be known: since \(\lambda_1 \le \mathbb{E}\|X_{t,i}\|_2^2 \le B_X^2\), one may always take \(\bar\lambda = B_X^2\), and a sharper envelope only tightens our bounds.

{\bf Example.} A natural setting satisfying these assumptions is the following loading-matrix model
\begin{align}
\label{eq:generative_model}
    X_{t,i} = \U \mathbf{\Lambda} Z_{t,i},
\end{align}
where $\mathbf{\Lambda} \in \mathbb{R}^{m \times m}$ is a fixed diagonal loading matrix and the latent vectors $Z_{t,i} \in \mathbb{R}^m$ are i.i.d.\ with $\mathbb{E}[Z_{t,i}] = 0$, $\mathbb{E}[Z_{t,i} Z_{t,i}^\top] \succeq \nu I_m$ for some $\nu > 0$, and $\|Z_{t,i}\|_\infty \leq B_Z$ almost surely. This model satisfies Assumption~\ref{ass:bounded_actions} with $B_X = \max_{j \in [m]}|\mathbf{\Lambda}_{jj}| \sqrt{m}\, B_Z$ and Assumption~\ref{ass:spectrum} with $\lambda_m \ge \nu \min_{j \in [m]} \mathbf{\Lambda}_{jj}^2$.

{\bf Conditioning of the action covariance.} It is important to note that the smallest eigenvalue of the action covariance matrix cannot be dimension-free: since $m\lambda_m \le \operatorname{tr}(\Sigma) = \mathbb{E}\|X_{t,i}\|_2^2 \le B_X^2$, we necessarily have $\lambda_m \le B_X^2/m$. As we will see below, our regret bounds depend on the spectrum through the quantity
\[
\kappa := \frac{B_X\sqrt{\bar\lambda}}{\lambda_m\sqrt{m}},
\]
which measures the conditioning of the action covariance. Since $\kappa^2 = \left(\tfrac{B_X^2}{m\lambda_m}\right)\left(\tfrac{\bar\lambda}{\lambda_m}\right)$, we always have $\kappa \ge 1$. In the well-conditioned regime $\lambda_m \asymp \bar\lambda \asymp B_X^2/m$, i.e., when the spectrum is flat and the norm bound is tight on average, we have \(\kappa = \Theta(1)\)

To reason about subspace recovery from partial observations, an important property is the  \emph{incoherence} of the subspace (with respect to the canonical basis). A coherent subspace may be extremely concentrated on a small set of coordinates and missing these coordinates would make learning impossible. Incoherence assumptions are standard (see e.g.,~\cite{candes2009exact}), and make restrictions on the incoherence parameter, which is defined as follows.

\begin{definition}[Incoherence]
\label{def:incoherence}
Let $\U \in \mathbb{R}^{d \times m}$ be a matrix with orthonormal columns. The incoherence parameter of $\U$ is defined as
\(
\mu := \sqrt{\frac{d}{m}} \cdot \max_{j \in [d]} \| \U_{j,:} \|_2,
\)
where $\U_{j,:} \in \mathbb{R}^m$ denotes the $j$-th row of $\U$.

\end{definition}

A small $\mu$ means that the subspace energy is spread evenly across coordinates, which is precisely the regime where missing observations still carry useful information about the latent subspace.

{\bf Missingness and observation model.}
The learner does not observe the ideal action vectors directly. Instead, for each round $t \in [T]$, arm $i \in [K]$, and coordinate $j \in [d]$, we draw an observation indicator $s_{t,i}^{(j)} \stackrel{\text{i.i.d.}}{\sim} \mathrm{Bernoulli}(p)$, independently across rounds, arms, and coordinates, and independently of the ideal action vectors. Writing $S_{t,i} := (s_{t,i}^{(1)}, \ldots, s_{t,i}^{(d)}) \in \{0,1\}^d$, the partially observed action vector is defined as
\begin{equation}
\label{eq:missing_model_bernoulli}
    \dX_{t,i} = S_{t,i} \odot X_{t,i},
\end{equation}
where $\odot$ denotes entrywise multiplication. Equivalently, $\dX_{t,i}^{(j)} = X_{t,i}^{(j)} s_{t,i}^{(j)}$ for each coordinate $j$. We suppose that the learner observes the partially observed decision set $\dD_t := \{\dX_{t,1}, \ldots, \dX_{t,K}\}$, and based on this set, it selects an action $\dX_t \in \dD_t$ and the resulting reward is
\(r_t = X_t^\top \theta^\star + \eta_t.\)
Here $X_t$ is the corresponding {\em ideal action}, $\theta^\star \in \text{span}(\U)$ is an unknown parameter vector with $\|\theta^\star\|_2 \le S$ for a known constant $S$, which, as with $B_X$, we treat as $O(1)$, and $\eta_t$ is conditionally $R$-sub-Gaussian:
$\mathbb{E}[\exp(\lambda \eta_t) \mid \mathcal{F}_{t-1}] \leq \exp\left(\frac{\lambda^2 R^2}{2}\right)$, for all $\lambda \in \mathbb{R}.$

Indeed our goal is to design a learning algorithm with {\em small cumulative regret},
$$
R_T := \sum_{t=1}^T (X_t^\star - X_t)^\top \theta^\star,
$$
where $X_t^\star = \arg\max_{X \in D_t} X^\top \theta^\star$ is the optimal ideal action at round $t$.

%% file: algorithm.tex
\section{Our Algorithm: \algname{}}
\label{sec:algorithm-overview}

In this section, we describe \textbf{\algname{}}
(\emph{Two-phase OFUL with Partially Observed Vectors}), our epoch-wise algorithm for stochastic linear bandits with partially observed action features. The algorithm takes as input a burn-in length \(t_b\), regularization \(\lambda\), subspace dimension \(m\), and a burn-in policy \(\pi_{\rm burn}\). The Algorithm~\ref{alg:bandits} display gives the formal pseudocode. During burn-in, the learner observes each masked decision set, plays according to a burn-in policy \(\pi_{\rm burn}\), and records the reward; in our experiments, \(\pi_{\rm burn}\) is taken to be standard OFUL where missing action vector coordinates are filled with zero. After burn-in, time is divided into epochs whose lengths double with \(\tau_0=t_b+1\), \(\tau_{e+1}=2\tau_e\). In what follows, we let \(\mathfrak{T}_e
    :=\{\tau_e,\tau_e+1,\ldots,\min(\tau_{e+1}-1,T)\}\) denote the time indices in the \(e\)-th epoch. 

At the start of epoch \(e\), the learner estimates a subspace basis \(\hU_e\) as the top-\(m\) eigenvectors of a corrected covariance estimator (Equation~\eqref{eq:unbiased_est_cov} below) built from only decision sets observed before \(\tau_e\), and then freezes this representation throughout the epoch. For each round \(t\in\mathfrak{T}_e\), it imputes the currently offered arms using \(\hU_e\), and then forms reduced features
\[
    z_{t,i}=\hU_e^\top \hat X_{t,i}\in\mathbb R^m. 
\]
The learner then runs an OFUL policy inside the epoch using rewards collected earlier in the same epoch. For notational ease, we let \(z_s := z_{s,i_s}\) denote the reduced feature vector of the arm played at round \(s\). Notice that this epoch structure makes the representation predictable relative to the rewards used by OFUL. Section~\ref{sec:subspace_estimation} controls the subspace estimation error, Section~\ref{sec:impute_entries} controls the imputation error, and Section~\ref{sec:confidence_set_analysis} converts these into epoch-wise confidence sets for the frozen-coordinate OFUL problem. Section~\ref{sec:regret_analysis_main} then combines all of these ingredients into our \(m\sqrt T\) regret guarantee, and Section~\ref{sec:unknown_rank} gives the rank-adaptive extension. It is instructive to compare the repeated epoch updates with a simpler one-shot two-phase strategy that estimates the subspace once and then freezes it for the rest of the horizon. In Appendix~\ref{app:two_phase_comparison} we show that this would make the regret scale like \(T^{2/3}\) instead of \(\sqrt{T}\).

\makeatletter
\algrenewcommand\ALG@beginalgorithmic{\small}
\algrenewcommand\algorithmicindent{1.5em}
\makeatother
\begin{algorithm}
\vspace{-1mm}
\caption{\algname{}: Two-phase OFUL with Partially Observed Vectors}
\label{alg:bandits}
\begin{algorithmic}[1]
   \State {\bf Inputs}: burn-in length $t_b$, regularization $\lambda$, subspace dimension $m$, burn-in policy $\pi_{\rm burn}$
   \For{$t = 1, \dots, t_b$}
        \State Receive partially observed decision set $\dD_t = \{\dX_{t,1}, \dots, \dX_{t,K}\}$
        \State Play $i_t=\pi_{\rm burn}(\dD_1,r_1,\ldots,\dD_{t-1},r_{t-1},\dD_t)$ and observe reward $r_t$
    \EndFor
	\State Set $\tau_0 \gets t_b + 1$ and $\tau_{e+1} \gets 2\tau_e, e = 0,1,2,\ldots$
   \For{epochs $e = 0,1,2,\dots$}

        \State Estimate subspace $\hU_e$ as the top-$m$ eigenvectors of $\dot\Sigma_{\tau_e - 1}$ in Equation~\eqref{eq:unbiased_est_cov} \hfill [Lemma~\ref{lemma:sub_space_lemma}]
        \For{$t \in \mathfrak{T}_e$}
            \State Receive partially observed decision set $\dD_t = \{\dX_{t,1}, \dots, \dX_{t,K}\}$
            \State Impute each arm using the frozen subspace to obtain $\hat X_{t,i}$ \hfill [Lemma~\ref{lemma: imputation_error}]
            \State Form features $z_{t,i} \gets \hU_e^\top \hat X_{t,i} \in \mathbb{R}^m$ \hfill [Sec.~\ref{sec:confidence_set_analysis}]
            \State $V_{e,t} \gets \lambda I_m + \sum_{\substack{s \in \mathfrak{T}_e\\ s<t}} z_s z_s^\top, \qquad \hat{\vartheta}_{e,t} \gets V_{e,t}^{-1}\sum_{\substack{s \in \mathfrak{T}_e\\ s<t}} z_s r_s$ \hfill [Thm.~\ref{thm:confidence_set_m_t_main}]
            \State Choose
            \(
            i_t \in \arg\max_{i \in [K]}
            \left\{
            \langle z_{t,i}, \hat{\vartheta}_{e,t} \rangle
            +
            \beta_{e,t}\|z_{t,i}\|_{V_{e,t}^{-1}}
            \;
            \right\}
            \)\hfill [Thm.~\ref{thm:confidence_set_m_t_main}]
            \State Play arm $i_t$ and observe reward $r_t$
        \EndFor
    \EndFor
\end{algorithmic}
\end{algorithm}
{\bf Practical implementation.}
Algorithm~\ref{alg:bandits} is the conservative version used in our regret analysis below. In {some of our experiments}, we make a natural data-reuse modification: at the start of each epoch \(e\), after computing \(\hU_e\), we re-impute every previously played arm using the frozen representation \(\hU_e\). We then form the corresponding coordinates \(z_s^{(e)}\), and initialize the epoch design matrix and response vector with all past reward observations (with newly imputed actions). Analyzing this ``warm-start''  variant requires handling the dependence between the design matrix and \(\hU_e\). We expect this can be done by a careful self-normalized confidence argument over a neighborhood of the true subspace, or by a sample-splitting construction; we leave a formal regret analysis to future work.

%% file: subspace-est.tex
\subsection{Estimating the Subspace from Partial Observations}
\label{sec:subspace_estimation}

If the action vectors were fully observed, the low-dimensional subspace could be estimated by applying PCA to their empirical covariance. Under partial observation, the naive covariance of the masked vectors is biased: unlike diagonal entries, off-diagonal entries are observed only when two coordinates are simultaneously revealed. We therefore use an inverse-probability correction that treats diagonal and off-diagonal entries differently. For $t\geq 1$, our corrected estimator, which uses all partially observed vectors offered in all previous rounds is given as follows:
\begin{align}
\label{eq:unbiased_est_cov}
\dot\Sigma_t
:=
\frac{1}{tK}\sum_{s=1}^{t}\sum_{i=1}^K
\left[
\frac{1}{p^2}\dX_{s,i}\dX_{s,i}^{\top}
+
\left(\frac{1}{p}-\frac{1}{p^2}\right)
\mathrm{diag}\bigl(\dX_{s,i}\dX_{s,i}^{\top}\bigr)
\right].
\end{align}
Note that this estimator includes arms that were not played as well and that the offered arms are i.i.d.\ and independent of the learner's policy.  These estimators appear in the matrix completion and missing-data covariance estimation literature (see e.g.,~\cite{candes2009exact, lounici2012covarainceestimation}). Let $\hU_t\in\R^{d\times m}$ denote the matrix of top-$m$ eigenvectors of $\dot\Sigma_t$, and let
\(
\hP_t:=\hU_t\hU_t^\top,
\quad
\Pb:=\U\U^\top
\)
be the estimated and true projection matrices. A standard way to measure subspace error is
\[
\mathrm{dist}(\hU_t,\U)
:=
\|(I-\hU_t\hU_t^\top)\U\|_2 = \left\| (I - \U\U^\top) \hU_t\right\|_2,
\]
which equals the sine of the largest principal angle between the estimated and true subspaces. {Since both projectors have rank \(m\), this quantity also coincides with the projector distance \(\|\hP_t-\Pb\|_2\), and we work with the latter in what follows.}

\begin{lemma}[Subspace recovery from partial observations]
\label{lemma:sub_space_lemma}
Suppose Assumptions~\ref{ass:bounded_actions} and~\ref{ass:spectrum} hold, and consider the Bernoulli missingness model in Equation~\eqref{eq:missing_model_bernoulli}. Then, with probability at least $1-\delta$, simultaneously for all $t\in[T]$,
\begin{align}
\label{eq:epsilon}
{\left\|\hP_t - \Pb\right\|_2
\leq
\epsilon_t,}
\qquad
\epsilon_t
:=
C_{\mathrm{sub}}
\frac{\kappa}{p}
\sqrt{\frac{m}{tK}\log\!\left(\frac{2dT}{\delta}\right)} .
\end{align}
Here $C_{\mathrm{sub}}>0$ is a universal numerical constant, and \(\kappa\) is the action subspace conditioning constant from Section~\ref{sec:problem_setup}. Equivalently, suppressing constants and logarithms, this is the rate
\(
\epsilon_t
=
\widetilde O\!\left(
\frac{\kappa}{p}
\sqrt{\frac{m}{tK}}
\right).
\)
\end{lemma}
{\bf Proof sketch.}
{We begin by establishing that \(\dot\Sigma_t\) is unbiased, and we then show that we can control \(\|\dot\Sigma_t-\Sigma\|_2\) using a matrix Bernstein bound with a variance proxy of order \(B_X^2\bar\lambda/p^2\). Since the \(m\)-th eigenvalue of \(\Sigma\) is \(\lambda_m\) and the \((m{+}1)\)-st is zero, we may then invoke the Davis--Kahan \(\sin\Theta\) theorem~\cite{davis1970rotation} to convert this covariance error into the subspace bound in~\eqref{eq:epsilon}, provided the covariance error is below \(\lambda_m/2\); for the (early) rounds where this fails, \(\epsilon_t\) exceeds a universal constant and the bound holds trivially since \(\|\hP_t-\Pb\|_2\le1\). Full details are in Appendix~\ref{sec:subspace-estimation-proof}.}

{\bf Burn-in period.}
The least-squares imputation step below requires the observed rows of the estimated subspace to be well-conditioned. Since the true observed Gram matrix concentrates around $pI_m$, it is sufficient for the projection error to be a small constant multiple of $p$; we use the convenient condition $\epsilon_t\leq p/32$. Solving Equation~\eqref{eq:epsilon} for this condition gives us the burn-in length
\begin{align}
\label{eq:t_b}
 t_b
:=
\left\lceil
C_b\,
\frac{\kappa^2 m}{p^4K}
\log\!\left(\frac{2dT}{\delta}\right)
\right\rceil,
\end{align}
for a universal constant $C_b>0$ (\(C_b = (32C_{\mathrm{sub}})^2\) suffices). That is, the algorithm needs to wait for this number of rounds before the imputation step starts helping. The $p^{-4}$ dependence here reflects our specific technique; improving this dependence is an interesting direction for future work. However, this term is an {\bf additive burn-in cost}. For fixed problem parameters, it does not scale with \(T\) except through logarithmic factors.

\subsection{Imputing the Partially Observed Actions and Controlling Errors}
\label{sec:impute_entries}
Once the subspace is estimated, the observed coordinates of the action vectors are used to impute the hidden ones using least squares. The key technical step here that allows us to control the quality of the imputation is to ensure that the observed rows of the estimated basis contains enough information about every latent direction. Under incoherence and sufficient observations (generated via a Bernoulli mask), we show that the true observed Gram matrix is well-conditioned at scale $p$. And, after burn-in, the estimated projector is close enough to the true projector that the estimated observed Gram matrix remains well-conditioned.

{An important subtlety (for our analysis) here is that the algorithm imputes with the \emph{frozen} epoch basis: within epoch \(e\), every arm is reconstructed using \(\hU_e\), which is computed from the \(\tau_e-1\) rounds preceding the epoch and is not updated as the epoch progresses. Fix an epoch \(e\), a round $t\in\mathfrak T_e$, and an arm $i$. Let $\Omega_{t,i}\subset[d]$ be the observed coordinates, and for a matrix $A\in\R^{d\times m}$ write $A_{\Omega_{t,i}}$ for the submatrix formed by selecting rows in $\Omega_{t,i}$. Given the frozen basis $\hU_e$, the imputed action $\hX_{t,i}$ keeps the observed entries unchanged and fills in the missing entries as follows {(since the epochs partition the horizon, the round index determines the epoch, and we leave the epoch implicit in \(\hX_{t,i}\); we use an epoch superscript only where an action is re-imputed under a different epoch's basis, as in the practical variant discussed after Algorithm~\ref{alg:bandits})}
\begin{align}
\label{eq:imputation_operator}
\big(\hX_{t,i}\big)^{(\Omega_{t,i}^c)}
:=
\hU_{e,\Omega_{t,i}^c}\hat a_{t,i}, \;\mbox{where } \hat a_{t,i}
:=
(\hU_{e,\Omega_{t,i}}^\top\hU_{e,\Omega_{t,i}})^{-1}
\hU_{e,\Omega_{t,i}}^\top X_{t,i}^{(\Omega_{t,i})}.
\end{align}}

\begin{lemma}[Uniform imputation error]
\label{lemma: imputation_error}
Assume the conditions of Lemma~\ref{lemma:sub_space_lemma}. Assume also that the true subspace is $\mu$-incoherent (as in Definition~\ref{def:incoherence}),
and that
\begin{align}
\label{eq:p_condition_imputation}
p
\geq
C_\mu\frac{\mu^2m}{d}
\log\!\left(\frac{mTK}{\delta}\right)
\end{align}
for a sufficiently large universal constant $C_\mu>0$. {Let $t_b$ be as in Equation~\eqref{eq:t_b}, and write \(\epsilon_e:=\epsilon_{\tau_e-1}\) for the subspace error bound of Lemma~\ref{lemma:sub_space_lemma} at the start of epoch \(e\). Then, with probability at least $1-3\delta$, for every epoch \(e\), every $t\in\mathfrak T_e$, and every $i\in[K]$,
\(\|X_{t,i}-\hX_{t,i}\|_2
\leq
\left(1+\frac{2}{p}\right)B_X\epsilon_e.\)}
\end{lemma}

{\bf Proof sketch.}
{We first show, via a matrix Chernoff bound, that the true observed Gram matrix satisfies $\U_{\Omega_{t,i}}^\top\U_{\Omega_{t,i}}\gtrsim p I_m$. Since \(\tau_e-1\ge t_b\), we can then transfer this conditioning to the frozen estimated Gram matrices. A technical detail here is that Lemma~\ref{lemma:sub_space_lemma} controls only the subspace distance, while the error analysis compares the matrices $\hU_e$ and $\U$ directly, and the latter is not determined by the former ($\hU_e$ is only defined up to a right rotation). However, the imputed vector is invariant under such rotations, so we may analyze the best-aligned basis, whose distance to $\U$ is {at most \(\epsilon_e\) (see Corollary~\ref{cor:aligned_basis}; the \(\sqrt2\) factor from basis alignment is absorbed into the constant \(C_{\mathrm{sub}}\)).} We then decompose the least-squares imputation error into the subspace perturbation plus the induced coefficient error, the latter amplified by the inverse Gram matrix (whose norm is \(2/p\)); combining the pieces gives the result. The complete proof is in Appendix~\ref{sec:imputation_error_analysis}.}

{\bf Epoch-wise representation event.}
\label{def:representation_event}
{For a target representation failure probability \(\delta_{\rm rep}\), let \(\mathcal E_{\rm rep}\) denote the event on which the subspace recovery guarantee in Lemma~\ref{lemma:sub_space_lemma} (simultaneously for all \(t\in[T]\)) and the epoch-wise imputation guarantee in Lemma~\ref{lemma: imputation_error} (simultaneously for all epochs \(e\), rounds \(t\in\mathfrak T_e\), and arms) hold, with both lemmas invoked at confidence parameter \(\delta_{\rm rep}/4\). By the two lemmas and a union bound, \(\mathbb P(\mathcal E_{\rm rep})\ge 1-\delta_{\rm rep}\).}

%% file: model-conf-set.tex
\subsection{Epoch-wise Surrogate Model and Estimation Error}
\label{sec:confidence_set_analysis}

We now fix an epoch \(e\) and analyze the bandit problem induced by the frozen representation \(\hU_e\). For each round \(t\in\mathfrak T_e\) and arm \(i\in[K]\), define
\[
z_{t,i}:=\hU_e^\top \hat X_{t,i}\in\mathbb R^m,
\qquad
\vartheta_e^\star:=\hU_e^\top\theta^\star\in\mathbb R^m,
\]
where \(\hat X_{t,i}\) is the epoch-\(e\) imputed action from Section~\ref{sec:impute_entries}. We also write
\[
\bar\mu_{t,i}:=\langle z_{t,i},\vartheta_e^\star\rangle,
\qquad
b_e:=S B_X\left(2+\frac{2}{p}\right)\epsilon_e ,
\]
{where \(\epsilon_e := \epsilon_{\tau_e-1}\) is the subspace error bound of Lemma~\ref{lemma:sub_space_lemma} at the start of epoch \(e\), as in Lemma~\ref{lemma: imputation_error}.}

The quantity \(\bar\mu_{t,i}\) is the surrogate mean inside the epoch. Indeed, it is only an approximation of the true mean \(\langle X_{t,i},\theta^\star\rangle\), since the learner works with an estimated subspace and imputed actions rather than the true action vectors. Our first lemma shows that, on the representation event, this approximation error is uniformly controlled by \(b_e\).

\begin{lemma}[Surrogate approximation inside epoch]
\label{lemma:surrogate-model}
On the event \(\mathcal E_{\rm rep}\) defined in Section~\ref{def:representation_event}, for every epoch \(e\), every round \(t\in\mathfrak T_e\), and every arm \(i\in[K]\),
\[
\left|\langle X_{t,i},\theta^\star\rangle-\bar\mu_{t,i}\right|
\le
b_e.
\]
\end{lemma}
The proof is in Appendix~\ref{app:epoch-wise-surrogate}. An immediate consequence is that, on \(\mathcal E_{\rm rep}\), the learner faces an ordinary \(m\)-dimensional linear bandit with bounded misspecification inside each epoch. Recalling that \(z_t := z_{t,i_t}\) denotes, for notational convenience, the reduced feature of the arm played at round \(t\in\mathfrak T_e\), the observed reward satisfies
\[
r_t=\langle z_t,\vartheta_e^\star\rangle+\xi_t+\eta_t,
\qquad
\xi_t:=\langle X_{t,i_t},\theta^\star\rangle-\bar\mu_{t,i_t},
\qquad
|\xi_t|\le b_e,
\]
where \(\eta_t\) is the reward noise from Section~\ref{sec:problem_setup}.

Our next goal is an OFUL-style confidence set for the estimation of the surrogate parameter \(\vartheta_e^\star\), centered at the epoch's estimator \(\hat\vartheta_{e,t}\) from Algorithm~\ref{alg:bandits}. Indeed, as in standard OFUL, this is the object that drives the optimistic arm selection. Notice that one may substitute the reward decomposition above into the definition of \(\hat\vartheta_{e,t}\) to get:
\begin{align}\label{eq:error-decomp}
V_{e,t}\left(\hat\vartheta_{e,t}-\vartheta_e^\star\right)
=
-\lambda\vartheta_e^\star
+\sum_{\substack{s\in\mathfrak T_e\\s<t}}z_s\xi_s
+\sum_{\substack{s\in\mathfrak T_e\\s<t}}z_s\eta_s ,
\end{align}
where, we recall that the design matrix \(V_{e,t}=\lambda I_m+\sum_{s\in\mathfrak T_e,\,s<t}z_sz_s^\top\). 
The first term is the usual (ridge regularization) bias. The second term \(\sum_{s<t}z_s\eta_s\) is handled exactly as in standard OFUL. It is worth noting that  since the feature map is frozen over the epoch, each \(z_s\) is predictable (measurable given the history before the reward \(r_s\) is revealed), and therefore the standard self-normalized inequality machinery applies. The third term \(\sum_{s<t}z_s\xi_s\) is specific to our setting, and the next lemma shows how we control it.

\begin{lemma}[Misspecification control inside an epoch]
\label{lem:epoch_bias_term}
Fix an epoch \(e\) and suppose that for all rounds \(s\in\mathfrak T_e\),
\(
r_s=\langle z_s,\vartheta_e^\star\rangle+\xi_s+\eta_s,
\;
|\xi_s|\le b_e,
\)
where \(\eta_s\) is conditionally \(R\)-sub-Gaussian. Then for every \(t\in\mathfrak T_e\),
\(
\left\|
\sum_{\substack{s\in\mathfrak T_e\\s<t}}z_s\xi_s
\right\|_{V_{e,t}^{-1}}
\le
b_e\sqrt{t-\tau_e},
\)
where \(V_{e,t}:=\lambda I_m+\sum_{s\in\mathfrak T_e,\,s<t}z_s z_s^\top\).
\end{lemma}

A naive triangle-inequality argument here would bound the vectors \(z_s\xi_s\) one at a time and we will have to pay a price through their leverage scores. While this is valid, it ignores that the same features \(z_s\) also build \(V_{e,t}\), and therefore loses an extra factor on the order of \(\sqrt m\). We highlight the suboptimality of this approach in Appendix~\ref{app:bias_control}. We instead pursue a sharper argument (detailed in Appendix~\ref{app:epoch-wise-surrogate}) that keeps the misspecification aggregated. 
With both error sources under control, we can now state the confidence set guarantee for the surrogate parameter \(\vartheta_e^\star\).

\begin{theorem}[Surrogate estimation error inside epoch]
\label{thm:confidence_set_m_t_main}
Fix an epoch \(e\) and a confidence level \(\delta_e\in(0,1)\). For
\(t\in\mathfrak T_e\), define
\[
V_{e,t}:=\lambda I_m+\sum_{\substack{s\in\mathfrak T_e\\s<t}}z_s z_s^\top,
\qquad
\hat\vartheta_{e,t}
:=
V_{e,t}^{-1}
\sum_{\substack{s\in\mathfrak T_e\\s<t}}z_s r_s .
\]
Then, on the representation event \(\mathcal E_{\rm rep}\), with probability
at least \(1-\delta_e\), the surrogate parameter \(\vartheta_e^\star\)
satisfies, simultaneously for all \(t\in\mathfrak T_e\),
\[
\|\hat\vartheta_{e,t}-\vartheta_e^\star\|_{V_{e,t}}
\le
\beta_{e,t} := \sqrt{\lambda}S
+
R\sqrt{
2\log\!\left(
\frac{\det(V_{e,t})^{1/2}}
{\det(\lambda I_m)^{1/2}\delta_e}
\right)}
+
b_e\sqrt{t-\tau_e}.
\]
\end{theorem}
We write \(\mathcal E_{{\rm conf},e}\) for the event in
Theorem~\ref{thm:confidence_set_m_t_main} on which the above bound holds
simultaneously for all \(t\in\mathfrak T_e\). Conditional on
\(\mathcal E_{\rm rep}\), this event has probability at least \(1-\delta_e\).

{\bf Proof sketch.} {We take \(V_{e,t}^{-1}\)-weighted norms in the error decomposition of \eqref{eq:error-decomp}, and we bound its three terms by the three terms of \(\beta_{e,t}\), respectively: the regularization term using \(\|\vartheta_e^\star\|_2\le S\), the stochastic term via the self-normalized inequality (as discussed above), and the misspecification term via Lemma~\ref{lem:epoch_bias_term}. The full proof is in Appendix~\ref{app:epoch-wise-surrogate}.}

Consequently, for every candidate arm with epoch-\(e\) surrogate feature \(z\), Theorem~\ref{thm:confidence_set_m_t_main} implies
\(
|\langle z,\hat\vartheta_{e,t}-\vartheta_e^\star\rangle|
\le
\beta_{e,t}\|z\|_{V_{e,t}^{-1}}.
\)
Thus the optimistic score used in Algorithm~\ref{alg:bandits} is the upper confidence bound induced by the frozen-epoch estimation guarantee. The next section converts this surrogate optimism into a true-regret bound by combining the surrogate approximation error with Theorem~\ref{thm:confidence_set_m_t_main}.

%% file: regret.tex
\section{Regret Analysis}
\label{sec:regret_analysis_main}

The preceding section reduces the post-burn-in analysis to a sequence of fixed-coordinate OFUL problems, one for each frozen epoch, with an additional approximation error from subspace estimation and imputation. The regret proof combines these two effects. Within an epoch, we get the OFUL-style control (in dimension \(m\)). Across epochs, the doubling schedule makes the shrinking representation error (which avoids an extra factor of \(\sqrt{m}\) as discussed above) summable at the \(\sqrt T\) scale. We now state the resulting bound.

\begin{theorem}[\algname{} regret]
\label{thm:main_regret}
Fix a confidence parameter \(\delta\in(0,1)\) and set \(\delta_{\rm rep}:=\delta/2\). Suppose Assumptions~\ref{ass:bounded_actions} and~\ref{ass:spectrum} hold, the true subspace is \(\mu\)-incoherent, and \(p\) satisfies Equation~\eqref{eq:p_condition_imputation} with \(\delta_{\rm rep}/4\) in place of \(\delta\). Choose the burn-in time \(t_b\) as in Equation~\eqref{eq:t_b} with \(\delta_{\rm rep}/4\) in place of \(\delta\), and set the regularization to \(\lambda:=4B_X^2\). Let \(E:=\max\{e:\mathfrak{T}_e\neq\emptyset\}\) denote the final epoch index (since \(\tau_{e+1}=2\tau_e\), \(E\le \lceil\log_2 T\rceil+1\)), and set \(\delta_e:=\delta/(2(E+1))\).

Then, with probability at least $1-\delta$, the epoch-wise algorithm satisfies
\[
    R_T
    \le
    \widetilde O\!\left(
        S B_X \frac{\kappa^2 m}{p^4K}
    \right)
    +
    \widetilde O\!\left((R+S B_X)\,m\sqrt T\right)
    +
    \widetilde O\!\left(
        S B_X
        \frac{\kappa m\sqrt T}{p^2\sqrt K}
    \right).
\]
\end{theorem}

{\bf Interpretation. }The first term in this bound is the {\em additive} burn-in cost needed for the imputation step to be stable. The second term  is the usual stochastic linear
bandit regret, which scales in \(m\) given how we set the intra-epoch linear bandit problem up. The third term is the cost of learning and using the representation from partially observed actions. Its \(1/\sqrt K\) dependence reflects that all
\(K\) displayed arms contribute to subspace estimation, while only one arm is
played for reward. Under the standard normalization \(B_X, S, R = O(1)\), the bound reveals the scaling \(\widetilde O\big(\kappa^2m/(p^4K) + m\sqrt T + \kappa m\sqrt T/(p^2\sqrt K)\big)\). In particular, for well-conditioned action covariances (\(\kappa = \Theta(1)\)), the regret bound for \algname{} scales as \(\widetilde O(m\sqrt T)\). 

{\bf Proof sketch.} The full proof is in Appendix~\ref{app:regret-analysis-full}. At a high level, the representation event lets us compare the true rewards to the frozen surrogate model, costing us a per-epoch bias \(b_e\). {Conditional on this event, the confidence set and optimistic action choice give us the usual OFUL regret term inside each epoch. Because the coordinates are frozen, we can control the accumulated uncertainty by the standard elliptical-potential argument in the fixed \(m\)-dimensional coordinates (Lemma~\ref{lem:epoch_potential} in Appendix~\ref{app:regret-analysis-full}); the choice \(\lambda = 4B_X^2\) suffices for this argument since the frozen features have norm at most \(2B_X\) after burn-in. The remaining cost is the surrogate approximation error. Since the epochs double and \(b_e\) decays like \(1/\sqrt{\tau_e}\), these approximation errors also sum at the \(\sqrt T\) scale. Combining the OFUL and approximation contributions over all epochs gives us the bound above.} In the next section, we show how to extend the method to unknown \(m\).

\section{Adaptivity to Unknown Subspace Dimension}
\label{sec:unknown_rank}
The epoch-wise algorithm described above assumes that the latent dimensionality \(m\), or equivalently, the rank of the action covariance \(\Sigma\) (Assumption~\ref{ass:spectrum}), is known. This assumption can be removed by implementing a thresholding procedure on the eigenspectrum of the estimated covariances. That is, the learner estimates the spectrum of the corrected covariance from all decision sets observed before \(\tau_e\), keeps the empirical eigenvectors whose eigenvalues clear a confidence threshold, and then runs the same frozen-coordinate OFUL procedure in the selected dimension. We call this variant {\tt Rank-Adaptive} \algname{}.

Concretely, let \(\hat\lambda_{t,1}\ge\cdots\ge\hat\lambda_{t,d}\) denote the eigenvalues of the corrected covariance estimator \(\dot\Sigma_t\) from Equation~\eqref{eq:unbiased_est_cov}. For a target failure probability \(\delta_{\rm rank}\), at the start of epoch \(e\) the learner first selects the dimension
{\[
\hat m_e := \#\left\{j\in[d] : \hat\lambda_{\tau_e-1,j}\ \ge\ 2\rho_{\tau_e-1}\right\},
\qquad
\rho_t := 2B_X\sqrt{\frac{\bar\lambda}{p^2 tK}\log\frac{2dT}{\delta_{\rm rank}}}
+\frac{2B_X^2}{p^2 tK}\log\frac{2dT}{\delta_{\rm rank}},
\]
and then uses the top \(\hat m_e\) eigenvectors of \(\dot\Sigma_{\tau_e-1}\) as the frozen basis \(\hU_e\) for epoch \(e\) (as in Algorithm~\ref{alg:bandits}, only decision sets observed before the epoch are used).} The threshold \(\rho_t\) is a high-probability upper bound on \(\|\dot\Sigma_t-\Sigma\|_2\), and is computable from \((B_X,p,K)\) alone (recall that one may always take \(\bar\lambda=B_X^2\)). In particular, the learner needs no knowledge of \(m\) or even an upper bound on it. 

\begin{theorem}[Rank-Adaptive \algname{} regret]
\label{thm:unknown_rank}
Suppose the assumptions of Theorem~\ref{thm:main_regret} hold, but \(m\) is unknown to the algorithm. Then, with probability at least \(1-\delta\) (under a natural allocation of \(\delta\) across various events), {\tt Rank-Adaptive}~\algname{} satisfies
\[
R_T
\le
\widetilde O\!\left(S B_X\, t_{\rm id}\right)
+
\widetilde O\!\left((R+S B_X)\,m\sqrt T\right)
+
\widetilde O\!\left(
S B_X \frac{\kappa m\sqrt T}{p^2\sqrt K}
\right),
\qquad
t_{\rm id} := \max\{t_b,\, t_{\rm rank}\},
\]
where \(t_{\rm rank} = \widetilde O\big(\kappa^2 m/(p^2K)\big)\) is the time at which the threshold separates the \(m\) signal eigenvalues from the null ones.
\end{theorem}

The full statement with explicit constants, and the proof, appear in Appendix~\ref{app:rank-adaptivity}. Informally, once the threshold is able to separate the signal from the noise floor, \(\hat m_e=m\). Indeed, after this, the algorithm coincides with the known-rank procedure. In our analysis, we charge all regret before this point at the worst-case rate, giving the first term. Moreover, somewhat unsurprisingly (given the complexity of the tasks in question), rank identification is faster than stable imputation. Indeed, comparing rates, \(t_{\rm rank}=\widetilde O\big(\kappa^2m/(p^2K)\big)\) while \(t_b=\widetilde O\big(\kappa^2m/(p^4K)\big)\), so \(t_{\rm id}=t_b\) up to constants. Therefore, the identification cost is absorbed into the burn-in term of Theorem~\ref{thm:main_regret} and we essentially get adaptivity to the unknown subspace dimension for free.

%% file: lower-bound.tex
\section{A Lower Bound: Bandit Learning and Missingness Discovery}
\label{sec:lower_bound_main}

The above analysis shows us the effect of the cost of reward learning and of missing coordinates on the final regret of \algname{}.  In this section, we show that both of these costs are essentially unavoidable. Towards this end, we provide a novel construction of a hard family of bandit instances which is parametrized by two independent (hidden) signs: one parametrizes noisy reward-learning, while the other selects one of two completions that cannot be distinguished from single-coordinate observations and is revealed only when a particular coordinate pair is co-observed. The regret for any policy \(\pi\) is denoted by \(R_T^\pi\), and is measured against the full-information oracle that sees the complete action set in each round, which only strengthens the lower bound.

\begin{theorem}[Lower bound]
\label{thm:iid_missingness_lower_bound_main}
Assume \(K\ge4\), \(p\in(0,1/2]\), and Gaussian reward noise with variance \(R^2\).  There exist universal constants \(c,c_0>0\) such that the following holds.  For any \(T\ge1\), \(S>0\), and action norm bound \(B_X>0\) satisfying
\(
    Kp^2+\frac{S^2B_X^2}{R^2}\le c_0,
\)
there is a four-instance family \(\mathfrak I=\{\mathcal I_{\nu,\sigma}:(\nu,\sigma)\in\{\pm1\}^2\}\) of Bernoulli-\(p\) missing-feature linear-bandit instances whose i.i.d. \(K\)-arm slate distributions are supported on rank-three subspaces of \(\mathbb R^4\) and satisfy \(\|X_{t,i}\|_2\le B_X\) and \(\|\theta^\star\|_2\le S\), such that every policy \(\pi\) obeys
\begin{equation}
\label{eq:main_lower_bound_optimized}
\begin{aligned}
\sup_{\mathcal I\in\mathfrak I}
\mathbb E_{\mathcal I}R_T^\pi
&\ge
c\min\{S B_XT,\,R\sqrt T\}
+
 cS B_X
\min\left\{
T,\frac{1}{Kp^2+S^2B_X^2/R^2}
\right\}.
\end{aligned}
\end{equation}
In particular, if \(R\ge S B_X/(p\sqrt K)\), then
\(\sup_{\mathcal I\in\mathfrak I}
\mathbb E_{\mathcal I}R_T^\pi
\ge
c\min\{S B_XT,\,R\sqrt T\}
+
 cS B_X
\min\left\{
T,\frac{1}{Kp^2}
\right\}.\)
\end{theorem}

The first term in Theorem~\ref{thm:iid_missingness_lower_bound_main} is the standard stochastic-bandit cost.  The second term is specific to missing features: we show that we can obfuscate the sign from any agent until some side-information arm reveals both relevant coordinates, an event occurring at rate \(\Theta(Kp^2)\) per round (where \(K\) arms are revealed), or until rewards identify the sign at rate \(\Theta(S^2B_X^2/R^2)\).  Thus, in the context-limited regime, missingness contributes an additional unavoidable cost of order \(S B_X\min\{T,1/(Kp^2)\}\).
We note that the hard instances lie within the class covered by our upper bound: the action vectors are bounded by \(B_X\) almost surely and the action  distribution is i.i.d.\ with a rank-three covariance, so Assumptions~\ref{ass:bounded_actions} and~\ref{ass:spectrum} hold. Comparing with Theorem~\ref{thm:main_regret}, both bounds exhibit a missingness cost that decreases with \(K\) and \(p\), but the scaling with respect to these parameters  does not match exactly. Identifying a lower bound construction that provides a sharp dependence on \(p\) and \(K\) as well remains an interesting open problem. The full gap-dependent statement and proof appear in Appendix~\ref{app:iid_missingness_lower_bound}.

%% file: experiments.tex
\section{Simulations}
\label{sec:simulations}

{\bf Experimental setup.}
We evaluate \algname{} with the practical choices described after Algorithm~\ref{alg:bandits}; all practical variants use zero-imputed OFUL during burn-in. In synthetic experiments, \textbf{TOFU} and \textbf{RA-TOFU} denote the known-rank and rank-adaptive versions. In real-feature experiments, \textbf{TOFU-FH} and \textbf{RA-TOFU-FH} denote full-history replay variants that re-impute and replay past selected rewards at each epoch start. Baselines are \textbf{ZF-OFUL}, ambient OFUL on zero-filled masked arms, and \textbf{ZF-PSLB}, the analogous zero-filled adaptation of PSLB~\citep{lale2019stochastic}.

The main {\bf synthetic experiments} use \(d=30\), true rank \(m^\star=3\), \(K=8\), horizon \(T=400\), burn-in \(t_b=30\), and Gaussian reward noise with standard deviation \(0.05\). We vary the observation probability over \(p\in\{0.8,0.6,0.4,0.3,0.2\}\). Figure~\ref{fig:synthetic_main_experiments} reports cumulative regret over time and final regret as a function of missingness. Means and standard errors over 20 random seeds are reported.

\begin{figure*}[t]
    \centering
    \begin{subfigure}[b]{0.9\textwidth}
        \includegraphics[width=\textwidth]{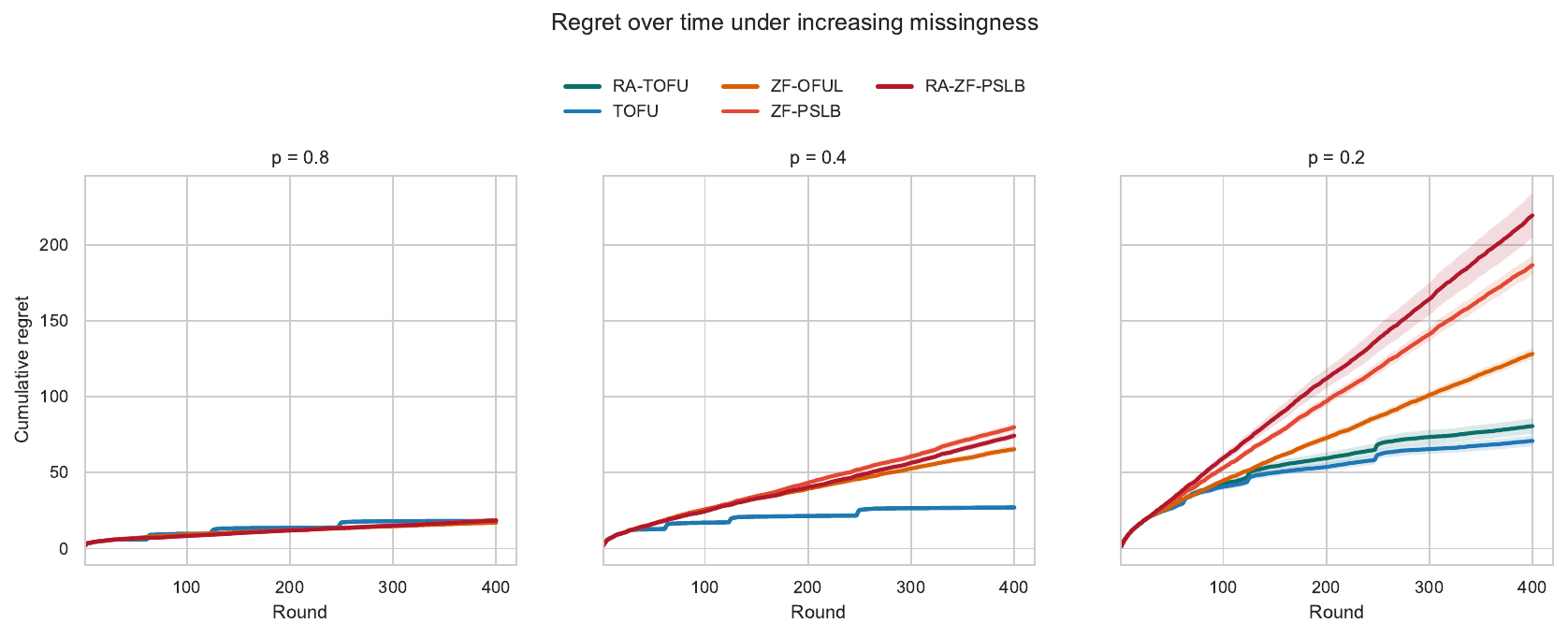}
        \caption{Regret over time.}
        \label{fig:synthetic_regret_over_time}
    \end{subfigure}
    \hfill
    \begin{subfigure}[b]{0.5\textwidth}
        \includegraphics[width=\textwidth]{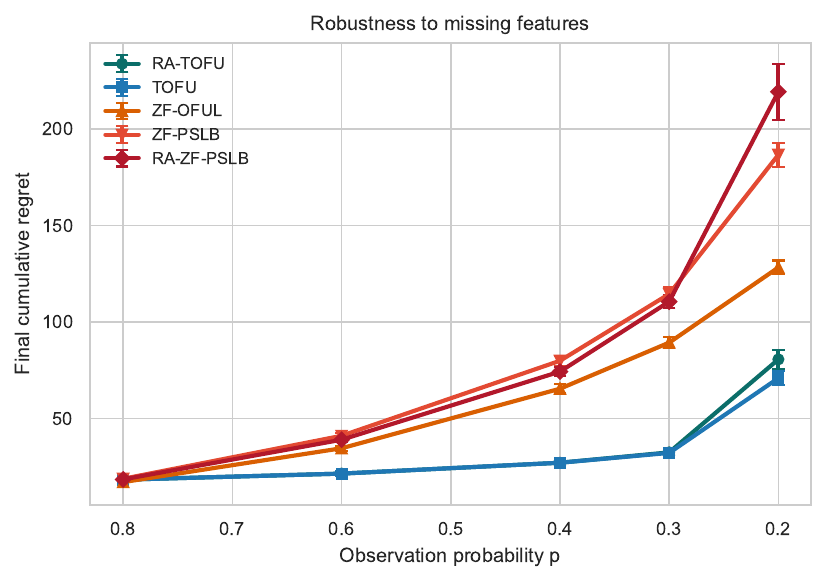}
        \caption{Final regret vs. \(p\).}
        \label{fig:synthetic_missingness_sweep}
    \end{subfigure}
    \caption{\footnotesize Synthetic missing-feature experiments. TOFU and RA-TOFU gain most when missingness is substantial: ZF-OFUL and ZF-PSLB degrade as \(p\) decreases, while the TOFU variants continue to exploit the corrected low-rank structure.}
    \label{fig:synthetic_main_experiments}
\end{figure*}

These results show that exploiting the low-rank structure becomes increasingly valuable when missingness is substantial. At high observation probability, ZF-OFUL is competitive because the ambient zero-filled representation retains enough information, and as \(p\) decreases, ambient learning becomes increasingly biased and sample-inefficient. TOFU and RA-TOFU remain close in moderate missingness and substantially outperform both ZF-OFUL and ZF-PSLB in sparse regimes. In the real-feature experiments below we use the full-history replay variants; Appendix~\ref{app:experimental_details} repeats this synthetic study with TOFU-FH and RA-TOFU-FH, which gives significant additional improvement.

{\bf MNIST product-context experiment. (Figure~\ref{fig:mnist_cnn_product_context})}
For MNIST~\citep{lecun1998gradient}, we train a small CNN with \(m=4\)-dimensional penultimate feature \(h(x)\) and class-head weights \(w_k\). The class-\(k\) arm is \(X_k(x)=h(x)\odot w_k\), preserving the classifier score through \(\langle X_k(x),\mathbf 1_m\rangle=w_k^\top h(x)\). We lift the ten product-context arms into \(\mathbb R^{100}\) by a fixed orthonormal map and mask coordinates. Rewards are classification rewards; \(T=5000\), \(t_b=500\), \(p\in\{0.7,0.5,0.3,0.2\}\), and ranks, thresholds, and confidence parameters are chosen on validation seeds disjoint from reporting seeds.

The MNIST product-context experiment gives a real-feature setting where the low-rank reward geometry is present by construction. TOFU-FH matches ZF-OFUL at mild missingness and increasingly outperforms it as features become sparse. ZF-PSLB is consistently worse, with the gap widening at lower observation probabilities.

{\bf Additional experiments. }
Appendix~\ref{app:experimental_details} contains several supporting experiments: real-feature synthetic tasks using optical digit covariates~\citep{alpaydin1998optical}, rank-recovery and fixed-rank misspecification diagnostics, warm-start comparisons, MNIST rank validation, and a 20 Newsgroups product-context experiment~\citep{lang1995newsweeder} with approximately low-rank features. The {\bf code and scripts} for reproducing the experimental results are available at \url{https://github.com/gautamdasarathy/tofu-pov-arxiv}.

%% file: discussion.tex
\newpage
\section{Discussion}
\label{sec:discussion}
\begin{wrapfigure}[20]{r}{0.45\textwidth}
    \vspace{-1pt}
    \centering
    \includegraphics[width=\linewidth]{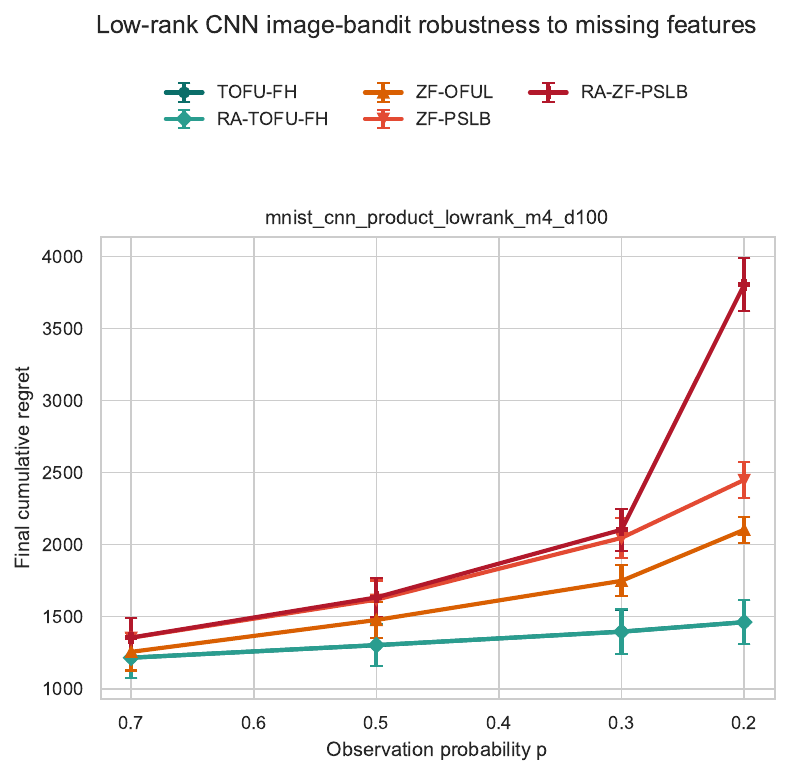}
    \caption{\footnotesize MNIST experiment. TOFU-FH wins or ties, and the gap over baselines grows as \(p\) decreases}.
    \label{fig:mnist_cnn_product_context}
    \vspace{-1pt}
\end{wrapfigure}
This paper studies contextual decision making when the learner must act from incomplete action descriptions. We focus on the case of i.i.d. action sets with low-rank structure and missing-at-random coordinates. Natural extensions include approximately low-rank action models, where the tail of the spectrum would create an additional approximation term, and missing-not-at-random observation patterns, where the missingness process itself may be informative or biased. Both settings would require separating representation error, imputation bias, and reward-learning uncertainty more carefully and are promising avenues for future work. The experiments partly probe beyond the theory by using public benchmark-derived covariates and approximately low-rank nuisance directions; these results suggest the method is not brittle to such deviations. Our epoch-wise freezing is deliberately conservative, and a truly sequential subspace-identification procedure would be more natural; this would require new concentration and potential arguments for learned, time-varying representations. {Finally, when the learner has some control over which coordinates are revealed, the problem acquires an active learning flavor and adaptively targeting informative coordinates (along the lines of techniques in~\citep{dasarathy2016active}) could sharpen both the subspace estimate and the burn-in cost. If the learner can instead choose the \emph{amount} of missingness across rounds and arms, by paying more for a higher observation probability \(p\), this becomes a multi-fidelity decision-making problem, and the techniques developed in the multi-fidelity bandit and optimization literature~\citep{kandasamy2016multifidelity,kandasamy2016gaussian,kandasamy2017multifidelity,kandasamy2019multifidelity} offer a natural starting point.}\\
On the lower-bound side, our construction shows that partial observation creates a genuine missingness-discovery cost, but it does not settle the sharp dependence on \(p\). In particular, it remains unclear whether the \(\sqrt T/p^2\)-type term in the upper bound is intrinsic or an artifact of worst-case subspace estimation and imputation control.

%% file: appendices.tex
\begin{center}{\bf \Large Stochastic Linear Bandits with Partially Observed Actions: Appendices}\end{center}

\section{Technical Comparison with PSLB and the Role of Epoch-wise Freezing}
\label{sec:comparison-with-literature}
\label{app:pslb-comparison}

The closest predecessor to our work is the projected stochastic linear bandit
framework of ~\cite{lale2019stochastic}. PSLB estimates a low-dimensional
subspace from the fully observed actions appearing in the decision sets and then
uses projected confidence sets inside an optimistic linear bandit rule. Our
setting differs at the modeling level because the ideal action vectors are
observed only through random coordinate masks, so the learner must recover the
latent subspace and impute the currently available actions. There is also a
proof-level distinction: a direct projected-OFUL analysis with a projection
matrix updated every round does not follow from the standard self-normalized and
elliptical-potential arguments. More specifically, the PSLB proof appears to rely on two steps that require additional justification: a self-normalized martingale argument for a retrospectively projected noise sum, and a projected-potential/minimum-eigenvalue argument for the covariance of OFU-selected actions under a time-varying projection.

{\bf Retrospective projections and self-normalization.}
The usual OFUL confidence analysis controls
\(
    S_t^{\rm OFUL}:=\sum_{s=1}^t x_s\eta_s,
\)
where $x_s$ is predictable before the reward noise $\eta_s$ is observed. In a
projected analysis with a time-varying projection, the analogous object is
\[
    S_t^{\rm proj}:=\sum_{s=1}^t \hP_t x_s\eta_s
    =\hP_t\sum_{s=1}^t x_s\eta_s .
\]
This is not the predictable martingale transform
$\sum_{s=1}^t\hP_sx_s\eta_s$: the projection at time $t$ is applied
retroactively to all previous noise terms. Indeed,
\[
    S_{t+1}^{\rm proj}-S_t^{\rm proj} = \hP_{t+1}x_{t+1}\eta_{t+1} + (\hP_{t+1}-\hP_t)\sum_{s=1}^t x_s\eta_s .
\]
 The second term reweights past noise and is not a standard martingale increment.
Thus the standard anytime OFUL
confidence proof, which relies on a supermartingale/Ville argument, cannot be
invoked simply by replacing each past feature with its projection under the
latest estimated subspace.

{\bf Projected potential and selected-action covariance.}
The usual elliptical-potential lemma is an algebraic, pathwise statement. For any realized sequence of features, if the design evolves by rank-one updates, we get the following:
\[
    V_{t+1}=V_t+z_tz_t^\top, \qquad \sum_{t=1}^T \min\{1,\|z_t\|_{V_t^{-1}}^2\} \le
    2\log\frac{\det(V_{T+1})}{\det(V_1)}.
\]
With a changing projection, the natural projected design
\[
    A_t=\hP_t\left(\lambda I+\sum_{s<t}x_sx_s^\top\right)\hP_t
\]
does not satisfy such a recursion. Instead,
\[
\begin{aligned}
    A_{t+1}-A_t
    &=
    \hP_{t+1}
    \left(\lambda I+\sum_{s<t}x_sx_s^\top\right)
    \hP_{t+1}
    -
    \hP_t
    \left(\lambda I+\sum_{s<t}x_sx_s^\top\right)
    \hP_t \\
    &\quad+
    \hP_{t+1}x_tx_t^\top\hP_{t+1},
\end{aligned}
\]
and the first two terms need not be positive semidefinite or low-rank. Hence the
determinant-telescoping proof does not directly apply to projected designs with
a projection updated every round.

One alternative, used in the PSLB analysis of ~\cite{lale2019stochastic}, is to prove that the covariance of the selected actions has a
linearly growing minimum eigenvalue. The delicate point is that this is the
covariance of actions chosen by the OFU rule, not the population covariance of a
random arm from the offered decision set. Population excitation of the offered
arms does not, by itself, imply selected-design excitation: if every offered set
is $\{e_1,e_2\}$, then a uniformly sampled offered arm has covariance
$I_2/2$, but the rule that always chooses $e_1$ has selected
covariance $Te_1e_1^\top$, whose minimum eigenvalue is zero.  A Matrix Chernoff argument for selected actions therefore
requires an explicit policy-specific excitation or stronger assumptions that guarantee this.

{\bf Why epoch-wise freezing avoids these issues.}
We tackle both of these with our epoch-wise algorithm construction (Algorithm~\ref{alg:bandits}). We estimate
$\hU_e$ using only pre-epoch decision sets, then keep it fixed during
$\mathfrak T_e$. Conditional on the epoch-start sigma-field, the induced features
$z_{t,i}=\hU_e^\top\widehat X_{t,i}$ live in a fixed
$m$-dimensional coordinate system. Thus the selected feature $z_t$ is
predictable, $\sum_{s\in\mathfrak T_e:s<t}z_s\eta_s$ is the usual
self-normalized martingale transform, and
$V_{e,t}:=\lambda I_m+\sum_{s\in\mathfrak T_e:s<t}z_sz_s^\top$ satisfies the
rank-one recursion. Standard OFUL confidence and potential bounds therefore
apply inside each epoch; this is the reduction used in
Section~\ref{sec:confidence_set_analysis}.

What remains is a controlled approximation error rather than a
time-varying-projection issue. In the frozen coordinates,
$r_t=\langle z_t,\theta_e^\star\rangle+b_t+\eta_t$ with $|b_t|\le b_e$, where
$b_e$ is determined by the epoch-start subspace and imputation accuracy
(Lemma~\ref{lemma:surrogate-model}). The main
regret proof in Section~\ref{sec:regret_analysis_main} then sums the usual OFUL
terms and this representation-misspecification contribution over epochs; Lemma~\ref{lem:epoch_bias_term}
gives the corresponding misspecification-control step, and Appendix~\ref{app:bias_control}
explains why this misspecification handling needs some care.

\section{Comparison with a One-shot Two-phase Baseline}
\label{app:two_phase_comparison}

This section compares the epoch-wise algorithm with a simpler one-shot
two-phase strategy. This natural baseline first uses \(\tau\) rounds to estimate the
subspace, then freezes this estimate for the remaining \(T-\tau\) rounds and
runs a standard \(m\)-dimensional OFUL algorithm in the frozen coordinates. One could then ask if a \(\sqrt{T}\) regret is achievable by optimizing $\tau$. In what follows, we will show that the best one could hope for with such a strategy is a $T^{2/3}$ regret.

Let \(\hU_\tau\) be the subspace estimate after \(\tau\ge t_b\) rounds, and
let \(\hP_\tau=\hU_\tau\hU_\tau^\top\). For \(t>\tau\), the two-phase
baseline imputes each arm using \(\hU_\tau\), forms
\(
z_{t,i}:=\hU_\tau^\top \hat X_{t,i}^{(\tau)}\in\mathbb R^m,
\)
and runs OFUL on these fixed \(m\)-dimensional features. On the representation
event \(\mathcal E_{\rm rep}\), defined in Section~\ref{sec:subspace_estimation}
by combining Lemma~\ref{lemma:sub_space_lemma} and
Lemma~\ref{lemma: imputation_error}, the one-shot analogue of
Lemma~\ref{lemma:surrogate-model} gives, for all \(t>\tau\) and \(i\in[K]\),
\[
\left|
\langle X_{t,i},\theta^\star\rangle
-
z_{t,i}^\top\vartheta_\tau^\star
\right|
\le
b_\tau,
\qquad
\vartheta_\tau^\star:=\hU_\tau^\top\theta^\star,
\]
where
\[
b_\tau
:=
SB_X\left(2+\frac{2}{p}\right)\epsilon_\tau
=
\widetilde O\left(
SB_X\,\kappa
\frac{\sqrt m}{p^2\sqrt{K\tau}}
\right).
\]
Thus, during the exploitation phase, the realized reward in each round satisfies
\[
r_t=z_t^\top\vartheta_\tau^\star+\xi_t+\eta_t,
\qquad
|\xi_t|\le b_\tau .
\]

Applying the same misspecified-OFUL argument as in the epoch-wise proof
(Theorem~\ref{thm:confidence_set_m_t_main}, Lemma~\ref{lem:epoch_bias_term},
and Lemma~\ref{lemma:epoch_regret}, with the single frozen epoch beginning at
\(\tau\)), we get
\[
R_T^{\mbox{\footnotesize \sc 1-shot}}(\tau)
\le
O(B_XS\,\tau)
+
\widetilde O\!\left((R+SB_X)\,m\sqrt T\right)
+
\widetilde O\!\left(
SB_X\,\kappa
\frac{mT}{p^2\sqrt{K\tau}}
\right).
\]

Ignoring problem-dependent constants, the two \(\tau\)-dependent terms have the
form
\(\tau+\frac{T}{\sqrt{\tau}}.\) Balancing them gives us a scaling of \(\tau\asymp T^{2/3},\) and hence the regret scales as \(\widetilde O(T^{2/3})\) up to the standard bandit terms and problem-dependent constants.

\section{Subspace Estimation Error: Proof of Lemma~\ref{lemma:sub_space_lemma}}
\label{sec:subspace-estimation-proof}

We first record two ingredients used in our subspace recovery argument: unbiasedness of the corrected covariance estimator and a high-probability spectral-norm concentration bound.

\begin{lemma}[Unbiased covariance estimator under Bernoulli masking]
\label{lemma:unbiased_covariance_missing}
Let \(\dX=S\odot X\), where coordinates of \(S\), $S_j\stackrel{\rm iid}{\sim}\mathrm{Bernoulli}(p)$ and are independent of $X$. Define
\[
Y(\dX)
:=
\frac{1}{p^2}\dX\dX^\top
+
\left(\frac{1}{p}-\frac{1}{p^2}\right)
\mathrm{diag}(\dX\dX^\top).
\]
Then
\(
\E[Y(\dX)\mid X]=XX^\top.
\)
Consequently, \(\dot\Sigma_t\) in Equation~\eqref{eq:unbiased_est_cov} is an unbiased estimator of \(\Sigma=\E[XX^\top]\).
\end{lemma}

\begin{proof}
For \(j\neq k\), observe that 
\[
\E[\dX_j\dX_k\mid X]
=
\E[S_jS_k]X_jX_k
=p^2X_jX_k.
\]
Therefore the off-diagonal entries of \(p^{-2}\dX\dX^\top\) are unbiased estimates of the off-diagonal entries of \(XX^\top\). 

Next, lets turn our attention to the \(j\)-th diagonal entry and observe that 
\[ \mathbb{E}\left[[Y(\dot X)]_{jj} \middle | X \right] = \frac{1}{p^2} \mathbb{E}\left[S_j^2\right]X_j^2 + \left(\frac{1}{p} - \frac{1}{p^2}\right)\mathbb{E}\left[S_j^2\right]X_j^2 = \frac{1}{p}\times p \times X_j^2 = X_j^2.  
\]
From this, the proof follows. 
\end{proof}

Next, we compute the conditional second moment of a single corrected sample exactly. This is the ingredient that allows us to determine the variance proxy in the matrix Bernstein bound.

\begin{lemma}[Conditional second moment of the corrected estimator]
\label{lemma:masked_second_moment}
Let $S_1,\dots,S_d$ be i.i.d.\ $\mathrm{Bernoulli}(p)$ random variables, independent of $X$, and let $Y(\dX)$ be as in Lemma~\ref{lemma:unbiased_covariance_missing}. Then
\[
\E\big[Y(\dX)^2\,\big|\,X\big]
=\frac{\|X\|_2^2}{p}\,XX^\top
+\Big(\frac{1}{p^2}-\frac1p\Big)\,
\mathrm{diag}\Big(\big(X_j^2(\|X\|_2^2-X_j^2)\big)_{j\in[d]}\Big).
\]
\end{lemma}

\begin{proof}
For notational ease, we write $Y:=Y(\dX)$ in the sequel. First, observe that for $j\neq k$,  $Y_{jk}=p^{-2}S_jS_kX_jX_k$ and that 
$Y_{jj}=p^{-1}S_jX_j^2$. 

\underline{\bf Off-diagonal entries.} Fix $j\ne k$ and
expand $(Y^2)_{jk}=\sum_{l}Y_{jl}Y_{lk}$. Notice that the term corresponding to $l=j$ simplifies to 
$p^{-3}S_jS_kX_j^3X_k$, and hence (after applying \(\mathbb{E}\left[\cdot \mid X\right]\)) contributes $p^{-1}X_j^3X_k$ to the sum. The $l=k$ term  similarly contributes $p^{-1}X_jX_k^3$. Each $l\notin\{j,k\}$ on the other hand contributes $p^{-4}S_jS_kS_lX_jX_kX_l^2$, with
conditional mean $p^{-1}X_jX_kX_l^2$, where we use the fact $S_l^2=S_l$ and
$\E[S_jS_kS_l]=p^3$ for distinct indices. Summing over $l$, we have 
\[
\E\big[(Y^2)_{jk}\,\big|\,X\big]=\frac{X_jX_k\|X\|_2^2}{p},
\]
which, after some algebra, can be seen to match the $(j,k)$ entry of the stated matrix. 

\underline{\bf Diagonal Entries.} Notice that 
$(Y^2)_{jj}=Y_{jj}^2+\sum_{l\ne j}Y_{jl}^2$, which in turn equals 
\[p^{-2}S_jX_j^4+p^{-4}\sum_{l\ne j}S_jS_lX_j^2X_l^2,\] by definition. We therefore have 
\[
\E\big[(Y^2)_{jj}\,\big|\,X\big]
=\frac{X_j^4}{p}+\frac{X_j^2(\|X\|_2^2-X_j^2)}{p^2}.
\]
Indeed, this is the $(j,j)$ entry of the stated matrix.
\end{proof}

We can now use these to control the covariance estimation error using a matrix Bernstein argument. 

\begin{lemma}[Covariance estimation error]
\label{lemma:covariance_estimation_error}
Let Assumptions~\ref{ass:bounded_actions} and~\ref{ass:spectrum} hold,
and let the coordinate masks be i.i.d.\ $\mathrm{Bernoulli}(p)$,
independent across coordinates, arms, and rounds, and independent of
the actions. For any fixed $t\in[T]$, with probability at least
$1-\delta$,
\begin{equation}
\label{eq:corrected_cov_rate}
\|\dot\Sigma_t-\Sigma\|_2
\;\le\;
2B_X\sqrt{\frac{\bar\lambda}{p^2\,tK}\log\frac{2d}{\delta}}
\;+\;
\frac{2B_X^2}{p^2\,tK}\log\frac{2d}{\delta}.
\end{equation}
\end{lemma}

\begin{proof}
In the following, we will let $n:=tK$ and index the i.i.d. samples by $\ell\in[n]$. Let's set 
$Q_\ell:=Y_\ell-\Sigma$, so $\dot\Sigma_t-\Sigma=n^{-1}\sum_\ell Q_\ell$
and $\E[Q_\ell]=0$ by Lemma~\ref{lemma:unbiased_covariance_missing}.

\underline{\bf Uniform bound.} Since the mask cannot increase Euclidean norms, we have that 
$\|\dX_\ell\|_2\le B_X$ and $\max_j\dX_{\ell,j}^2\le B_X^2$. This gives us the following uniform norm bound on the \(Q_\ell\)'s
\[
\|Y_\ell\|_2
\le\frac{\|\dX_\ell\|_2^2}{p^2}
+\frac{1-p}{p^2}\max_j\dX_{\ell,j}^2
\le\frac{2B_X^2}{p^2},
\qquad
\|Q_\ell\|_2\le\frac{2B_X^2}{p^2}+\lambda_1
\le\frac{3B_X^2}{p^2}=:L_Q,
\]
where we of course use $\lambda_1\le\E\|X\|_2^2\le B_X^2$.

\underline{\bf Variance.} Since $\E Y_\ell=\Sigma$, we have that 
$\E Q_\ell^2=\E Y_\ell^2-\Sigma^2\preceq\E Y_\ell^2$. We will use the fact that $X_j^2(\|X\|_2^2-X_j^2)\le B_X^2X_j^2$ and $0\le p^{-2}-p^{-1}\le p^{-2}$ in
Lemma~\ref{lemma:masked_second_moment} to get
\[
\E\big[Y_\ell^2\,\big|\,X_\ell\big]
\preceq\frac{B_X^2}{p}X_\ell X_\ell^\top
+\frac{B_X^2}{p^2}\,\mathrm{diag}(X_\ell X_\ell^\top).
\]
Now, taking expectations over $X_\ell$ and writing
$\Sigma^D:=\mathrm{diag}(\Sigma)$, whose operator norm is
$\max_j\Sigma_{jj}\le\lambda_1\le\bar\lambda$, we have the following bound on the variance term
\[
\big\|\E Q_\ell^2\big\|_2
\le\frac{B_X^2\lambda_1}{p}+\frac{B_X^2\bar\lambda}{p^2}
\le\frac{2B_X^2\bar\lambda}{p^2}=:v,
\qquad
\Big\|\sum_{\ell=1}^n\E Q_\ell^2\Big\|_2\le nv.
\]

\underline{\bf Concentration.} The matrices $Q_1,\ldots,Q_n$ are independent (the offered arms and their masks are i.i.d.\ and independent of the learner's policy), symmetric, and zero-mean. Moreover, by the two preceding steps, they satisfy the almost-sure bound $\|Q_\ell\|_2\le L_Q$ and the variance bound $\big\|\sum_{\ell}\E Q_\ell^2\big\|_2\le\sigma^2:=nv$. For such a family of matrices, recall that the Matrix Bernstein's
inequality~\cite{tropp2015introductionmatrixconcentrationinequalities} says, for every $s\ge0$,
\[
\Pr\Big(\Big\|\sum_{\ell=1}^nQ_\ell\Big\|_2\ge s\Big)
\le
2d\exp\!\left(\frac{-s^2/2}{\sigma^2+L_Qs/3}\right).
\]
Equivalently, we may use the following form. Fix $u>0$, and observe that the right-hand side above is at most $2d\,e^{-u}$ whenever $s^2/2\ge u\left(\sigma^2+L_Qs/3\right)$, that is, whenever $s$ exceeds the larger root
\[
s_+
:=
\frac{L_Qu}{3}+\sqrt{\frac{L_Q^2u^2}{9}+2\sigma^2u}
\]
of the corresponding quadratic equation. By $\sqrt{a+b}\le\sqrt a+\sqrt b$, this root satisfies $s_+\le\tfrac23L_Qu+\sqrt{2\sigma^2u}$, so the choice $s:=\sqrt{2\sigma^2u}+\tfrac23L_Qu$ suffices, and we get
\[
\Pr\Big(\Big\|\sum_{\ell=1}^nQ_\ell\Big\|_2
\ge\sqrt{2nvu}+\tfrac{2}{3}L_Qu\Big)\le 2d\,e^{-u}.
\]
Setting $u=\log(2d/\delta)$ and dividing by $n=tK$ gives us the desired 
\eqref{eq:corrected_cov_rate}, since
$\sqrt{2v}=2B_X\sqrt{\bar\lambda}/p$ and $\tfrac23L_Q=2B_X^2/p^2$.
\end{proof}

We are now ready to prove the subspace recovery guarantee.

{\bf Proof of Lemma~\ref{lemma:sub_space_lemma}.}
Recall from Section~\ref{sec:problem_setup} that \(\kappa\sqrt m=B_X\sqrt{\bar\lambda}/\lambda_m\), so that
\[
\epsilon_t
=
C_{\mathrm{sub}}\,\frac{B_X\sqrt{\bar\lambda}}{\lambda_m\,p}\sqrt{\frac{u}{tK}},
\qquad
u:=\log\frac{2dT}{\delta}. 
\]
In the sequel, we will fix \(C_{\mathrm{sub}}=12\) for the sake of convenience. Applying  Lemma~\ref{lemma:covariance_estimation_error} at level $\delta/T$
and taking a union bound over $t\in[T]$, we get that, with probability at least $1-\delta$,
for all $t\in[T]$,
\[
\|\dot\Sigma_t-\Sigma\|_2\le E_t
:=2B_X\sqrt{\frac{\bar\lambda\,u}{p^2\,tK}}
+\frac{2B_X^2u}{p^2\,tK}.
\]
For any fixed \(t\in[T]\), consider the following two cases.  

\emph{Case 1: \(\epsilon_t\ge\sqrt2\).} The difference of two orthogonal projectors always has operator norm at most one, so \(\|\hP_t-\Pb\|_2\le1\le\epsilon_t\), and the claim holds trivially.

\emph{Case 2: \(\epsilon_t<\sqrt2\).} By the definition of \(\epsilon_t\), this means
\[
tK
>
\frac{C_{\mathrm{sub}}^2}{2}\cdot\frac{B_X^2\bar\lambda\,u}{\lambda_m^2\,p^2}
=
72\,\frac{B_X^2\bar\lambda\,u}{\lambda_m^2\,p^2}.
\]
We use this lower bound on \(tK\) twice. Since \(\lambda_m\le\bar\lambda\), it implies \(tK\ge B_X^2u/(\bar\lambda p^2)\), which is exactly the condition under which the linear term of \(E_t\) is at most its square-root term; hence \(E_t\le4B_X\sqrt{\bar\lambda u/(p^2tK)}\). Substituting the lower bound on \(tK\) into this expression then gives
\(
E_t\le 4\lambda_m/\sqrt{72}=(\sqrt2/3)\,\lambda_m\le\lambda_m/2 .
\)
The Davis--Kahan
\(\sin\Theta\) theorem~\cite{davis1970rotation} applied to $\Sigma$,
whose $m$-th eigenvalue is $\lambda_m$ and whose $(m{+}1)$-st is $0$,
then gives
\[
\|\hP_t-\Pb\|_2\le\frac{2E_t}{\lambda_m}
\le\frac{8B_X\sqrt{\bar\lambda}}{\lambda_m\,p}
\sqrt{\frac{u}{tK}}
=
\frac{8}{C_{\mathrm{sub}}}\,\epsilon_t
\le
\epsilon_t .
\]
\hfill$\square$

In the imputation analysis of Appendix~\ref{sec:imputation_error_analysis}, we need a version of this guarantee for the basis matrices rather than the projectors. The following standard argument shows us how one can obtain the former from the latter. 

\begin{lemma}[Basis alignment]
\label{lemma:basis_alignment}
Let $\hU,\U\in\R^{d\times m}$ have orthonormal columns, with projectors
$\hP=\hU\hU^\top$ and $\Pb=\U\U^\top$. Then there is an orthogonal
$O\in\R^{m\times m}$ with
$\|\hU O-\U\|_2\le\sqrt2\,\|\hP-\Pb\|_2$.
\end{lemma}

\begin{proof}
Consider an SVD of the $m\times m$ matrix $\hU^\top\U$. Every singular value of this matrix is of the form $\langle\hU x,\U y\rangle$ for unit vectors $x,y\in\R^m$. Since $\hU$ and $\U$ have orthonormal columns, $\hU x$ and $\U y$ are also unit vectors, and therefore, by Cauchy--Schwarz all singular values lie in $[0,1]$. We may therefore parametrize them as $\cos\theta_1\ge\cdots\ge\cos\theta_m$ for angles $\theta_1\le\cdots\le\theta_m=:\theta_{\max}$ in $[0,\pi/2]$. These are, by definition, the principal angles between the two subspaces~\cite{bjorck1973numerical}. So, if we write the SVD of \(\hU^\top \hU\) as
\[
\hU^\top\U=A\cos\Theta\,B^\top,
\qquad
\cos\Theta:=\mathrm{diag}(\cos\theta_1,\ldots,\cos\theta_m),
\]
and set $O:=AB^\top$, which is orthogonal, we get the following:
\[
(\hU O-\U)^\top(\hU O-\U)
=
O^\top\hU^\top\hU O
-O^\top\hU^\top\U
-\U^\top\hU O
+\U^\top\U .
\]
The first and last terms each equal $I_m$. For the cross terms, $O^\top\hU^\top\U=BA^\top\cdot A\cos\Theta\,B^\top=B\cos\Theta B^\top$, and $\U^\top\hU O$ is its transpose, which is the same symmetric matrix. Therefore
\[
(\hU O-\U)^\top(\hU O-\U)
=
2I-2B\cos\Theta B^\top
=
B\left(2I-2\cos\Theta\right)B^\top.
\]
Notice that this is an eigendecomposition with eigenvalues $2(1-\cos\theta_j)$, the largest value being attained at $\theta_{\max}$. This implies the following inequality: 
\[
\|\hU O-\U\|_2^2
=
2(1-\cos\theta_{\max})
\le
2(1-\cos\theta_{\max})(1+\cos\theta_{\max})
=
2\sin^2\theta_{\max},
\]
where we use the fact that $\cos\theta_{\max}\in[0,1]$.
Since $\|\hP-\Pb\|_2=\sin\theta_{\max}$ for equal-rank projectors (see, e.g.,~\cite{stewart1990matrix}), the
claim follows.
\end{proof}

\begin{corollary}
\label{cor:aligned_basis}
On the event of Lemma~\ref{lemma:sub_space_lemma}, simultaneously for all \(t\in[T]\), we have the following
\[
\min_{O\in\mathbb O(m)}\|\hU_tO-\U\|_2\le\epsilon_t,
\]
where \(\mathbb O(m)\) is the set of \(m\times m\) orthogonal matrices.
\end{corollary}

\begin{proof} We consider two cases:\\
\underline{\bf Case 1: \(\epsilon_t\ge\sqrt2\)}. In this situation, notice that Lemma~\ref{lemma:basis_alignment} gives us \[\min_O\|\hU_tO-\U\|_2\le\sqrt2\,\|\hP_t-\Pb\|_2\le\sqrt2\le\epsilon_t.\] \\
\underline{\bf Case 2: \(\epsilon_t < \sqrt2\)}
Here, Case~2 of the preceding proof gives \(\|\hP_t-\Pb\|_2\le(8/C_{\mathrm{sub}})\epsilon_t\). This, along with Lemma~\ref{lemma:basis_alignment}, implies that \[\min_O\|\hU_tO-\U\|_2\le(8\sqrt2/C_{\mathrm{sub}})\epsilon_t\le\epsilon_t,\] where we used the fact that \(C_{\mathrm{sub}}=12\ge8\sqrt2\).
\end{proof}

\section{Imputation Error: Proof of Lemma~\ref{lemma: imputation_error}}
\label{sec:imputation_error_analysis}

In Section~\ref{sec:subspace-estimation-proof}, we showed that the projection matrices onto the estimated subspaces converge to the true projector. That is, on a single high-probability event, \(\|\hP_t-\Pb\|_2\le\epsilon_t\) simultaneously for all \(t\in[T]\) (Lemma~\ref{lemma:sub_space_lemma}), and the same bound holds for the best-aligned bases (Corollary~\ref{cor:aligned_basis}).
In this section, we prove Lemma~\ref{lemma: imputation_error}, which converts these subspace guarantees into a uniform bound on the error of the least-squares imputation in Equation~\eqref{eq:imputation_operator}, over every epoch, every round within it, and every offered arm. Our strategy is guided by the observation that the imputation solves a least-squares problem on the \emph{observed rows} of {the frozen basis \(\hU_e\), so its stability is governed by the smallest eigenvalue of the observed Gram matrix \(\hU_{e,\Omega_{t,i}}^\top\hU_{e,\Omega_{t,i}}\)}.
  We first prove that the observed rows of the true subspace are well-conditioned under Bernoulli masking. We then show that this conditioning transfers to the estimated subspace after sufficient burning in. Finally, we combine this conditioning with the aligned-basis guarantee to control the imputation error itself.

\begin{lemma}[Conditioning of the true observed subspace]
\label{lemma:true_observed_gram_conditioning}
Assume $\U$ is $\mu$-incoherent (i.e., 
\(
\max_{j\in[d]}\|e_j^\top\U\|_2^2\leq \frac{\mu^2m}{d}.
\))
If Equation~\eqref{eq:p_condition_imputation} holds with $C_\mu$ sufficiently large, then with probability at least $1-\delta$, simultaneously for all $t\in[T]$ and all $i\in[K]$,
\begin{align}
\lambda_{\min}(\U_{\Omega_{t,i}}^\top\U_{\Omega_{t,i}})
\geq
\frac{3p}{4}.
\end{align}
\end{lemma}

\begin{proof}
For a fixed mask $\Omega$, notice that we can write
\[
\U_\Omega^\top\U_\Omega
=
\sum_{j=1}^d S_j u_j u_j^\top,
\]
where $S_j\sim\mathrm{Bernoulli}(p)$ and $u_j^\top=e_j^\top\U$ is the $j$-th row of $\U$. Indeed, each summand is positive semidefinite and satisfies
\[
0\preceq S_j u_ju_j^\top\preceq \frac{\mu^2m}{d}I_m.
\]
Moreover,
\(
\E[\U_\Omega^\top\U_\Omega]
=
p\sum_{j=1}^d u_ju_j^\top
=pI_m.
\)
Now, we can use the Matrix Chernoff bound~\cite{tropp2015introductionmatrixconcentrationinequalities} and observe that
\[
\Pr\!\left(
\lambda_{\min}(\U_\Omega^\top\U_\Omega)\leq \frac{3p}{4}
\right)
\leq
m\exp\!\left(-c\frac{pd}{\mu^2m}\right).
\]
Choosing $p\geq C_\mu(\mu^2m/d)\log(mTK/\delta)$ and taking a union bound over all $TK$ masks gives us the desired result.
\end{proof}

\begin{lemma}[Conditioning of the estimated observed subspace]
\label{lemma:min_eigenvalue_uu}
{Fix an epoch \(e\), and write \(\hP_e:=\hU_e\hU_e^\top\). On the event of Lemma~\ref{lemma:sub_space_lemma}, since \(\hU_e\) is computed from the \(\tau_e-1\ge t_b\) rounds preceding the epoch, we have \(\|\hP_e-\Pb\|_2\leq\epsilon_e\leq p/32\). On the event of Lemma~\ref{lemma:true_observed_gram_conditioning}, for all \(t\in\mathfrak T_e\) and $i\in[K]$,
\begin{align}
\lambda_{\min}(\hU_{e,\Omega_{t,i}}^\top\hU_{e,\Omega_{t,i}})
\geq
\frac{p}{2}.
\end{align}}
\end{lemma}

\begin{proof}
{Let $R_\Omega$ be the coordinate-selection matrix for $\Omega=\Omega_{t,i}$. Since
\(
R_\Omega\Pb R_\Omega^\top
=
\U_\Omega\U_\Omega^\top
\), the non-zero eigenvalues of \(R_\Omega\Pb R_\Omega^\top\) coincide with the eigenvalues of $\U_\Omega^\top\U_\Omega$. Similarly, the nonzero eigenvalues of $R_\Omega\hP_eR_\Omega^\top$ coincide with those of $\hU_{e,\Omega}^\top\hU_{e,\Omega}$.

Note that on the event of Lemma~\ref{lemma:true_observed_gram_conditioning} we have $\lambda_{\min}(\U_\Omega^\top\U_\Omega)\ge3p/4>0$, which forces $\mathrm{rank}(\U_\Omega)=m$ and in particular $|\Omega|\ge m$; hence the spectra of the $m\times m$ Gram matrices are exactly the top $m$ eigenvalues of the corresponding restricted projectors, and Weyl's inequality may be applied to the $m$-th eigenvalue. Therefore, by Weyl's inequality~\cite{horn2012matrix},
\begin{align*}
\lambda_{\min}(\hU_{e,\Omega}^\top\hU_{e,\Omega})
&\geq
\lambda_{\min}(\U_{\Omega}^\top\U_{\Omega})
-
\|R_\Omega(\hP_e-\Pb)R_\Omega^\top\|_2 \\
&\geq
\frac{3p}{4}-\|\hP_e-\Pb\|_2
\geq
\frac{3p}{4}-\frac{p}{32}
\geq
\frac{p}{2}.
\end{align*}
This proves the claim.}
\end{proof}

{We are now ready to prove Lemma~\ref{lemma: imputation_error}. At a high level, the imputation error has two components. First, even if the latent coefficient were known, reconstructing with the frozen basis $\hU_e$ instead of $\U$ creates an error proportional to the subspace error. Second, the coefficient estimated from the observed entries is itself perturbed by the subspace error, and this perturbation is amplified by the inverse observed Gram matrix; since the smallest eigenvalue of that Gram matrix is at least $p/2$, the amplification is at most $2/p$.}

\begin{proof}
We invoke Lemma~\ref{lemma:sub_space_lemma} and Lemma~\ref{lemma:true_observed_gram_conditioning}, each at failure probability $\delta$, and work on the intersection of the two events. By a union bound, this intersection has probability at least $1-2\delta\ge1-3\delta$. On it, the conclusions of Lemma~\ref{lemma:min_eigenvalue_uu} and Corollary~\ref{cor:aligned_basis} hold deterministically.

{Fix an epoch \(e\), a round $t\in\mathfrak T_e$, and an arm $i\in[K]$, and abbreviate $\Omega=\Omega_{t,i}$ and $\hU=\hU_e$. Recall that \(\epsilon_e=\epsilon_{\tau_e-1}\), and that Corollary~\ref{cor:aligned_basis}, applied at the epoch start, controls the aligned distance of \(\hU_e\) at this level.} Since $X_{t,i}\in\mathrm{span}(\U)$, there exists $a\in\R^m$ such that
\[
X_{t,i}=\U a,
\qquad
\|a\|_2=\|X_{t,i}\|_2\leq B_X.
\]
The imputed coefficient is
\[
\hat a
=
(\hU_\Omega^\top\hU_\Omega)^{-1}\hU_\Omega^\top\U_\Omega a.
\]
Notice that the imputed vector $\hX_{t,i}$ is invariant under right-orthogonal rotations $\hU\mapsto\hU O$ since the coefficient transforms as $\hat a\mapsto O^\top\hat a$, and the reconstruction $\hU_{\Omega^c}\hat a$ remains unchanged. Similarly, $\lambda_{\min}(\hU_\Omega^\top\hU_\Omega)$ is rotation-invariant, so Lemma~\ref{lemma:min_eigenvalue_uu} is unaffected. Therefore, in what follows, we assume that $\hU$ is the aligned basis of Corollary~\ref{cor:aligned_basis}, so that we actually have 
\[
\|\hU-\U\|_2\leq \epsilon_e.
\]

Next, we let $\hat G:=\hU_\Omega^\top\hU_\Omega$. And, since
\(
\hat G(\hat a-a)
=
\hU_\Omega^\top(\U_\Omega-\hU_\Omega)a,
\) Lemma~\ref{lemma:min_eigenvalue_uu} gives us the following ineqiality: 
\begin{align*}
\|\hat a-a\|_2
&\leq
\|\hat G^{-1}\|_2\,\|\hU_\Omega\|_2\,\|\U_\Omega-\hU_\Omega\|_2\,\|a\|_2 \\
&\leq
\frac{2}{p}\cdot 1\cdot \epsilon_e\cdot B_X
=
\frac{2B_X\epsilon_e}{p}.
\end{align*}

Indeed, the imputation error is zero on the observed coordinates and, on the missing coordinates, we have
\begin{align*}
\|X_{t,i}^{(\Omega^c)}-\hX_{t,i}^{(\Omega^c)}\|_2
&=
\|\U_{\Omega^c}a-\hU_{\Omega^c}\hat a\|_2 \\
&\leq
\|(\U_{\Omega^c}-\hU_{\Omega^c})a\|_2
+
\|\hU_{\Omega^c}(a-\hat a)\|_2 \\
&\leq
B_X\epsilon_e
+
\|a-\hat a\|_2 \\
&\leq
\left(1+\frac{2}{p}\right)B_X\epsilon_e.
\end{align*}
Since the observed coordinates are copied exactly, the same bound holds for the full vector:
\[
\|X_{t,i}-\hX_{t,i}\|_2
\leq
\left(1+\frac{2}{p}\right)B_X\epsilon_e.
\]
{The argument is uniform over all epochs \(e\), rounds $t\in\mathfrak T_e$, and arms $i\in[K]$ on the same good event, which proves the lemma.}
\end{proof}

\section{Epoch-wise Surrogate Model and Estimation Error Proofs}
\label{app:epoch-wise-surrogate}

This appendix proves the confidence set bounds of Theorem~\ref{thm:confidence_set_m_t_main}, together with the supporting lemmas stated in Section~\ref{sec:confidence_set_analysis}. Throughout, we fix an epoch \(e\) and recall that the representation \(\hU_e\) is frozen for its duration. The proof proceeds in three stages, mirroring the error decomposition
\[
V_{e,t}\left(\hat\vartheta_{e,t}-\vartheta_e^\star\right)
=
-\lambda\vartheta_e^\star
+\sum_{\substack{s\in\mathfrak T_e\\s<t}}z_s\xi_s
+\sum_{\substack{s\in\mathfrak T_e\\s<t}}z_s\eta_s
\]
from Section~\ref{sec:confidence_set_analysis}. First, we prove the surrogate approximation guarantee (Lemma~\ref{lemma:surrogate-model}): on the representation event, the frozen-coordinate surrogate mean is within \(b_e\) of the true mean reward, which justifies the reward decomposition \(r_t=\langle z_t,\vartheta_e^\star\rangle+\xi_t+\eta_t\) with \(|\xi_t|\le b_e\). This is where the subspace and imputation guarantees of Appendices~\ref{sec:subspace-estimation-proof} and~\ref{sec:imputation_error_analysis} enter the bandit analysis. Second, we control the aggregated misspecification term \(\sum_{s<t}z_s\xi_s\) (Lemma~\ref{lem:epoch_bias_term}). Theorem~\ref{thm:confidence_set_m_t_main} then follows by taking \(V_{e,t}^{-1}\)-weighted norms in the error decomposition above, and combining with a self-normalized inequality argument that controls the noise term. 

\subsection{Surrogate approximation}

We first reintroduce the fixed-epoch objects used in the proof. Fix an epoch \(e\). For \(t\in\mathfrak T_e\) and \(i\in[K]\), let \(z_{t,i}:=\hU_e^\top \hat X_{t,i}\), \(\vartheta_e^\star:=\hU_e^\top\theta^\star\), and \(\bar\mu_{t,i}:=\langle z_{t,i},\vartheta_e^\star\rangle\). The purpose of this section is to show that, after conditioning on the representation event, the epoch behaves like an ordinary \(m\)-dimensional linear bandit with a controlled misspecification term. The first step is to compare the true reward mean with its frozen-coordinate surrogate. Notice that the error has exactly two sources: the estimated projection is not exactly the true projection, and the current partially observed arm must be imputed before it can be projected.

\begin{lemma}[Restatement of Lemma~\ref{lemma:surrogate-model}]
On the event \(\mathcal E_{\rm rep}\) defined in Section~\ref{def:representation_event}, for every epoch \(e\), every round \(t\in\mathfrak T_e\), and every arm \(i\in[K]\),
\[
\left|\langle X_{t,i},\theta^\star\rangle-\bar\mu_{t,i}\right|
\le
b_e,
\qquad
b_e:=S B_X\left(2+\frac{2}{p}\right)\epsilon_e.
\]
\end{lemma}

\begin{proof}
\newcommand{\ip}[2]{\left\langle #1,#2 \right\rangle}
\newcommand{\tr}{\operatorname{Tr}}
\newcommand{\polylog}{\operatorname{polylog}}

\newcommand{\thetaStar}{\theta^\star}
\newcommand{\varthetaStar}{\vartheta^\star}
\newcommand{\tO}{\widetilde O}
First we begin by observing that
\[
    \bar\mu_{t,i}  = z_{t,i}^\top \vartheta^\ast_e = 
    (\hU_e^\top\hX_{t,i})^\top(\hU_e^\top\thetaStar) = \ip{\hP_e\hX_{t,i}}{\thetaStar}.
\]
Therefore, on $\mathcal E_{\rm rep}$, we have the following inequality: 
\begin{equation}
\label{eq:bias-decomp}
    |\ip{X_{t,i}}{\thetaStar}-\bar\mu_{t,i}|
    =  \left|\ip{X_{t,i}-\hP_e\hX_{t,i}}{\thetaStar}\right|\le S\left\|{X_{t,i}-\hP_e\hX_{t,i}}\right\|_2,
\end{equation}
which follows from the Cauchy-Schwarz inequality. Next, we observe that 
\begin{align*}
    X_{t,i}-\hP_e\hX_{t,i}
    &=
    (I-\hP_e)X_{t,i}+\hP_e(X_{t,i}-\hX_{t,i}).
\end{align*}
Therefore, by triangle inequality, 
\begin{align}
    \left\|{X_{t,i}-\hP_e\hX_{t,i}}\right\|_2
    &\le
    \left\|{(I-\hP_e)X_{t,i}}\right\|_2
    +
    \left\|{X_{t,i}-\hX_{t,i}}\right\|_2
    \label{eq:feature-minus-projected}
\end{align}
We will now bound these two terms. For the first term, begin by observing that $\mathbf{P} X_{t,i}=X_{t,i}$,
$
   (I-\hP_e)X_{t,i}=(\mathbf{P}-\hP_e)X_{t,i}.
$
On $\mathcal E_{\rm rep}$, Lemma~\ref{lemma:sub_space_lemma}, invoked at the epoch start (recall that $\epsilon_e:=\epsilon_{\tau_e-1}$, matching the $\tau_e-1$ rounds from which $\hU_e$ is computed), gives $\|\mathbf{P}-\hP_e\|_2\le\epsilon_e$, and we therefore have:
\begin{align}
    \left\|{(I-\hP_e)X_{t,i}}\right\|_2
    \le
    \left\|{\mathbf{P}-\hP_e}\right\|_2 \left\|{X_{t,i}}\right\|_2
    \le
    B_X\epsilon_e.
    \label{eq:first-term}
\end{align}

Finally, using the imputation bound in Lemma~\ref{lemma: imputation_error}, we have 
\begin{align}
\left\|{X_{t,i}-\hX_{t,i}}\right\|_2
    \le
    \left(1+\frac{2}{p}\right)B_X\epsilon_e.
    \label{eq:final-piece}
\end{align}
Combining \eqref{eq:bias-decomp}, \eqref{eq:feature-minus-projected}, \eqref{eq:first-term}, and \eqref{eq:final-piece} gives
\[
    |\ip{X_{t,i}}{\thetaStar}-\bar\mu_{t,i}|
    \le
    SB_X\left(2+\frac{2}{p}\right)\epsilon_e
    = b_e.
\]
\end{proof}

Thus all representation error inside epoch \(e\) is controlled by the scalar radius \(b_e\). Once this approximation is in place, the confidence analysis can be carried out in the frozen feature space, provided we account for how the accumulated \(b_e\)-misspecification enters the confidence radius. We do this next.

\subsection{Proof of Lemma~\ref{lem:epoch_bias_term}}
We now control the contribution of the bounded misspecification terms \(\xi_s\). A term-by-term triangle inequality would be valid but loose; Appendix~\ref{app:bias_control} shows this loss explicitly. The sharper argument keeps the misspecification vector aggregated and uses the fact that the same feature matrix that multiplies the misspecification also appears in the ridge design matrix.

\begin{proof}
The proof proceeds in two steps. First, we rewrite the misspecification sum in matrix form and express its squared \(V_{e,t}^{-1}\)-norm as a quadratic form in a leverage matrix. Second, we show that this leverage matrix is a contraction, so that the quadratic form is bounded by \(\|\xi_t\|_2^2\), which in turn is at most \(b_e^2(t-\tau_e)\).

We begin by setting up the matrix form. Let $q:=t-\tau_e$, and form the feature matrix $Z_t$ and the misspecification error vector $\xi_t$ as follows:
\[
Z_t:=
\begin{bmatrix}
z_{\tau_e} & z_{\tau_e+1} & \cdots & z_{t-1}
\end{bmatrix}
\in\mathbb R^{m\times q}, \qquad \xi_t
:=
\begin{bmatrix}
\xi_{\tau_e} & \xi_{\tau_e+1} & \cdots & \xi_{t-1}
\end{bmatrix}^{\top}
\in\mathbb R^q.
\]
By construction, we have that
\(
\sum_{\substack{s\in\mathfrak T_e\\ s<t}} z_s\xi_s
=
Z_t\xi_t.
\)
\
Moreover, since \(V_{e,t}\) is the within-epoch design matrix before round \(t\), we may write it as
\[
V_{e,t}
=
\lambda I_m+\sum_{\substack{s\in\mathfrak T_e\\ s<t}}z_sz_s^\top
=
\lambda I_m+Z_tZ_t^\top.
\]
Combining the two preceding facts, we have the following identity:
\begin{align*}
\left\|
\sum_{\substack{s\in\mathfrak T_e\\ s<t}} z_s\xi_s
\right\|_{V_{e,t}^{-1}}^2
&=
\|Z_t\xi_t\|_{V_{e,t}^{-1}}^2 \\
&=
(Z_t\xi_t)^\top V_{e,t}^{-1}(Z_t\xi_t) \\
&=
\xi_t^\top Z_t^\top
(\lambda I_m+Z_tZ_t^\top)^{-1}
Z_t\xi_t.
\end{align*}

We will next show that the leverage matrix appearing in this quadratic form is a contraction:
\[
Z_t^\top(\lambda I_m+Z_tZ_t^\top)^{-1}Z_t
\preceq I_q.
\]
Towards this, we write a thin singular value decomposition $Z_t=A\Sigma B^\top$, where $r=\operatorname{rank}(Z_t)$, $A\in\mathbb R^{m\times r}$ and
$B\in\mathbb R^{q\times r}$ have orthonormal columns, and
$
\Sigma=\operatorname{diag}(\sigma_1,\ldots,\sigma_r)
$
contains the positive singular values of \(Z_t\). We then have that
\(
Z_tZ_t^\top
=
A\Sigma^2A^\top.
\)

Therefore, the matrix \(\lambda I_m+Z_tZ_t^\top\) has eigenvalue
\(\lambda+\sigma_j^2\) in the direction of the \(j\)-th column of \(A\), and
eigenvalue \(\lambda\) on the orthogonal complement of \(\operatorname{span}(A)\). This means, we can write:
\[
A^\top(\lambda I_m+Z_tZ_t^\top)^{-1}A
=
\operatorname{diag}\left(
\frac{1}{\lambda+\sigma_1^2},
\ldots,
\frac{1}{\lambda+\sigma_r^2}
\right).
\]
Now, substituting the singular value decomposition into the leverage matrix, we have:
\begin{align*}
Z_t^\top(\lambda I_m+Z_tZ_t^\top)^{-1}Z_t
&=
B\Sigma A^\top
(\lambda I_m+Z_tZ_t^\top)^{-1}
A\Sigma B^\top \\
&=
B
\operatorname{diag}\left(
\frac{\sigma_1^2}{\lambda+\sigma_1^2},
\ldots,
\frac{\sigma_r^2}{\lambda+\sigma_r^2}
\right)
B^\top.
\end{align*}
We observe that the nonzero eigenvalues of this matrix are
$
\frac{\sigma_j^2}{\lambda+\sigma_j^2}, j=1,\ldots,r,
$
and since \(\lambda>0\), each of these lies in \([0,1]\). On the other hand, on the
orthogonal complement of \(\operatorname{span}(B)\), the matrix has eigenvalue
zero. This immediately allows us to conclude the claimed contraction property:
$
0
\preceq
Z_t^\top(\lambda I_m+Z_tZ_t^\top)^{-1}Z_t
\preceq
I_q.
$

Indeed, applying this contraction to the quadratic form above, we observe that:
\[
\|Z_t\xi_t\|_{V_{e,t}^{-1}}^2
\le
\|\xi_t\|_2^2.
\]
Finally, since the hypothesis of the lemma gives \(|\xi_s|\le b_e\) for every \(s\in\mathfrak T_e\) with \(s<t\), we have:
\[
\|\xi_t\|_2^2
=
\sum_{s=\tau_e}^{t-1}\xi_s^2
\le
\sum_{s=\tau_e}^{t-1} b_e^2
=
b_e^2(t-\tau_e).
\]
Putting everything together, we conclude that
\[
\left\|
\sum_{\substack{s\in\mathfrak T_e\\ s<t}} z_s\xi_s
\right\|_{V_{e,t}^{-1}}^2
\le
b_e^2(t-\tau_e).
\]
Taking square roots gives us the claim.
\end{proof}

The preceding contraction argument shows that the misspecification contribution scales as \(b_e\sqrt{t-\tau_e}\). Combining this controlled term with the ridge regularization term and the standard self-normalized reward-noise term gives the frozen-epoch confidence radius.

\subsection{Proof of Theorem~\ref{thm:confidence_set_m_t_main}}
We finish the section by deriving the confidence set for the surrogate parameter \(\vartheta_e^\star\). The calculation is the usual ridge-regression decomposition, with the additional bounded misspecification term controlled by Lemma~\ref{lem:epoch_bias_term}.

\begin{proof}
Recall from Algorithm~\ref{alg:bandits} that \(\hat\vartheta_{e,t}=V_{e,t}^{-1}\sum_{s\in\mathfrak T_e,\,s<t}z_sr_s\) is the epoch-\(e\) ridge estimator. On the event \(\mathcal E_{\rm rep}\), Lemma~\ref{lemma:surrogate-model} allows us to write each within-epoch reward as \(r_s=z_s^\top\vartheta_e^\star+\xi_s+\eta_s\) with \(|\xi_s|\le b_e\). Substituting this decomposition into the definition of the estimator, we have:
\begin{align*}
    V_{e,t}\htheta_{e,t}
    =
    \sum_{s\in\mathfrak{T}_e,\,s<t}z_s r_s
    &=
    \sum_{s<t}z_s(z_s^\top\vartheta_e^\star+\xi_s+\eta_s) \\
    &=
    (V_{e,t}-\lambda I_m)\vartheta_e^\star
    +
    \sum_{s<t}z_s\xi_s
    +
    \sum_{s<t}z_s\eta_s,
\end{align*}
where all sums are over $s\in\mathfrak{T}_e$ with $s<t$. Therefore
$$
    V_{e,t}(\htheta_{e,t}-\vartheta_e^\star)
    =
    -\lambda\vartheta_e^\star
    +
    \sum_{s<t}z_s\xi_s
    +
    \sum_{s<t}z_s\eta_s.
$$
Taking the $V_{e,t}^{-1}$ norm of both sides and applying the triangle inequality (the first term uses \(V_{e,t}\succeq\lambda I_m\), so that \(\|\lambda\vartheta_e^\star\|_{V_{e,t}^{-1}}\le\sqrt\lambda\|\vartheta_e^\star\|_2\)), we have:
\begin{align*}
    \left\|{\htheta_{e,t}-\vartheta_e^\star}\right\|_{V_{e,t}}
    &\le
    \sqrt{\lambda}\left\|{\vartheta_e^\star}\right\|_2
    +
    \left\|\sum_{s<t}z_s\xi_s\right\|_{V_{e,t}^{-1}}
    +
    \left\|\sum_{s<t}z_s\eta_s\right\|_{V_{e,t}^{-1}}.
\end{align*}
Since $\left\|{\vartheta_e^\star}\right\|_2\le\left\|{\theta^\ast}\right\|_2\le S$, the first term is at most
$\sqrt\lambda S$. By Lemma~\ref{lem:epoch_bias_term}, the misspecification term is at most
$b_e\sqrt{t-\tau_e}$. Finally, we control the noise term with the self-normalized inequality of \citet{abbasi2011improved}. This applies here because, conditional on the epoch-start \(\sigma\)-field, the frozen feature map makes each \(z_s\) predictable (measurable with respect to the history available before the reward \(r_s\) is revealed), while \(\eta_s\) remains conditionally \(R\)-sub-Gaussian; thus \(\sum_{s<t}z_s\eta_s\) is the standard martingale transform. We therefore have, with probability at least $1-\delta_e$, simultaneously for all
$t\in\mathfrak{T}_e$,
\[
    \left\|\sum_{s<t}z_s\eta_s\right\|_{V_{e,t}^{-1}}
    \le
    R\sqrt{
        2\log\left(
            \frac{\det(V_{e,t})^{1/2}}
                 {\det(\lambda I_m)^{1/2}\delta_e}
        \right)
    }.
\]
Combining these three bounds proves the theorem.
\end{proof}

\section{Main Regret Analysis}
\label{app:regret-analysis-full}

This section proves the regret bound for the epoch-wise version of \algname{}. We use the representation event $\mathcal E_{\rm rep}$ from Section~\ref{sec:subspace_estimation} and the frozen-epoch confidence sets from Section~\ref{sec:confidence_set_analysis}; the only remaining task is to convert these ingredients into cumulative regret.

Let
\[
    \tau_0:=t_b+1,
    \qquad
    \tau_{e+1}:=2\tau_e,
    \qquad
    \mathfrak T_e:=\{\tau_e,\tau_e+1,\ldots,\min(\tau_{e+1}-1,T)\},
\]
and write $n_e:=|\mathfrak T_e|$. We also write $\epsilon_e:=\epsilon_{\tau_e-1}$, as in the frozen-epoch representation bounds. Let $E$ denote the final epoch index, so $E\le \lceil\log_2T\rceil+1$. Throughout this section, we define
\begin{align}
	R_T:=\sum_{t=1}^T
    \langle X_{t,i_t^\star}-X_{t,i_t},\theta^\star\rangle,
    \qquad
    i_t^\star\in\arg\max_{i\in[K]}\langle X_{t,i},\theta^\star\rangle .
    \label{eq:app-reg-definition}
\end{align}
For each epoch \(e\), let \(\mathcal E_{{\rm conf},e}\) denote the
epoch-wise confidence event defined after
Theorem~\ref{thm:confidence_set_m_t_main}, and let
\[
    \mathcal E_{\rm conf}:=\bigcap_{e=0}^E \mathcal E_{{\rm conf},e}.
\]

We begin by recording the elliptical-potential control for the frozen epochs. This is the ingredient that lets the regret summation proceed as in an ordinary \(m\)-dimensional linear bandit. Recall that \(z_t:=z_{t,i_t}\) denotes the reduced feature of the arm played at round \(t\).

\begin{lemma}[Potential control inside an epoch]
\label{lem:epoch_potential}
Assume \(\mathcal E_{\rm rep}\) holds, fix an epoch \(e\), and suppose that
\(
\lambda\ge B_e^2,
\)
where \(B_e:=B_X\left(1+\left(1+\frac{2}{p}\right)\epsilon_e\right)\). Then, we have
\[
\sum_{t\in\mathfrak T_e}\|z_t\|_{V_{e,t}^{-1}}^2\le 2\Gamma_e
\qquad
\text{and}
\qquad
\sum_{t\in\mathfrak T_e}\|z_t\|_{V_{e,t}^{-1}}
\le
\sqrt{2|\mathfrak T_e|\Gamma_e},
\]
where
\[
\Gamma_e
:=
\log\frac{\det(V_{e,\mathrm{end}})}{\det(\lambda I_m)}
\le
m\log\left(1+\frac{|\mathfrak T_e|B_e^2}{m\lambda}\right),
\qquad
V_{e,\mathrm{end}}:=\lambda I_m+\sum_{s\in\mathfrak T_e}z_s z_s^\top.
\]
\end{lemma}

\begin{proof}
We first begin by showing that the features are uniformly bounded. Since \(\hU_e\) has orthonormal columns,
\[
\|z_{t,i}\|_2
\le
\|\hat X_{t,i}\|_2
\le
\|X_{t,i}\|_2+\|\hat X_{t,i}-X_{t,i}\|_2
\le
B_e,
\]
where we both used the norm bound on \(X_{t,i}\) and the imputation error bound from Lemma~\ref{lemma: imputation_error}. For the potential bound, we proceed by fixing the selected features inside epoch \(e\) and writing
\(
q_t:=\|z_t\|_{V_{e,t}^{-1}}^2.
\)
Indeed what follows is the standard elliptical-potential/determinant-telescoping argument used in linear bandit analyses; see, for example, \citet{abbasi2011improved} or \citet{lattimore2020bandit}.
Because the representation is frozen, the design matrices satisfy
\[
V_{e,t+1}=V_{e,t}+z_tz_t^\top,
\]
and the matrix determinant lemma gives
\[
\log\frac{\det(V_{e,t+1})}{\det(V_{e,t})}
=
\log(1+q_t).
\]
If one takes \(\lambda\ge B_e^2\), then by the norm bound, we have \(q_t\le 1\), hence
\(q_t\le 2\log(1+q_t)\). Summing over the epoch yields
\[
\sum_{t\in\mathfrak T_e}q_t
\le
2\log\frac{\det(V_{e,\mathrm{end}})}{\det(\lambda I_m)}
=2\Gamma_e.
\]
The second inequality follows from Cauchy--Schwarz. Finally,
\[
\det(V_{e,\mathrm{end}})
\le
\left(\frac{\operatorname{tr}(V_{e,\mathrm{end}})}{m}\right)^m
\le
\left(\lambda+\frac{|\mathfrak T_e|B_e^2}{m}\right)^m,
\]
which gives the stated upper bound on \(\Gamma_e\).
\end{proof}

This lemma is where epoch-wise freezing enters the proof algebraically. Because \(\hU_e\) is fixed throughout the epoch, the design matrices evolve by the standard rank-one recursion, so the determinant telescope is the usual OFUL one. The only remaining departure from ordinary OFUL is the controlled misspecification term created by the surrogate approximation.

\subsection{From optimism to epoch regret}

We now use the estimation guarantee of Theorem~\ref{thm:confidence_set_m_t_main} to bound the regret accumulated inside a single epoch. First, on the confidence event \(\mathcal E_{{\rm conf},e}\), the optimistic action selection rule controls the \emph{surrogate} regret of each played arm by the usual OFUL width \(2\beta_{e,t}\|z_t\|_{V_{e,t}^{-1}}\) (Lemma~\ref{lemma:surrogate_optimism}). This is precisely where freezing the representation makes the argument identical in form to a standard \(m\)-dimensional linear bandit calculation. Second, the surrogate approximation guarantee lets us pass from surrogate regret back to true regret at an additive cost of \(2b_e\) per round (Lemma~\ref{lemma:true_vs_surrogate}). Combining these with the potential control of Lemma~\ref{lem:epoch_potential} yields the epoch regret bound (Lemma~\ref{lemma:epoch_regret}), which the proof of Theorem~\ref{thm:main_regret} then sums over the doubling epoch schedule. Along the way, Lemma~\ref{lemma:uniform_epoch_envelopes} records bounds on \(B_e\), \(\Gamma_e\), and \(b_e\) that hold uniformly over the post-burn-in epochs; this technical step simplifies the statement of Lemma~\ref{lemma:epoch_regret} and the final summation.

\begin{lemma}[Surrogate optimism]
\label{lemma:surrogate_optimism}
Fix an epoch $e$. On \(\mathcal E_{{\rm conf},e}\), for every $t\in\mathfrak T_e$,
\[
    \bar\mu_{t,\bar i_t^\star}-\bar\mu_{t,i_t}
    \le
    2\beta_{e,t}\|z_t\|_{V_{e,t}^{-1}},
\]
where $z_t:=z_{t,i_t}$ and
\[
    \bar i_t^\star\in\arg\max_{i\in[K]}\bar\mu_{t,i}.
\]
\end{lemma}

\begin{proof}
We start by observing that, on \(\mathcal E_{{\rm conf},e}\), Theorem~\ref{thm:confidence_set_m_t_main} gives \(\|\hat\vartheta_{e,t}-\vartheta_e^\star\|_{V_{e,t}}\le\beta_{e,t}\). Therefore, by the Cauchy--Schwarz inequality, for any arm $i$, we have
\[
    \bar\mu_{t,i}
    =z_{t,i}^\top\vartheta_e^\star
    \le
    z_{t,i}^\top\widehat\vartheta_{e,t}
    +
    \beta_{e,t}\|z_{t,i}\|_{V_{e,t}^{-1}}.
\]
Similarly, we also have
\[
    \bar\mu_{t,i_t}
    \ge
    z_t^\top\widehat\vartheta_{e,t}
    -
    \beta_{e,t}\|z_t\|_{V_{e,t}^{-1}}.
\]
By the optimistic action selection rule in Algorithm~\ref{alg:bandits}
\[
    z_{t,\bar i_t^\star}^\top\widehat\vartheta_{e,t}
    +
    \beta_{e,t}\|z_{t,\bar i_t^\star}\|_{V_{e,t}^{-1}}
    \le
    z_t^\top\widehat\vartheta_{e,t}
    +
    \beta_{e,t}\|z_t\|_{V_{e,t}^{-1}}.
\]
Combining the three equations  above proves the claim.
\end{proof}

The previous lemma only controls regret in the surrogate model. To return to the
original problem, we use the epoch-wise approximation lemma: every true arm value
and its surrogate value differ by at most \(b_e\). This converts surrogate
optimism into a one-step true regret bound, at the cost of the controlled
misspecification term.

\begin{lemma}[True regret versus surrogate regret]
\label{lemma:true_vs_surrogate}
On \(\mathcal E_{\rm rep}\cap\mathcal E_{{\rm conf},e}\), for every $t\in\mathfrak T_e$,
\[
    \langle X_{t,i_t^\star}-X_{t,i_t},\theta^\star\rangle
    \le
    2\beta_{e,t}\|z_t\|_{V_{e,t}^{-1}}+2b_e.
\]
\end{lemma}

\begin{proof}
We first show that, on \(\mathcal E_{\rm rep}\), the true one-step regret exceeds the surrogate one-step regret by at most \(2b_e\):
\[
    \langle X_{t,i_t^\star}-X_{t,i_t},\theta^\star\rangle
    \le
    \bar\mu_{t,\bar i_t^\star}-\bar\mu_{t,i_t}
    +2b_e.
\]
Towards this, notice that, by Lemma~\ref{lemma:surrogate-model}, we have
\[
    \langle X_{t,i_t^\star},\theta^\star\rangle
    \le
    \bar\mu_{t,i_t^\star}+b_e
    \le
    \bar\mu_{t,\bar i_t^\star}+b_e,
\]
where the second inequality holds since \(\bar i_t^\star\) maximizes the surrogate mean. Similarly, we also have
\[
    \langle X_{t,i_t},\theta^\star\rangle
    \ge
    \bar\mu_{t,i_t}-b_e.
\]
Subtracting the two inequalities above gives the claimed comparison. Finally, on \(\mathcal E_{{\rm conf},e}\), Lemma~\ref{lemma:surrogate_optimism} bounds the surrogate regret by \(2\beta_{e,t}\|z_t\|_{V_{e,t}^{-1}}\), which completes the proof.
\end{proof}

Lemma~\ref{lemma:true_vs_surrogate} is the bridge between the two scales at which the algorithm operates: decisions are made, and the confidence set lives, in the frozen surrogate coordinates, while regret is charged against the true means. The lemma shows that each round of this translation costs two prices: \(2\beta_{e,t}\|z_t\|_{V_{e,t}^{-1}}\) is the familiar OFUL price of parameter uncertainty in \(m\) dimensions, and \(2b_e\) is the price of acting through an estimated and imputed representation. It now remains to sum this one-step bound over the epoch. The OFUL term will be controlled by the elliptical potential of Lemma~\ref{lem:epoch_potential}, while the \(b_e\)-dependent contributions (both the additive \(2b_e\) and the \(b_e\sqrt{t-\tau_e}\) part inside \(\beta_{e,t}\)) will be summed separately. Before carrying this out, we record bounds on certain epoch-level quantities that hold uniformly across all epochs. As we will see below,  this allows the epoch regret bound to be stated with epoch-independent constants.

\begin{lemma}[Uniform bounds across epochs]
\label{lemma:uniform_epoch_envelopes}
Set \(\lambda:=4B_X^2\), and define
\[
G_T
:=
m\log\!\left(1+\frac{T}{m}\right),
\qquad
H_e
:=
2\log(1/\delta_e),
\qquad
A_{\rm rep}
:=
4\sqrt{2}\,C_{\rm sub}S B_X
\frac{\kappa}{p^2}
\sqrt{
\frac{m}{K}
\log\!\left(\frac{8dT}{\delta_{\rm rep}}\right)
},
\]
where \(\delta_e\) is the confidence level at which Theorem~\ref{thm:confidence_set_m_t_main} is invoked in epoch \(e\), and \(\delta_{\rm rep}\) is the failure probability of the representation event \(\mathcal E_{\rm rep}\).
On the representation event \(\mathcal E_{\rm rep}\), the following hold in every epoch \(e\): 
\[
B_e\le B_\star:=2B_X
\quad\text{(so that }\lambda\ge B_e^2\text{)},
\qquad
\Gamma_e\le G_T,
\qquad
b_e\le \frac{A_{\rm rep}}{\sqrt{\tau_e}}.
\]
\end{lemma}

\begin{proof}
We first bound \(B_e\). {Since \(\epsilon_t\) is nonincreasing and \(\tau_e-1\ge t_b\) (recall that \(\epsilon_e:=\epsilon_{\tau_e-1}\)), we have \(\epsilon_e\le\epsilon_{t_b}\), and the choice of \(t_b\) in Equation~\eqref{eq:t_b} (with \(C_b=(32C_{\rm sub})^2\)) guarantees \(\epsilon_{t_b}\le p/32\).} Since \(p\le1\), this gives \(\left(1+\frac{2}{p}\right)\epsilon_{t_b}\le\frac{p}{32}+\frac{2}{32}\le\frac{3}{32}<1\), and therefore
\[
B_e
=
B_X\left(1+\left(1+\frac{2}{p}\right)\epsilon_e\right)
\le
2B_X
=
B_\star.
\]
Next, since we take \(\lambda=4B_X^2\ge B_e^2\), Lemma~\ref{lem:epoch_potential} applies and gives us the bound on \(\Gamma_e\):
\[
\Gamma_e
\le
m\log\!\left(1+\frac{n_eB_e^2}{m\lambda}\right)
\le
m\log\!\left(1+\frac{T\cdot 4B_X^2}{m\cdot 4B_X^2}\right)
=
G_T.
\]
Finally, we bound \(b_e\). {Lemma~\ref{lemma:sub_space_lemma}, invoked at confidence level \(\delta_{\rm rep}/4\), gives
\[
\epsilon_e
=
C_{\rm sub}
\frac{\kappa}{p}
\sqrt{
\frac{m}{(\tau_e-1)K}
\log\!\left(\frac{8dT}{\delta_{\rm rep}}\right)
}.
\]
Since \(\tau_e\ge2\), we have \(\tau_e-1\ge\tau_e/2\), and since \(p\le 1\), we have \(2+2/p\le4/p\); together these give
\[
b_e
=
SB_X\left(2+\frac{2}{p}\right)\epsilon_e
\le
\frac{A_{\rm rep}}{\sqrt{\tau_e}},
\]
as claimed.}
\end{proof}

With these uniform bounds in hand, we can now sum the one-step regret bound of Lemma~\ref{lemma:true_vs_surrogate} over a single epoch.

\begin{lemma}[Epoch regret]
\label{lemma:epoch_regret}
{Set \(\lambda:=4B_X^2\), and let \(G_T\), \(H_e\), and \(A_{\rm rep}\) be as defined in Lemma~\ref{lemma:uniform_epoch_envelopes}.} On \(\mathcal E_{\rm rep}\cap\mathcal E_{{\rm conf},e}\),
\begin{align}
    R_e
    :=
    \sum_{t\in\mathfrak T_e}
    \langle X_{t,i_t^\star}-X_{t,i_t},\theta^\star\rangle
    \label{eq:app-epoch-regret}
\end{align}
satisfies
\begin{align}
    R_e
    \le
	    2\sqrt{2}S\sqrt{\lambda G_T n_e}
	    +
	    2\sqrt{2}R\sqrt{G_T(G_T+H_e) n_e}
	    +
    2A_{\rm rep}(\sqrt{G_T}+1)
    \frac{n_e}{\sqrt{\tau_e}}.
    \label{eq:epoch_regret_uniform}
\end{align}
\end{lemma}

\begin{proof}
Summing the one-step bound of Lemma~\ref{lemma:true_vs_surrogate} over \(t\in\mathfrak T_e\) (recall that \(n_e:=|\mathfrak T_e|\)), we have:
\begin{align}
    R_e
    &\le
    2\sum_{t\in\mathfrak T_e}\beta_{e,t}\|z_t\|_{V_{e,t}^{-1}}
    +2b_en_e.
    \label{eq:epoch_regret_raw}
\end{align}
Now, recalling the definition of the confidence radius \(\beta_{e,t}\) from Theorem~\ref{thm:confidence_set_m_t_main}, we write
\[
\beta_{e,t}
=
w_t+b_e\sqrt{t-\tau_e},
\qquad
{\rm where} \;w_t
:=
\sqrt\lambda S
+
R\sqrt{
    2\log\left(
        \frac{\det(V_{e,t})^{1/2}}
        {\det(\lambda I_m)^{1/2}\delta_e}
    \right)
}.
\]
In what follows, we abbreviate \(q_t:=\|z_t\|_{V_{e,t}^{-1}}\) and split the sum in \eqref{eq:epoch_regret_raw} using the notation above:
\begin{align}
\sum_{t\in\mathfrak T_e}\beta_{e,t}\,q_t
=
\sum_{t\in\mathfrak T_e}w_t\,q_t
+
\sum_{t\in\mathfrak T_e}b_e\sqrt{t-\tau_e}\,q_t.
\label{eq:beta_split}
\end{align}
We will bound the two sums in \eqref{eq:beta_split} in turn. Both bounds rely on the following consequence of Lemma~\ref{lem:epoch_potential}, which applies since our choice of \(\lambda\) (and Lemma~\ref{lemma:uniform_epoch_envelopes}) guarantees \(\lambda\ge B_e^2\), together with the Cauchy--Schwarz inequality:
\begin{align}
    \sum_{t\in\mathfrak T_e}q_t^2
    \le
    2\Gamma_e
    \qquad
    \text{and}
    \qquad
    \sum_{t\in\mathfrak T_e}q_t
    \le
    \sqrt{2n_e\Gamma_e}.
    \label{eq:potential_sums}
\end{align}

We first bound the sum involving \(w_t\), which is the part of the confidence radius that does not grow within the epoch. Since \(V_{e,t}\preceq V_{e,\mathrm{end}}\), we have \(\log\frac{\det(V_{e,t})}{\det(\lambda I_m)}\le\Gamma_e\), and therefore \(w_t\le \sqrt\lambda S+R\sqrt{\Gamma_e+2\log(1/\delta_e)}\) for every \(t\in\mathfrak T_e\). Combining this with the second bound in \eqref{eq:potential_sums}, we have:
\begin{align}
    \sum_{t\in\mathfrak T_e}w_t\,q_t
    \le
    \left(\sqrt\lambda S+R\sqrt{\Gamma_e+2\log(1/\delta_e)}\right)
    \sqrt{2n_e\Gamma_e}.
    \label{eq:first_sum_bound}
\end{align}

We next bound the second sum in \eqref{eq:beta_split}, which collects the growing misspecification part of the radius. By the Cauchy--Schwarz inequality, the first bound in \eqref{eq:potential_sums}, and the fact that \(\sum_{t\in\mathfrak T_e}(t-\tau_e)\le n_e^2/2\), we have:
\begin{align}
    \sum_{t\in\mathfrak T_e}b_e\sqrt{t-\tau_e}\,q_t
    \le
    b_e
    \sqrt{\sum_{t\in\mathfrak T_e}(t-\tau_e)}
    \sqrt{\sum_{t\in\mathfrak T_e}q_t^2}
    \le
    b_en_e\sqrt{\Gamma_e}.
    \label{eq:second_sum_bound}
\end{align}
Substituting \eqref{eq:first_sum_bound} and \eqref{eq:second_sum_bound} into \eqref{eq:beta_split}, and the result into \eqref{eq:epoch_regret_raw}, we have:
\begin{align}
    R_e
    &\le
    2\sqrt{2}S\sqrt{\lambda \Gamma_e n_e}
    +
    2\sqrt{2}R\sqrt{\Gamma_e\left(\Gamma_e+2\log(1/\delta_e)\right)n_e}
    +
    2b_en_e\left(\sqrt{\Gamma_e}+1\right).
    \label{eq:epoch_regret_expanded}
\end{align}
Finally, Lemma~\ref{lemma:uniform_epoch_envelopes} gives \(\Gamma_e\le G_T\), \(2\log(1/\delta_e)=H_e\), and \(b_e\le A_{\rm rep}/\sqrt{\tau_e}\); substituting these three bounds into \eqref{eq:epoch_regret_expanded} yields \eqref{eq:epoch_regret_uniform}.
\end{proof}

\subsection{Main regret theorem}

We are now ready to prove Theorem~\ref{thm:main_regret}. The proof splits the horizon into the burn-in rounds, which we charge at the worst-case rate, and the post-burn-in epochs, to each of which we apply Lemma~\ref{lemma:epoch_regret}; the doubling schedule then makes the epoch bounds summable at the \(\sqrt T\) scale.

\begin{proof}
Let
\(
    \mathcal G:=\mathcal E_{\rm rep}\cap\mathcal E_{\rm conf}
\)
be the event on which the representation guarantees and all epoch-wise
confidence sets hold. We first verify that \(\mathcal G\) has the claimed probability. The representation event \(\mathcal E_{\rm rep}\) is defined in Section~\ref{sec:subspace_estimation} and fails with probability at most \(\delta_{\rm rep}=\delta/2\). Conditional on \(\mathcal E_{\rm rep}\), Theorem~\ref{thm:confidence_set_m_t_main} shows that each epoch-wise confidence event \(\mathcal E_{{\rm conf},e}\) fails with probability at most \(\delta_e=\delta/(2(E+1))\). Since there are at most \(E+1\) epochs, a union bound gives:
\[
    \Pr(\mathcal G)
    \ge
    1-\delta_{\rm rep}-\sum_{e=0}^E\delta_e
    \ge 1-\delta .
\]
It is therefore enough to prove the claimed regret bound on
\(\mathcal G\). 

We now decompose the cumulative regret. Recalling the definition of \(R_T\) in \eqref{eq:app-reg-definition}, we split the sum over rounds into the burn-in rounds \(t\le t_b\) and the post-burn-in epochs \(\mathfrak T_0,\ldots,\mathfrak T_E\), which partition the remaining rounds \(\{t_b+1,\ldots,T\}\); the regret accumulated in epoch \(e\) is exactly the quantity \(R_e\) defined in \eqref{eq:app-epoch-regret}. We handle the burn-in part conservatively, charging every round the worst-case regret: since \(\|X_{t,i}\|_2\le B_X\) and \(\|\theta^\star\|_2\le S\), the Cauchy--Schwarz inequality bounds each instantaneous regret by \(2B_XS\). We therefore have:
\[
    R_T
    \le
    2B_XS\,t_b+\sum_{e=0}^E R_e .
\]
We next bound each \(R_e\). Recall that we set the regularization to \(\lambda=4B_X^2\) in Theorem~\ref{thm:main_regret}, and since we are on \(\mathcal G\), Lemma~\ref{lemma:epoch_regret} applies to every epoch. With the choice of \(\delta_e=\delta/(2(E+1))\), the quantity \(H_e\) in Lemma~\ref{lemma:epoch_regret} is the same for every epoch; write
\[
H_T:=2\log\!\left(\frac{2(E+1)}{\delta}\right).
\]
Thus, on \(\mathcal G\), Lemma~\ref{lemma:epoch_regret} allows us to bound the epoch regret for each epoch \(e\) as follows: 
\begin{align}
    R_e
    &\le
    2\sqrt{2}S\sqrt{\lambda G_T n_e}
    +
    2\sqrt{2}R\sqrt{G_T(G_T+H_T)n_e}
    +
    2A_{\rm rep}(\sqrt{G_T}+1)
    \frac{n_e}{\sqrt{\tau_e}}.
    \label{eq:main_proof_imported_epoch_bound}
\end{align}

Next, we notice that the doubling schedule implies that \(\tau_e=2^e\tau_0\). Since \(\mathfrak T_E\neq\emptyset\), \(\tau_E\le T\), and since \(n_e\le\tau_e\),
\[
\sum_{e=0}^E\sqrt{n_e}
\le
\sum_{e=0}^E\sqrt{\tau_e}
=
\sqrt{\tau_E}\sum_{j=0}^E2^{-j/2}
\le
\frac{\sqrt{\tau_E}}{1-2^{-1/2}}
\le
4\sqrt T,
\]
and
\[
\sum_{e=0}^E\frac{n_e}{\sqrt{\tau_e}}
\le
\sum_{e=0}^E\sqrt{\tau_e}
\le
4\sqrt T.
\]
Summing \eqref{eq:main_proof_imported_epoch_bound} over \(e=0,\ldots,E\), we therefore have:
\begin{align}
\sum_{e=0}^E R_e
&\le
8\sqrt{2}S\sqrt{\lambda G_TT}
+
8\sqrt{2}R\sqrt{G_T(G_T+H_T)T}
+
8A_{\rm rep}(\sqrt{G_T}+1)\sqrt T.
\label{eq:main_regret_post_burnin_explicit}
\end{align}
Combining \eqref{eq:main_regret_post_burnin_explicit} with the burn-in split, we conclude that
\[
R_T
\le
2B_XS\,t_b
+
8\sqrt{2}S\sqrt{\lambda G_TT}
+
8\sqrt{2}R\sqrt{G_T(G_T+H_T)T}
+
8A_{\rm rep}(\sqrt{G_T}+1)\sqrt T.
\]
Finally, we substitute the definitions of the constants involved. By Equation~\eqref{eq:t_b}, the burn-in term satisfies \(2B_XS\,t_b=\widetilde O\big(SB_X\,\kappa^2m/(p^4K)\big)\). Since \(\lambda=4B_X^2\) and \(G_T=m\log(1+T/m)\), the second term is \(8\sqrt2\,S\sqrt{\lambda G_TT}=\widetilde O\big(SB_X\sqrt{mT}\big)\le\widetilde O\big(SB_X\,m\sqrt T\big)\) (using \(m\ge1\)), and the third term is \(8\sqrt2\,R\sqrt{G_T(G_T+H_T)T}=\widetilde O\big(Rm\sqrt T\big)\), since \(G_T+H_T=\widetilde O(m)\). For the last term, substituting the definition of \(A_{\rm rep}\) from Lemma~\ref{lemma:uniform_epoch_envelopes}, we have:
\[
8A_{\rm rep}(\sqrt{G_T}+1)\sqrt T
=
\widetilde O\left(SB_X\frac{\kappa}{p^2}\sqrt{\frac{m}{K}}\cdot\sqrt{m}\cdot\sqrt T\right)
=
\widetilde O\left(SB_X\,\frac{\kappa m\sqrt T}{p^2\sqrt K}\right).
\]
Collecting these contributions gives the three terms stated in Theorem~\ref{thm:main_regret}.

		\end{proof}

\section{Misspecification Control in Frozen Epochs: Why the naive bound loses a factor of \texorpdfstring{\(\sqrt m\)}{sqrt(m)}}
\label{app:bias_control}

This appendix section explains why a direct triangle-inequality bound on the
epoch-wise misspecification error leads to a suboptimal dependence on the
intrinsic dimension \(m\).

Fix an epoch \(e\){, and let \(
\mathfrak T_e=\{\tau_e,\ldots,\tau_e+n_e-1\}.
\) denote the time stamps inside this epoch. Let \(n_e:=|\mathfrak T_e|\) denote its length.

Within this epoch the representation is frozen, and the learner uses
\(m\)-dimensional features \(z_t\in\mathbb R^m\). Define
\[
V_{e,t} := \lambda I_m+\sum_{\substack{s\in\mathfrak T_e\\s<t}}z_sz_s^\top.
\]
As in Section~\ref{sec:confidence_set_analysis}, we suppose that the reward model in the epoch is
\[
r_t=z_t^\top\vartheta_e^\star+\xi_t+\eta_t, \qquad |\xi_t|\le b_e,
\]
where \(\xi_t\) is the controlled representation/imputation misspecification and
\(\eta_t\) is the stochastic reward noise. Then, the contribution of this misspecification to the self-normalized confidence radius is given by
\[
\mathsf B_{e,t} := \left\| \sum_{\substack{s\in\mathfrak T_e\\s<t}}z_s\xi_s
\right\|_{V_{e,t}^{-1}}.
\]

{\bf Naive triangle-inequality control.}
The most direct bound is
\begin{align}
\mathsf B_{e,t}
&\le
\sum_{\substack{s\in\mathfrak T_e\\s<t}}
|\xi_s|\,\|z_s\|_{V_{e,t}^{-1}}
\le
b_e
\sum_{\substack{s\in\mathfrak T_e\\s<t}}
\|z_s\|_{V_{e,t}^{-1}}
\le
b_e
\sum_{\substack{s\in\mathfrak T_e\\s<t}}
\|z_s\|_{V_{e,s}^{-1}},
\label{eq:naive-bias-triangle}
\end{align}
where the last inequality uses \(V_{e,t}\succeq V_{e,s}\) for \(s<t\), hence
\(V_{e,t}^{-1}\preceq V_{e,s}^{-1}\).

{We know that \(\|z_t\|_2\le B_e\) (the exact constant is calculated in the proof of Lemma~\ref{lem:epoch_potential}), and we arrange that \(\lambda\ge B_e^2\).} This implies that \(\|z_t\|_{V_{e,t}^{-1}}^2\le 1\), and the standard elliptical-potential lemma then
gives
\[
\sum_{t\in\mathfrak T_e}\|z_t\|_{V_{e,t}^{-1}}^2
\le
2\Gamma_e,
\qquad
\Gamma_e
:=
\log\frac{\det(V_{e,\mathrm{end}})}{\det(\lambda I_m)}.
\]
By Cauchy--Schwarz and \eqref{eq:naive-bias-triangle}, for any
\(t\in\mathfrak T_e\),
\[
\mathsf B_{e,t}
\le
b_e
\sqrt{
(t-\tau_e)
\sum_{\substack{s\in\mathfrak T_e\\s<t}}
\|z_s\|_{V_{e,s}^{-1}}^2
}
\le
b_e\sqrt{2(t-\tau_e)\Gamma_e}.
\]
In particular, at the end of the epoch,
\(
\mathsf B_{e,\mathrm{end}}
\le
b_e\sqrt{2n_e\Gamma_e}.
\)

Since
\(
\Gamma_e
\le
m\log\left(1+\frac{n_eB_e^2}{m\lambda}\right)
=
\widetilde O(m),
\)
{it can be seen that the naive misspecification bound causes the regret to scale like \(\widetilde{\mathcal{O}} \left(\kappa\,m^{3/2} \sqrt{T}/(p^2\sqrt{K})\right)\). This is worse by a factor of \(\sqrt m\) than the corresponding term \(\widetilde O\big(\kappa m\sqrt T/(p^2\sqrt K)\big)\) in Theorem~\ref{thm:main_regret}. Our analysis instead treats the whole sum together and leverages the contraction
\(
Z_t^\top(\lambda I_m+Z_tZ_t^\top)^{-1}Z_t\preceq I,
\) as shown in the proof of Lemma~\ref{lem:epoch_bias_term},
which gives}
\[
\left\|
\sum_{s<t}z_s\xi_s
\right\|_{V_{e,t}^{-1}}
\le
b_e\sqrt{t-\tau_e}
\]
and avoids the extra \(\sqrt{\Gamma_e}\sim\sqrt m\) factor.

\section{Proof of the Lower Bound}
\label{app:iid_missingness_lower_bound}

We prove Theorem~\ref{thm:iid_missingness_lower_bound_main}. The proof is organized as a sequence of elementary reductions. First we build a four-instance product family. Then we isolate the two pieces of regret: one piece comes from an ordinary reward-learning sign, and the other from a hidden completion sign. Finally, two two-point testing arguments lower bound the lifetime of these two uncertainties. Throughout this section \(c,c_0,c_1,\ldots\) denote positive universal constants whose values may change.

We begin by describing the lower bound model.

{\bf Model.}
At each round $t\in[T]$, the learner is presented with $K$ action vectors
\(
    X_{t,1},\ldots,X_{t,K}\in\mathbb R^d,
\)
drawn independently from a fixed environment-dependent distribution.  Unless otherwise specified in the oracle-augmented setting below, each coordinate of each action is observed independently with probability $p$.  The learner then chooses an arm $i_t\in[K]$ and receives
\[
    Y_t=\langle X_{t,i_t},\theta^\star\rangle+\eta_t,
    \qquad
    \eta_t\sim N(0,R^2),
\]
with independent reward noise.  Regret is measured against the full-information oracle (which only makes our lower bound stronger):
\[
    R_T
    :=
    \sum_{t=1}^T
    \left(
    \max_{i\in[K]}\langle X_{t,i},\theta^\star\rangle
    -
    \langle X_{t,i_t},\theta^\star\rangle
    \right).
\]

{\bf Hard i.i.d. action distribution.}
Assume $K\ge4$.  The construction uses dimension $d=4$, with orthonormal basis
\(
    \{e_s,e_0,e_1,e_2\}.
\)
We further suppose that there are two hidden signs,
\(
    \nu\in\{+1,-1\},
    \;
    \sigma\in\{+1,-1\}.
\)
The sign $\nu$ enforces the usual noisy reward-learning difficulty, while the sign $\sigma$ enforces the missingness-discovery difficulty.
For $\sigma\in\{+1,-1\}$, define
\(
    v_\sigma:=\frac{1}{\sqrt2}(e_1+\sigma e_2),
\)
and the rank-three action subspace
\(
    \mathcal U_\sigma:=\operatorname{span}\{e_s,e_0,v_\sigma\}.
\)
Further, suppose that $\gamma,\Delta>0$ satisfy
\begin{equation}
\label{eq:iid_lb_norm_budget_refined}
    \gamma^2+\Delta^2\le \frac{S^2B_X^2}{4}.
\end{equation}
Then, for environment $(\nu,\sigma)$, if we set
\[
    \theta_{\nu,\sigma}
    :=
    \frac{2\gamma}{B_X}\nu e_s
    +
    \frac{2\Delta}{B_X}\sigma e_0,
\]
then we have $\|\theta_{\nu,\sigma}\|_2\le S$. Now, we are ready to define the arm distribution. With probability $1/8$ for each $(a,b)\in\{+1,-1\}^2$, the arm is a decision arm
\[
    X^{\rm dec}_{a,b}
    :=
    \frac{B_X}{2}(a e_s+b e_0).
\]
With the remaining probability $1/2$, the arm is a side-information arm
\[
    X^{\rm side}
    :=
    \frac{B_X}{2}Wv_\sigma,
\]
where \(W\in\{+1,-1\}\) is an independent Rademacher random variable.  Given our choices, notice that all actions have norm at most \(B_X\) and lie in \(\mathcal U_\sigma\).

We prove the lower bound in an oracle-augmented observation model.  The learner is told whether each arm is a decision arm or a side-information arm.  If an arm is a decision arm, its label $(a,b)$ and its full vector $X^{\rm dec}_{a,b}$ are revealed without masking.  Thus, for decision arms, the Bernoulli coordinate mask is suspended.  If an arm is a side-information arm, its type is revealed, but its coordinates are observed through the usual independent Bernoulli$(p)$ coordinate mask.  This augmentation can only make the learner stronger, so any lower bound in the augmented experiment also holds in the original partially observed model.

The side-information arms have zero expected reward because they lie in $\operatorname{span}\{e_1,e_2\}$, while $\theta_{\nu,\sigma}\in\operatorname{span}\{e_s,e_0\}$.  The learner may know this fact.  {\bf The side arms are useful only because their partially observed coordinates may reveal the hidden sign $\sigma$.}

Let $\mathcal D_t$ be the event that, among the $K$ arms shown in round $t$, all four decision labels $(+,+), (+,-)$, $(-,+),(-,-)$ are present at least once.  Since the arms are drawn i.i.d., {the inclusion--exclusion principle gives us the following expression for this probability:}
\begin{equation}
\label{eq:piK_def}
    \pi_K
    :=
    \mathbb P(\mathcal D_t)
    =
    \sum_{j=0}^4(-1)^j\binom{4}{j}\left(1-\frac{j}{8}\right)^K.
\end{equation}
In particular, {since it suffices that the first four offered arms carry the four distinct labels,} for every $K\ge4$ we have:
\begin{equation}
\label{eq:piK_lower}
    \pi_K\ge 4!\left(\frac18\right)^4=\frac{3}{512}.
\end{equation}
On $\mathcal D_t$, the full-information oracle has access to the decision arm with label $(\nu,\sigma)$, whose mean reward is $\gamma+\Delta$.

For the learner's chosen arm at time $t$, define $\widehat a_t,\widehat b_t\in\{-1,0,+1\}$ as follows.  
 {If the learner chooses a decision arm with label \((a,b)\), we set \(\widehat a_t:=a\) and \(\widehat b_t:=b\); if the learner chooses a side-information arm, we set \(\widehat a_t:=\widehat b_t:=0\).}

\begin{lemma}[Regret on a complete decision round]
\label{lem:iid_regret_complete_round_refined}
For every environment $(\nu,\sigma)$ and every round $t$, on the event $\mathcal D_t$,
\[
    \max_{i\in[K]}\langle X_{t,i},\theta_{\nu,\sigma}\rangle
    -
    \langle X_{t,i_t},\theta_{\nu,\sigma}\rangle
    \ge
    \gamma\mathbf 1\{\widehat a_t\ne \nu\}
    +
    \Delta\mathbf 1\{\widehat b_t\ne \sigma\}.
\]
\end{lemma}

\begin{proof}
The intuition here is that on $\mathcal D_t$, all four decision labels are available.  Therefore the oracle can choose the label matching both hidden signs, $(\nu,\sigma)$.  If the learner chooses a decision arm with the wrong $\nu$-label, it loses at least $\gamma$; if it chooses one with the wrong $\sigma$-label, it loses at least $\Delta$.

{For a decision arm with label \((a,b)\), we have \(\left\langle X^{\rm dec}_{a,b},\theta_{\nu,\sigma}\right\rangle = a\nu\gamma+b\sigma\Delta\).}
On $\mathcal D_t$, the arm with label $(\nu,\sigma)$ is present and has mean reward $\gamma+\Delta$.  If the learner chooses a decision arm $(a,b)$, its regret is
\[
    \gamma+\Delta-a\nu\gamma-b\sigma\Delta
    =
    \gamma(1-a\nu)+\Delta(1-b\sigma).
\]
{Since \(1-a\nu\) equals \(0\) when \(a=\nu\) and \(2\) when \(a\ne\nu\), we have \(\gamma(1-a\nu)\ge\gamma\mathbf 1\{a\ne\nu\}\), and similarly \(\Delta(1-b\sigma)\ge\Delta\mathbf 1\{b\ne\sigma\}\).}
This proves the claim if the learner chooses a decision arm.  If the learner chooses a side-information arm, then its mean reward is zero, while the oracle obtains $\gamma+\Delta$.  Since then $\widehat a_t=\widehat b_t=0$, both indicators equal one, and the same bound holds.
\end{proof}

\begin{lemma}[Testing the reward bit]
\label{lem:iid_nu_testing_refined}
Fix $\sigma\in\{+1,-1\}$.  At any round $t$, conditional on $\mathcal D_t$ and under the uniform prior on $\nu\in\{+1,-1\}$,
\[
    \mathbb P(\widehat a_t\ne \nu\mid \mathcal D_t,\sigma)
    \ge
    \frac14
    \exp\left(-\frac{2(t-1)\gamma^2}{R^2}\right).
\]
\end{lemma}

\begin{proof}
The intuition here is that the sign $\nu$ affects only the reward means through the $e_s$ coordinate.  All contexts, labels, and side-information observations have the same law under $\nu=+1$ and $\nu=-1$.  Thus information about $\nu$ can only accumulate through noisy rewards, and each reward mean changes by at most $2\gamma$ between the two alternatives.

Let $\mathbb P_+$ and $\mathbb P_-$ denote the conditional laws, given $\mathcal D_t$, of the learner's information before choosing at round $t$ under $\nu=+1$ and $\nu=-1$, respectively, with $\sigma$ fixed.  Since $\mathcal D_t$ depends only on the current arm labels, and these labels have the same law under the two values of $\nu$, conditioning on $\mathcal D_t$ introduces no information about $\nu$.

By the chain rule for KL divergence and the Gaussian reward model, we have the following bound:
\begin{equation}
\label{eq:iid_lb_kl_nu}
    \mathrm{KL}(\mathbb P_+\|\mathbb P_-)
    \le
    \sum_{s=1}^{t-1}
    \frac{1}{2R^2}(\mu_s^+-\mu_s^-)^2
    \le
    \frac{2(t-1)\gamma^2}{R^2},
\end{equation}
{where \(\mu_s^+\) and \(\mu_s^-\) are the conditional mean rewards of the arm selected at time \(s\) under \(\nu=+1\) and \(\nu=-1\), and the second inequality follows since every possible selected arm satisfies \(|\mu_s^+-\mu_s^-|\le 2\gamma\).}
Let $A$ be the event $\{\widehat a_t=+1\}$.  Now, we can use the Bretagnolle--Huber inequality~\cite{lattimore2020bandit} and get
\[
    \mathbb P_+(A^c)+\mathbb P_-(A)
    \ge
    \frac12\exp\{-\mathrm{KL}(\mathbb P_+\|\mathbb P_-)\}.
\]
Since $\mathbb P_+(\widehat a_t\ne +1)\ge \mathbb P_+(A^c)$ and $\mathbb P_-(\widehat a_t\ne -1)\ge \mathbb P_-(A)$, averaging the two errors under the uniform prior on $\nu$ {and applying the bound in \eqref{eq:iid_lb_kl_nu}} gives
\[
    \mathbb P(\widehat a_t\ne \nu\mid \mathcal D_t,\sigma)
    \ge
    \frac14
    \exp\left(-\frac{2(t-1)\gamma^2}{R^2}\right).
\]
\end{proof}

\begin{lemma}[Testing the missingness bit]
\label{lem:iid_sigma_testing_refined}
Assume $p\le 1/2$.  Fix $\nu\in\{+1,-1\}$.  At any round $t$, conditional on $\mathcal D_t$ and under the uniform prior on $\sigma\in\{+1,-1\}$,
\[
    \mathbb P(\widehat b_t\ne \sigma\mid \mathcal D_t,\nu)
    \ge
    \frac14
    \exp\left(
    -2Kp^2t
    -\frac{2(t-1)\Delta^2}{R^2}
    \right).
\]
\end{lemma}

\begin{proof}

Notice that the sign $\sigma$ can be learned in two ways.  First, a side-information arm may reveal both coordinates $e_1$ and $e_2$; then the relative sign of the two observed entries reveals $\sigma$.  This occurs at rate proportional to $Kp^2$.  Second, rewards from decision arms carry information about $\sigma$, at rate proportional to $\Delta^2/R^2$.  If neither source has provided enough information, the learner cannot reliably choose the correct $\sigma$-label. {To leverage this intuition, we proceed as follows. We first lower bound the probability that no side-information arm reveals both coordinates. Then,  we argue that, on this event, only the rewards carry information about \(\sigma\). This then allows us to finally use a hypothesis-testing argument to conclude the proof.}

Let $\mathbb P_+$ and $\mathbb P_-$ denote the conditional laws, given $\mathcal D_t$, of the learner's information before choosing at round $t$ under $\sigma=+1$ and $\sigma=-1$, respectively, with $\nu$ fixed.

Let $\mathcal N_t$ be the event that, throughout rounds $1,\ldots,t$, no side-information arm has both coordinates $e_1$ and $e_2$ observed.  This includes side-information arms in the current round, because current action observations are available to the learner before it chooses.  Conditional on $\mathcal D_t$, the probability of $\mathcal N_t$ is bounded below as follows.  In each of the first $t-1$ rounds, each arm is a side-information arm with probability $1/2$, and conditional on being a side-information arm, it reveals both coordinates $e_1$ and $e_2$ with probability $p^2$.  Hence no revealing side-information arm occurs in the first $t-1$ rounds with probability $(1-p^2/2)^{K(t-1)}$.  In the current round, conditional on $\mathcal D_t$, there are at most $K$ side-information arms, and each such arm reveals both coordinates with probability $p^2$.  We therefore have:
\begin{align}
    \mathbb P(\mathcal N_t\mid \mathcal D_t)
    &\ge
    (1-p^2/2)^{K(t-1)}(1-p^2)^K \ge
    \exp(-2Kp^2t).
    \label{eq:no_reveal_lower_refined}
\end{align}
{Here, the second inequality follows because \(p\le1/2\): both \(p^2/2\) and \(p^2\) are then at most \(1/2\), so that \(1-x\ge \exp(-2x)\) applies with \(x\in[0,1/2]\).}

On the event $\mathcal N_t$, the non-reward observations have the same distribution under $\sigma=+1$ and $\sigma=-1$.  Indeed, if a side-information arm reveals neither or only one of the coordinates $e_1,e_2$, then the observed value has the same distribution under both signs because the side arm contains an independent Rademacher multiplier \(W\).  Decision-arm labels and vectors are independent of $\sigma$ in the oracle-augmented observation model.

Now condition on $\mathcal D_t\cap\mathcal N_t$.  Under this conditioning, the only remaining difference between the two signs comes from reward observations.  By the chain rule for KL divergence and the Gaussian reward model, we have the following bound:
\begin{equation}
\label{eq:iid_lb_kl_sigma}
    \mathrm{KL}(\mathbb P_+^{\mathcal N}\|\mathbb P_-^{\mathcal N})
    \le
    \sum_{s=1}^{t-1}
    \frac{1}{2R^2}(\mu_s^+-\mu_s^-)^2
    \le
    \frac{2(t-1)\Delta^2}{R^2},
\end{equation}
{where \(\mathbb P_+^{\mathcal N}\) and \(\mathbb P_-^{\mathcal N}\) denote the conditional laws given \(\mathcal D_t\cap\mathcal N_t\), and the second inequality follows since every possible selected arm satisfies \(|\mu_s^+-\mu_s^-|\le 2\Delta\).}
Let $A$ be the event $\{\widehat b_t=+1\}$.  Again, we invoke the Bretagnolle--Huber inequality~\cite{lattimore2020bandit}, applied conditionally on $\mathcal D_t\cap\mathcal N_t$ {together with the bound in \eqref{eq:iid_lb_kl_sigma}}, to get
\[
    \mathbb P_+^{\mathcal N}(A^c)+\mathbb P_-^{\mathcal N}(A)
    \ge
    \frac12
    \exp\left(-\frac{2(t-1)\Delta^2}{R^2}\right).
\]
Multiplying by the common lower bound \eqref{eq:no_reveal_lower_refined} for the probability of $\mathcal N_t$ conditional on $\mathcal D_t$, and then averaging the two signs under the uniform prior on $\sigma$, yields
\[
    \mathbb P(\widehat b_t\ne \sigma\mid \mathcal D_t,\nu)
    \ge
    \frac14
    \exp\left(
    -2Kp^2t
    -\frac{2(t-1)\Delta^2}{R^2}
    \right).
\]
\end{proof}

\begin{theorem}[i.i.d. action-set lower bound]
\label{thm:iid_action_lower_bound_refined}
Assume $K\ge4$, $p\le1/2$, and let $\gamma,\Delta>0$ satisfy \eqref{eq:iid_lb_norm_budget_refined}.  For the i.i.d. action-set instance above, every algorithm satisfies
\[
    \max_{\nu,\sigma\in\{\pm1\}}
    \mathbb E_{\nu,\sigma}R_T
    \ge
    \frac{\pi_K\gamma}{4}
    \sum_{t=1}^T
    \exp\left(-\frac{2(t-1)\gamma^2}{R^2}\right)
    +
    \frac{\pi_K\Delta}{4}
    \sum_{t=1}^T
    \exp\left(
    -2Kp^2t
    -\frac{2(t-1)\Delta^2}{R^2}
    \right),
\]
where $\pi_K$ is defined in \eqref{eq:piK_def}.  
{Consequently, there is a universal constant \(c>0\) such that, whenever \(Kp^2+S^2B_X^2/R^2\le c^{-1}\), the following optimized bound holds:}
\[
    \sup_{\textnormal{valid instances}}
    \mathbb E R_T
    \ge
    c\min\{SB_XT,\,R\sqrt T\}
    +
    cSB_X
    \min\left\{
    T,
    \frac{1}{Kp^2+S^2B_X^2/R^2}
    \right\}.
\]
{In particular, in the context-limited regime \(R\ge SB_X/(p\sqrt K)\), one obtains}
\[
    \sup_{\textnormal{valid instances}}
    \mathbb E R_T
    \ge
    c\min\{SB_XT,\,R\sqrt T\}
    +
    cSB_X
    \min\left\{
    T,
    \frac{1}{Kp^2}
    \right\}.
\]
\end{theorem}

\begin{proof}
The proof of this theorem puts the above pieces together. Notice that our construction allows us to add the contributions because of the two independent hidden bits. The bit $\nu$ is hard to learn only through noisy rewards, producing the standard stochastic-bandit term.  The bit $\sigma$ is hard to learn until a side-information arm reveals both relevant coordinates, which happens at rate $Kp^2$, or until rewards reveal it at rate $\Delta^2/R^2$. {The proof has two stages: we first establish a per-round regret bound and sum it over \(t\), and we then optimize over the parameters \(\gamma\) and \(\Delta\).}

We begin by placing a uniform prior on $(\nu,\sigma)\in\{\pm1\}^2$.  By Lemma~\ref{lem:iid_regret_complete_round_refined}, on $\mathcal D_t$, we have the following:
\[
    \textnormal{regret}_t
    \ge
    \gamma\mathbf 1\{\widehat a_t\ne\nu\}
    +
    \Delta\mathbf 1\{\widehat b_t\ne\sigma\}.
\]
Notice that the event $\mathcal D_t$ depends only on the current arm labels and is independent of the hidden signs.  It has probability $\pi_K$ (as defined in \eqref{eq:piK_def}).  {Moreover, the instantaneous regret is nonnegative on every round (the oracle maximizes over the offered set), so restricting attention to the rounds on which \(\mathcal D_t\) holds can only decrease the total.} Taking expectation under the uniform prior and using Lemmas~\ref{lem:iid_nu_testing_refined} and \ref{lem:iid_sigma_testing_refined} gives us the following:
\begin{equation}
\label{eq:iid_lb_per_round}
    \mathbb E[\textnormal{regret}_t]
    \ge
    \frac{\pi_K\gamma}{4}
    \exp\left(-\frac{2(t-1)\gamma^2}{R^2}\right)
    +
    \frac{\pi_K\Delta}{4}
    \exp\left(
    -2Kp^2t
    -\frac{2(t-1)\Delta^2}{R^2}
    \right).
\end{equation}
{We now sum \eqref{eq:iid_lb_per_round} over \(t=1,\ldots,T\) to obtain the first claim.}  Since the maximum over environments is at least the Bayes average, the same lower bound holds for $\max_{\nu,\sigma}\mathbb E_{\nu,
\sigma}R_T$.

Next, for the optimized form, we choose
\[
    \Delta=c_0SB_X,
    \qquad
    \gamma=c_0\min\{SB_X,R/\sqrt T\},
\]
with $c_0>0$ sufficiently small so that \eqref{eq:iid_lb_norm_budget_refined} holds.  Since $\pi_K\ge3/512$ by \eqref{eq:piK_lower}, the first exponential sum may be bounded as follows:
\[
    \gamma\sum_{t=1}^T
    \exp\left(-\frac{2(t-1)\gamma^2}{R^2}\right)
    \ge
    c\min\{SB_XT,R\sqrt T\}.
\]
For the second sum, let
\(
    a:=2Kp^2+2\Delta^2/R^2.
\)
Notice that, if $a\le1$, we have
\(
    \sum_{t=1}^T e^{-at}
    \ge
    c\min\{T,1/a\}.
\)
Therefore, using $\Delta=c_0SB_X$, we obtain the following bound:
\begin{equation}
\label{eq:iid_lb_second_sum}
    \Delta\sum_{t=1}^T
    \exp\left(-2Kp^2t-\frac{2(t-1)\Delta^2}{R^2}\right)
    \ge
    cSB_X
    \min\left\{
    T,
    \frac{1}{Kp^2+S^2B_X^2/R^2}
    \right\}.
\end{equation}
This proves the optimized bound under the stated nontriviality condition.  If $R\ge SB_X/(p\sqrt K)$, then $S^2B_X^2/R^2\le Kp^2$, and {Equation~\eqref{eq:iid_lb_second_sum} implies the ``context-limited'' corollary.}
\end{proof}

\section{Adaptivity to Unknown Subspace Dimensionality}
\label{app:rank-adaptivity}

We now give the full argument for the rank-adaptive version of
\algname{}. The known-rank proof assumes that the learner is told the latent
dimension \(m\), so that each epoch uses the top \(m\) directions of the
corrected covariance estimate. When \(m\) is unknown, our rank-adaptive algorithm first
estimates the spectrum of the corrected covariance and keeps only eigenvalues
that are separated from the noise floor. The proof below shows that, after a
finite ``rank-identification'' time, this thresholding rule selects exactly the
signal subspace. From that epoch onward the algorithm is identical to the
known-rank procedure, and all regret before this point is charged by a worst-case
bound.

The argument below has three parts. First, we define a uniform covariance
perturbation event that controls the empirical spectrum at all epoch starts.
Second, we show that a simple spectral threshold recovers the correct rank once
the perturbation radius is below the population eigengap. Third, we combine rank
identification with the burn-in condition needed for imputation, and then reuse
the known-rank epoch regret template. We first begin with the spectral event used to separate signal eigenvalues
from null directions.

{\bf Covariance perturbation event.}
Let
\(
\Sigma:=\E[X_{t,i}X_{t,i}^{\top}]
\)
denote the population covariance of the ideal action vectors. Let its eigenvalues be
\[
\lambda_1(\Sigma)\ge \lambda_2(\Sigma)\ge\cdots\ge\lambda_d(\Sigma).
\]
Recall that we suppose that the rank of $\Sigma$ is $m$, and therefore, we have that
\(
\lambda_m(\Sigma)>0\), while
\(\lambda_{m+1}(\Sigma)=0.
\)
To avoid overloading notation, define the population eigengap
\[
\Delta_m:=\lambda_m(\Sigma)-\lambda_{m+1}(\Sigma)=\lambda_m(\Sigma).
\]

Let \(\dot\Sigma_t\) be the unbiased covariance estimator in Equation~\eqref{eq:unbiased_est_cov}, and let
\[
\hat\lambda_{t,1}\ge \hat\lambda_{t,2}\ge\cdots\ge \hat\lambda_{t,d}
\]

be its eigenvalues. For a target failure probability \(\delta_{\rm rank}\), define
\begin{equation}
\label{eq:rho_unknown_rank}
\rho_t
:=
2B_X\sqrt{\frac{\bar\lambda}{p^2\,tK}\log\frac{2dT}{\delta_{\rm rank}}}
+\frac{2B_X^2}{p^2\,tK}\log\frac{2dT}{\delta_{\rm rank}},
\end{equation}
which is exactly the high-probability bound of Lemma~\ref{lemma:covariance_estimation_error} at level \(\delta_{\rm rank}/T\). Note that \(\rho_t\) is computable from \((B_X,p,K)\) alone (recall that one may always take \(\bar\lambda=B_X^2\)); in particular, no knowledge of upper bound on the rank \(m\) is required. In what follows, we work on the event
\begin{equation}
\label{eq:cov_event_unknown_rank}
\mathcal G_{\rm rank}
:=
\left\{
\|\dot\Sigma_t-\Sigma\|_2\le \rho_t
\quad
\text{for all }t\in[T]
\right\}.
\end{equation}
Indeed, applying Lemma~\ref{lemma:covariance_estimation_error} at level \(\delta_{\rm rank}/T\) and taking a union bound over \(t\in[T]\), we have that
\[
\Pr(\mathcal G_{\rm rank})\ge 1-\delta_{\rm rank}.
\]

On this event, the empirical eigenvalues are uniformly close to the
population eigenvalues. The rank selector below keeps precisely those empirical
directions whose eigenvalues exceed twice this perturbation radius.

{\bf Rank selector.}
At the beginning of epoch \(e\), define
\begin{equation}
\label{eq:unknown_rank_selector}
\hat m_e
:=
\#\left\{
j\in[d]:
\hat\lambda_{\tau_e-1,j}\ge 2\rho_{\tau_e-1}
\right\}.
\end{equation}
The rank-adaptive algorithm uses the top \(\hat m_e\) eigenvectors of \(\dot\Sigma_{\tau_e-1}\) to form \(\hU_e\). If \(\hat m_e=0\), the algorithm may use any admissible fallback policy in that epoch; the regret before \(m\) will be bounded conservatively.

Define the rank-identification time
\begin{equation}
\label{eq:t_rank_explicit}
t_{\rm rank}
:=
\left\lceil
\frac{
512\,B_X^2\bar\lambda
}{
\Delta_m^2p^2K
}
\log\!\left(\frac{2dT}{\delta_{\rm rank}}\right)
\right\rceil
=
\left\lceil
\frac{512\,\kappa^2 m}{p^2K}
\log\!\left(\frac{2dT}{\delta_{\rm rank}}\right)
\right\rceil,
\end{equation}
where the second expression follows since \(\Delta_m=\lambda_m\) and \(\kappa^2m=B_X^2\bar\lambda/\lambda_m^2\). This is the form quoted in Theorem~\ref{thm:unknown_rank}. We now show that, for every \(t\ge t_{\rm rank}\),
\begin{equation}
\label{eq:rho_after_t_rank}
\rho_t\le \frac{\Delta_m}{4}.
\end{equation}
Since \(tK\ge512\,B_X^2\bar\lambda\log(2dT/\delta_{\rm rank})/(\Delta_m^2p^2)\), the square-root term of \eqref{eq:rho_unknown_rank} is at most \(2\Delta_m/\sqrt{512}\le\Delta_m/8\), and the linear term is at most \(\Delta_m^2/(256\,\bar\lambda)\le\Delta_m/8\), where we used \(\Delta_m=\lambda_m\le\bar\lambda\). Adding the two terms gives \eqref{eq:rho_after_t_rank}.

The next lemma formalizes the separation argument. Once
\(\rho_{\tau_e-1}\le \Delta_m/4\), every true signal eigenvalue remains above the
threshold \(2\rho_{\tau_e-1}\), while every null eigenvalue remains below it.

\begin{lemma}[Rank identification by spectral thresholding]
\label{lemma:unknown_rank_identification}
On the event \(\mathcal G_{\rm rank}\), for every epoch \(e\) with \(\tau_e-1\ge t_{\rm rank}\), the selector in Equation~\eqref{eq:unknown_rank_selector} recovers the true rank. That is 
\(
\hat m_e=m.
\)
\end{lemma}

\begin{proof}
Fix an epoch \(e\) with \(\tau_e-1\ge t_{\rm rank}\). We will show that, on \(\mathcal G_{\rm rank}\), every signal eigenvalue (\(j\le m\)) clears the threshold \(2\rho_{\tau_e-1}\), while every null eigenvalue (\(j>m\)) falls below it, so that the selector in Equation~\eqref{eq:unknown_rank_selector} counts exactly \(m\) directions.

We begin by transferring the covariance perturbation to the eigenvalues. On \(\mathcal G_{\rm rank}\), we have \(\|\dot\Sigma_{\tau_e-1}-\Sigma\|_2\le\rho_{\tau_e-1}\), and therefore, by Weyl's inequality~\cite{horn2012matrix}, we have the following inequality for all \(j\in[d]\):
\[
|\hat\lambda_{\tau_e-1,j}-\lambda_j(\Sigma)|
\le
\rho_{\tau_e-1}.
\]
Moreover, since \(\tau_e-1\ge t_{\rm rank}\), Equation~\eqref{eq:rho_after_t_rank} gives \(\rho_{\tau_e-1}\le\Delta_m/4\).

First, we consider the signal eigenvalues, \(j\le m\). Since \(\lambda_j(\Sigma)\ge \lambda_m(\Sigma)=\Delta_m\), we may chain the two preceding bounds to obtain:
\[
\hat\lambda_{\tau_e-1,j}
\ge
\lambda_j(\Sigma)-\rho_{\tau_e-1}
\ge
\Delta_m-\frac{\Delta_m}{4}
=
\frac{3\Delta_m}{4}
>
\frac{\Delta_m}{2}
\ge
2\rho_{\tau_e-1}.
\]
Therefore, every signal eigenvalue is selected.

Next, we consider the null eigenvalues, \(j>m\). Since \(\lambda_j(\Sigma)=0\), the same perturbation bound gives:
\[
\hat\lambda_{\tau_e-1,j}
\le
\rho_{\tau_e-1}
<
2\rho_{\tau_e-1},
\]
so no null eigenvalue is selected. Putting the two cases together, exactly \(m\) empirical eigenvalues exceed the threshold \(2\rho_{\tau_e-1}\), and hence \(\hat m_e=m\).
\end{proof}

Once Lemma~\ref{lemma:unknown_rank_identification} has identified the rank, the
remaining representation and imputation guarantees exactly follow the known-rank
ones (conditional on the event identified above). The only bookkeeping is to wait until both prerequisites hold: the rank
must be identified, and the burn-in condition for stable imputation must have
passed. This is why we introduce the synchronization time
\(t_{\rm id}:=\max\{t_b,t_{\rm rank}\}\).

\begin{lemma}[Representation event after rank identification]
\label{lemma:unknown_rank_representation}
Let
\(
t_{\rm id}:=\max\{t_b,t_{\rm rank}\},
\)
and let \(\mathcal G_{\rm rep}\) denote the subspace and imputation good event of Lemma~\ref{lemma:sub_space_lemma} and Lemma~\ref{lemma: imputation_error}. On the event
\(
\mathcal G_{\rm rank}\cap \mathcal G_{\rm rep},
\)
every epoch \(e\) with \(\tau_e-1\ge t_{\rm id}\) satisfies \(\hat m_e=m\). Consequently, the adaptive basis \(\hU_e\) coincides with the known-rank epoch basis, and the subspace and imputation guarantees of Lemma~\ref{lemma:sub_space_lemma} and Lemma~\ref{lemma: imputation_error} hold for epoch \(e\) as stated.
\end{lemma}

\begin{proof}
Fix an epoch \(e\) with \(\tau_e-1\ge t_{\rm id}\). Since
\(t_{\rm id}\ge t_{\rm rank}\), Lemma~\ref{lemma:unknown_rank_identification}
gives
\(
\hat m_e=m,
\)
so the adaptive basis \(\hU_e\) is the top-\(m\) eigenspace of \(\dot\Sigma_{\tau_e-1}\); this is exactly the basis that the known-rank algorithm would use at the start of the epoch. Since
\(t_{\rm id}\ge t_b\), we also have \(\tau_e-1\ge t_b\), and therefore, on
\(\mathcal G_{\rm rep}\), the conclusions of Lemma~\ref{lemma:sub_space_lemma} and of Lemma~\ref{lemma: imputation_error} apply verbatim to \(\hU_e\).
\end{proof}

It remains to account for the few epochs before both conditions hold.
We isolate the first epoch whose start time is beyond \(t_{\rm id}\); all
earlier rounds will be charged directly, and all later epochs can use the
known-rank analysis.

{\bf First correctly ranked epoch.}
Let
\[
e_\star
:=
\min\{e:\tau_e-1\ge t_{\rm id}\},
\qquad
\tau_\star:=\tau_{e_\star}.
\]
Since the epoch starts are on a doubling schedule, we have that
\begin{equation}
\label{eq:tau_star_bound}
\tau_\star\le 2(t_{\rm id}+1).
\end{equation}
Indeed, if \(\tau_\star\) is the first epoch start at least \(t_{\rm id}\), then the previous epoch start, if it exists, is smaller than \(t_{\rm id}\), and the next epoch start is twice the previous one. The additive \(1\) covers the case where \(t_{\rm id}\) is below the first epoch start. We can now state the explicit regret theorem. Let
\[
E_T:=\lceil \log_2 T\rceil+1,
\qquad
\delta_e:=\frac{\delta_{\rm oful}}{E_T}.
\]
As in the known-rank analysis, we set the regularization to \(\lambda:=4B_X^2\) and define
\[
G_T
:=
m\log\!\left(1+\frac{T}{m}\right),
\qquad
H_T
:=
2\log\!\left(\frac{E_T}{\delta_{\rm oful}}\right),
\qquad
A_{\rm rep}
:=
4\sqrt{2}\,C_{\rm sub}SB_X
\frac{\kappa}{p^2}
\sqrt{
\frac{m}{K}
\log\!\left(\frac{8dT}{\delta_{\rm rep}}\right)
},
\]
matching Lemma~\ref{lemma:uniform_epoch_envelopes}.
By Lemma~\ref{lemma:sub_space_lemma},
for every correctly ranked epoch \(e\),
\begin{equation}
\label{eq:be_unknown_rank_bound}
b_e
:=
SB_X\left(2+\frac{2}{p}\right)\epsilon_{\tau_e-1}
\le
\frac{A_{\rm rep}}{\sqrt{\tau_e}}.
\end{equation}
{The inequality uses \(p\le 1\), so \(2+2/p\le 4/p\), together with \(\tau_e-1\ge\tau_e/2\), exactly as in the proof of Lemma~\ref{lemma:uniform_epoch_envelopes}.}

We can now state the regret bound. The theorem is the same
known-rank epoch summation, with one additional cost for the rounds before the
first correctly ranked and stably imputable epoch.

\begin{theorem}[Restatement of Theorem~\ref{thm:unknown_rank}, with explicit constants]
Assume the hypotheses of Theorem~\ref{thm:main_regret}, except that \(m\) is not known to the algorithm. Let the rank-adaptive algorithm use the selector in Equation~\eqref{eq:unknown_rank_selector}. Suppose
\[
\Pr(\mathcal G_{\rm rank})\ge 1-\delta_{\rm rank},
\qquad
\Pr(\mathcal G_{\rm rep})\ge 1-\delta_{\rm rep},
\]
and allocate the OFUL confidence probabilities as above. Then, with probability at least
\(
1-\delta_{\rm rank}-\delta_{\rm rep}-\delta_{\rm oful},
\)
the regret of the rank-adaptive epoch-wise algorithm satisfies
\begin{align}
R_T
&\le
4B_XS(t_{\rm id}+1)
+
8\sqrt{2}S\sqrt{\lambda G_TT}
+
8\sqrt{2}R\sqrt{G_T(G_T+H_T)T}
+
8A_{\rm rep}(\sqrt{G_T}+1)\sqrt T.
\label{eq:unknown_rank_explicit_regret}
\end{align}
Consequently, up to logarithmic factors,
\[
R_T
\le
O(B_XS\,t_{\rm id})
+
\widetilde O\!\left((R+SB_X)\,m\sqrt T\right)
+
\widetilde O\!\left(
SB_X
\frac{\kappa m\sqrt T}{p^2\sqrt K}
\right),
\]
where \(t_{\rm id}=\max\{t_b,\,t_{\rm rank}\}\), with \(t_{\rm rank}\) as in Equation~\eqref{eq:t_rank_explicit}.
\end{theorem}

\begin{proof}
The proof follows that of Theorem~\ref{thm:main_regret} in Appendix~\ref{app:regret-analysis-full}; we describe the two modifications. First, the good event additionally includes the rank event: we work on
\(
\mathcal G_{\rm rank}\cap \mathcal G_{\rm rep}\cap \mathcal G_{\rm oful},
\)
where \(\mathcal G_{\rm oful}\) is the intersection of all epoch-wise OFUL confidence events. Since
\(
\sum_{e=0}^{E_T-1}\delta_e
=
\delta_{\rm oful},
\)
the same union bound as before gives
\[
\Pr(\mathcal G_{\rm rank}\cap \mathcal G_{\rm rep}\cap \mathcal G_{\rm oful})
\ge
1-\delta_{\rm rank}-\delta_{\rm rep}-\delta_{\rm oful}.
\]

Second, the worst-case portion of the horizon is potentially longer: rather than only the burn-in rounds, we charge every round before the first correctly ranked and stably imputable epoch at the worst-case rate. Let
\[
e_\star:=\min\{e:\tau_e\ge t_{\rm id}\},
\qquad
\tau_\star:=\tau_{e_\star}.
\]
Before \(\tau_\star\), the rank may be wrong or imputation may not yet
be stable. Since each one-step regret is
at most \(2B_XS\) (by the Cauchy--Schwarz inequality, as before), Equation~\eqref{eq:tau_star_bound} gives
\[
\sum_{t<\tau_\star}(\mu_{t,i_t^\star}-\mu_{t,i_t})
\le
2B_XS\,\tau_\star
\le
4B_XS(t_{\rm id}+1).
\]

From \(\tau_\star\) onward, the analysis is identical to the known-rank case. For every epoch \(e\ge e_\star\),
Lemma~\ref{lemma:unknown_rank_representation} shows that the adaptive algorithm
uses the correct \(m\)-dimensional representation with the same subspace
and imputation guarantees as in the known-rank proof, so
Lemma~\ref{lemma:epoch_regret} applies with \(H_e=H_T\) and yields exactly the epoch bound \eqref{eq:main_proof_imported_epoch_bound} from the proof of Theorem~\ref{thm:main_regret}. The summation over epochs is also unchanged: the doubling-schedule bounds \(\sum_{e\ge e_\star}\sqrt{n_e}\le 4\sqrt T\) and \(\sum_{e\ge e_\star}n_e/\sqrt{\tau_e}\le 4\sqrt T\), established in that proof, hold verbatim here since the sums run over a subset of the epochs. We therefore have:
\begin{align*}
\sum_{e\ge e_\star}R_e
&\le
8\sqrt{2}S\sqrt{\lambda G_TT}
+
8\sqrt{2}R\sqrt{G_T(G_T+H_T)T}
+
8A_{\rm rep}(\sqrt{G_T}+1)\sqrt T.
\end{align*}
Adding the pre-identification regret \(4B_XS(t_{\rm id}+1)\) establishes
Equation~\eqref{eq:unknown_rank_explicit_regret}.
\end{proof}

{\bf Remark. }
The theorem makes the additional cost of unknown rank explicit. Rank identification requires
\[
\rho_t\le \Delta_m/4,
\]
which yields the time \(t_{\rm rank}\) in Equation~\eqref{eq:t_rank_explicit}. After the first epoch beginning after \(t_{\rm id}=\max\{t_b,t_{\rm rank}\}\), the algorithm is identical to the known-rank epoch-wise method. Moreover, comparing Equations~\eqref{eq:t_rank_explicit} and~\eqref{eq:t_b}, we have \(t_{\rm rank}=\widetilde O\big(\kappa^2m/(p^2K)\big)\) while \(t_b=\widetilde O\big(\kappa^2m/(p^4K)\big)\), so \(t_{\rm rank}\le t_b\) up to constants whenever \(\delta_{\rm rank}\) and \(\delta\) are of the same order. Hence \(t_{\rm id}=t_b\) up to constants, and the identification cost is absorbed by the imputation burn-in: stable imputation requires a stronger representation condition (accuracy at scale \(p\)) than merely separating the nonzero and zero eigenvalues of the covariance matrix.

\section{Additional Experimental Details}
\label{app:experimental_details}

\subsection{Additional Experimental Diagnostics}
This appendix collects additional experiments supporting Section~\ref{sec:simulations}: the full-history synthetic counterpart to the main synthetic experiment, real-feature synthetic tasks using optical digit covariates~\citep{alpaydin1998optical}, rank recovery and rank-misspecification diagnostics, warm-start comparisons, MNIST product-context diagnostics~\citep{lecun1998gradient}, and a text product-context experiment using 20 Newsgroups~\citep{lang1995newsweeder} with an approximately low-rank nuisance tail.

The code and scripts for reproducing these results are available at: \\ \url{https://github.com/gautamdasarathy/tofu-pov-arxiv}.

Figure~\ref{fig:app_synthetic_fh} repeats the controlled synthetic experiment from Figure~\ref{fig:synthetic_main_experiments} using the full-history replay variants. These are the practical variants used in the real-feature experiments in the main text. As expected, replaying past rewards after each frozen representation update lowers regret relative to the restart version while preserving the same missingness trend.

\begin{figure*}[!htb]
    \centering
    \begin{subfigure}[b]{0.89\textwidth}
        \includegraphics[width=\textwidth]{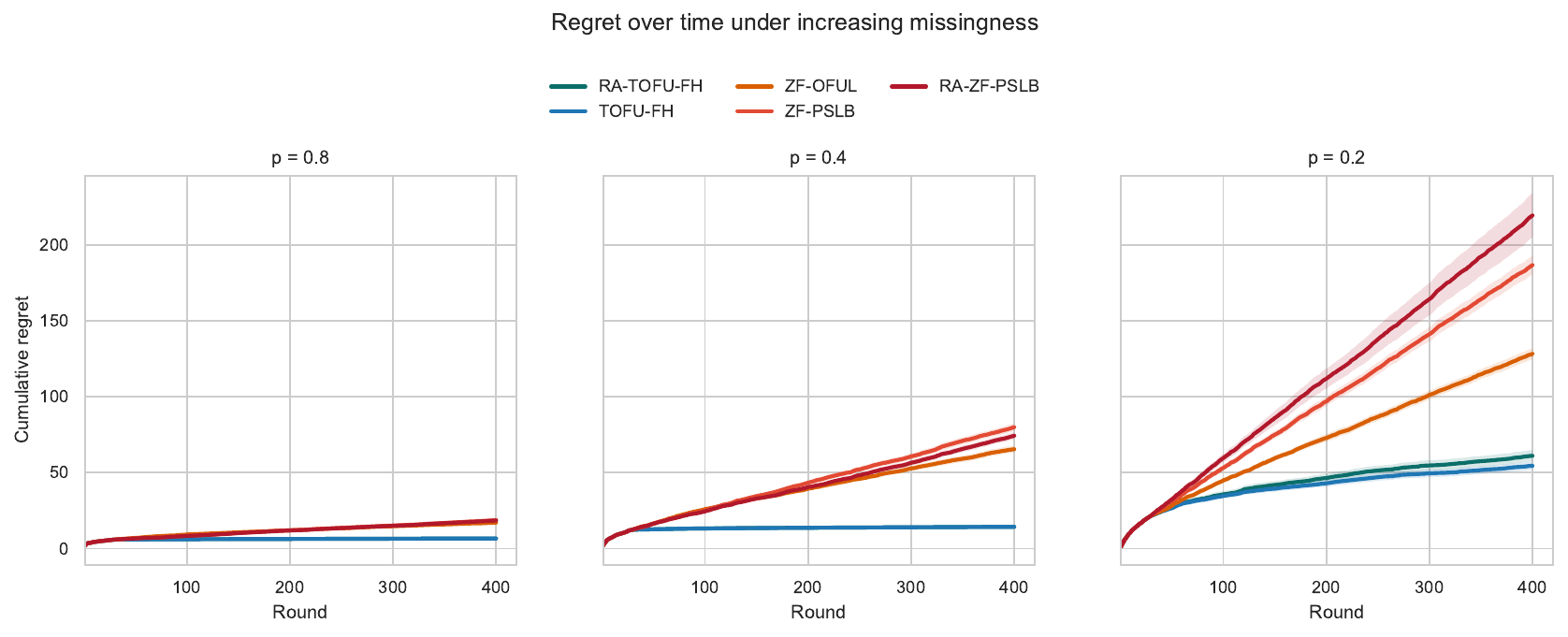}
        \caption{Regret over time.}
        \label{fig:app_synthetic_fh_regret}
    \end{subfigure}
    \hfill
    \begin{subfigure}[b]{0.58\textwidth}
        \includegraphics[width=\textwidth]{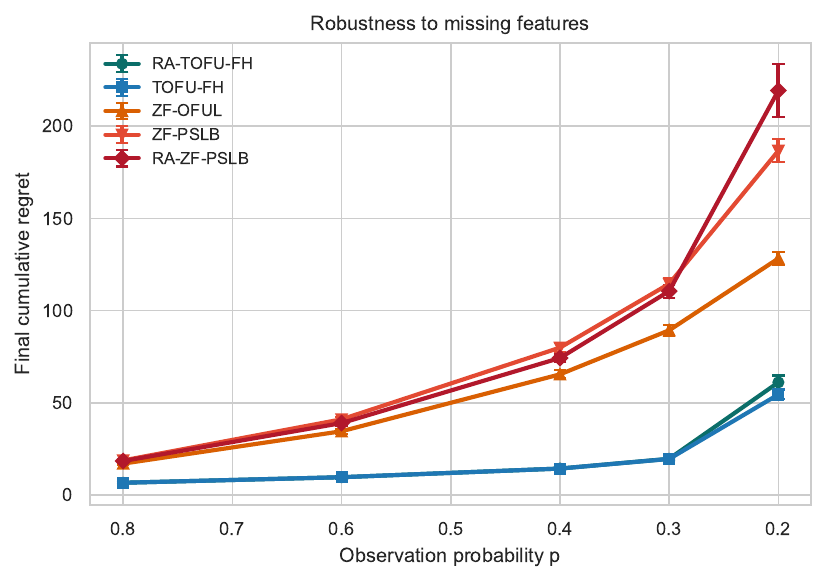}
        \caption{Final regret vs. \(p\).}
        \label{fig:app_synthetic_fh_missingness}
    \end{subfigure}
    \caption{Full-history synthetic counterpart. TOFU-FH and RA-TOFU-FH use the same subspace-estimation mechanism as TOFU and RA-TOFU, but replay previously observed rewards after each representation update.}
    \label{fig:app_synthetic_fh}
\end{figure*}

Figure~\ref{fig:app_real_feature_synthetic} reports a ``quasi-synthetic'' digit experiment. The raw covariates are optical digit features~\citep{alpaydin1998optical}, so the candidate arms are no longer drawn from the Gaussian latent model used in the main synthetic study. At the same time, we keep the reward geometry controlled: the arm latents are constructed from these real covariates, embedded into a rank-\(m^\star\) subspace, and then masked coordinatewise using the same Bernoulli observation model as in the theory. This lets us test whether the corrected low-rank mechanism remains useful when the feature distribution is less idealized, while still retaining a known reward-relevant subspace. The qualitative pattern matches the fully synthetic experiment. When \(p\) is large, zero-imputed OFUL is competitive; as \(p\) decreases, the ambient zero-filled representation becomes increasingly distorted, and TOFU separates from the baselines by exploiting the recovered low-rank structure.

\begin{figure}[!htb]
    \centering
    \begin{subfigure}[b]{0.48\textwidth}
        \includegraphics[width=\textwidth]{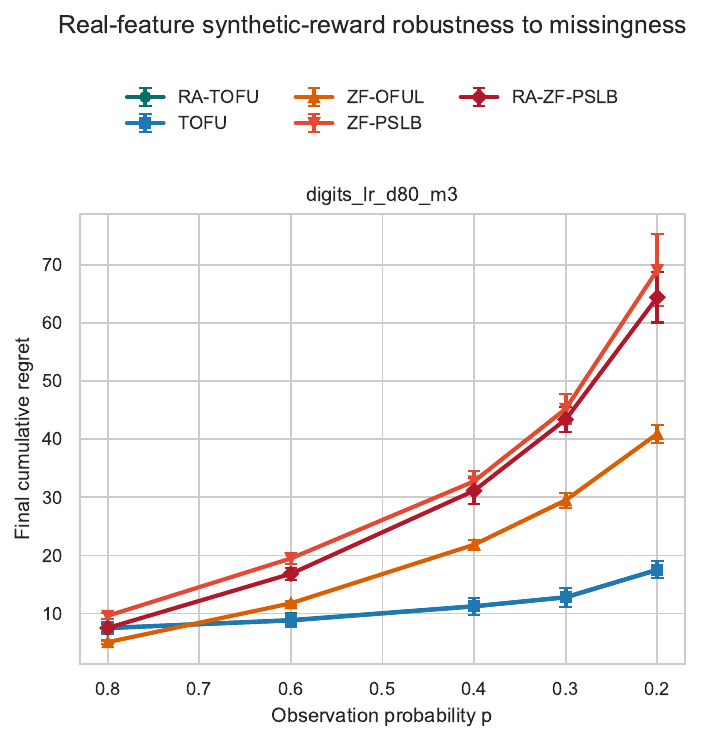}
        \caption{Final regret vs. \(p\).}
        \label{fig:app_real_feature_missingness}
    \end{subfigure}
    \hfill
    \begin{subfigure}[b]{0.9\textwidth}
        \includegraphics[width=\textwidth]{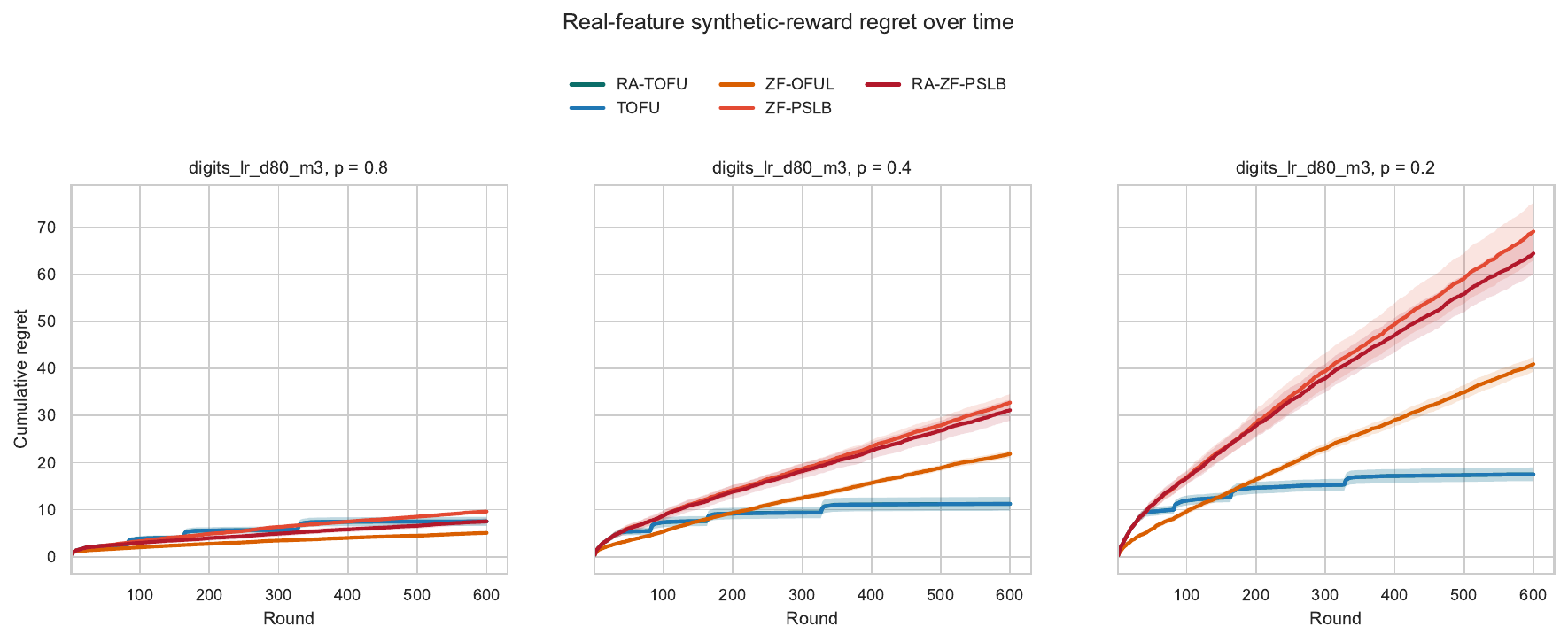}
        \caption{Regret trajectories.}
        \label{fig:app_real_feature_regret}
    \end{subfigure}
    \caption{Real-feature synthetic experiment. Real covariates replace Gaussian arms, but the reward-relevant geometry remains low-rank; TOFU benefits most when missingness is substantial.}
    \label{fig:app_real_feature_synthetic}
\end{figure}

Figures~\ref{fig:app_adaptive_rank_summary} and
\ref{fig:app_rank_and_warm_start} probe the two implementation choices that are
suppressed in the main synthetic figure: how the rank is selected, and how much
benefit comes from reusing past reward data. Figure~\ref{fig:app_adaptive_rank_summary}
shows the adaptive-rank full-history method together with the final selected
ranks, making visible whether the thresholding rule is stabilizing near the
intended dimension. Figure~\ref{fig:app_rank_and_warm_start} then separates the
diagnostics. The fixed-rank misspecification panel shows the cost of choosing a
rank below or above the true value; the rank-recovery panel checks that the
corrected-covariance spectrum contains a usable eigengap; and the warm-start
panel isolates the finite-sample gain from re-imputing previously selected arms
and replaying their rewards after the first learned subspace is formed.

\begin{figure}[!htb]
    \centering
    \includegraphics[width=0.72\textwidth]{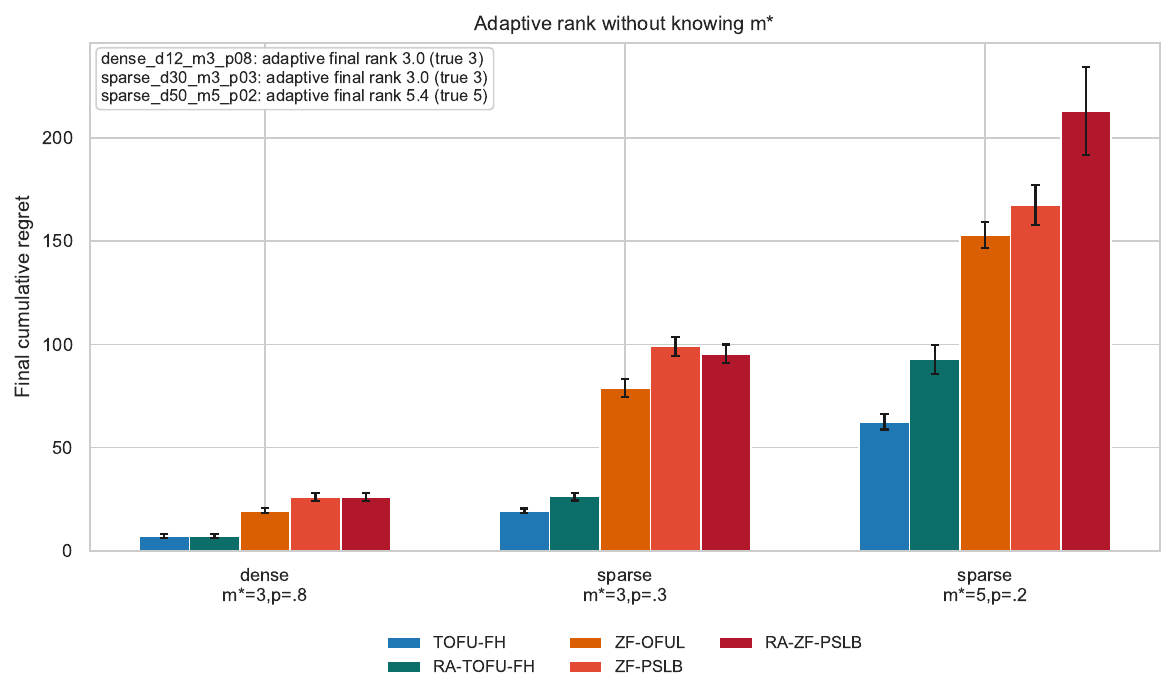}
    \caption{Synthetic adaptive-rank summary. The final selected ranks are shown here rather than in the main paper; RA-TOFU-FH tracks the known-rank TOFU-FH method closely when the eigenspectrum separates the signal directions from the null directions.}
    \label{fig:app_adaptive_rank_summary}
\end{figure}

\begin{figure*}[!htb]
    \centering
    \begin{subfigure}[b]{0.49\textwidth}
        \includegraphics[width=\textwidth]{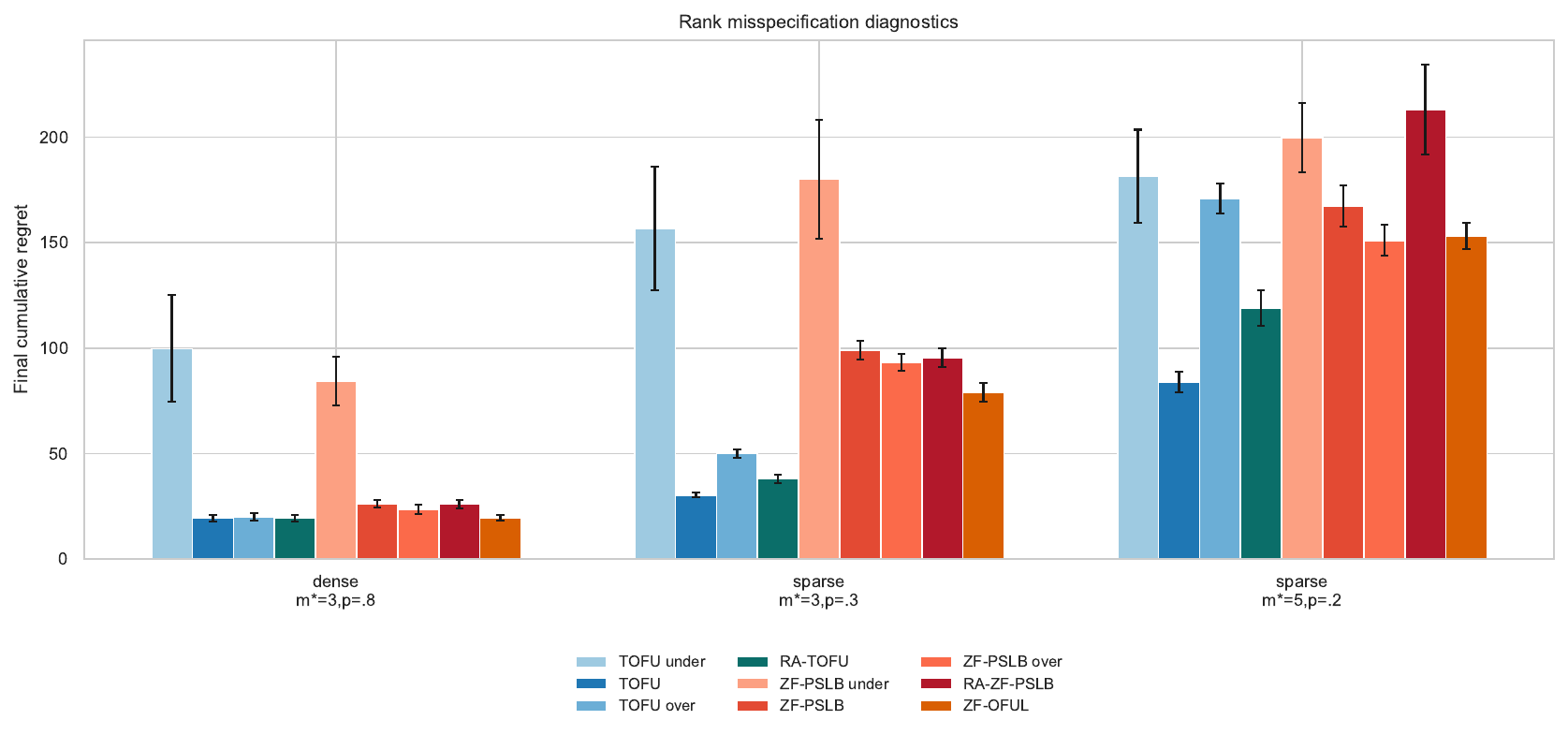}
        \caption{Fixed-rank misspecification.}
        \label{fig:app_rank_misspecification}
    \end{subfigure}
    \hfill
    \begin{subfigure}[b]{0.49\textwidth}
        \includegraphics[width=\textwidth]{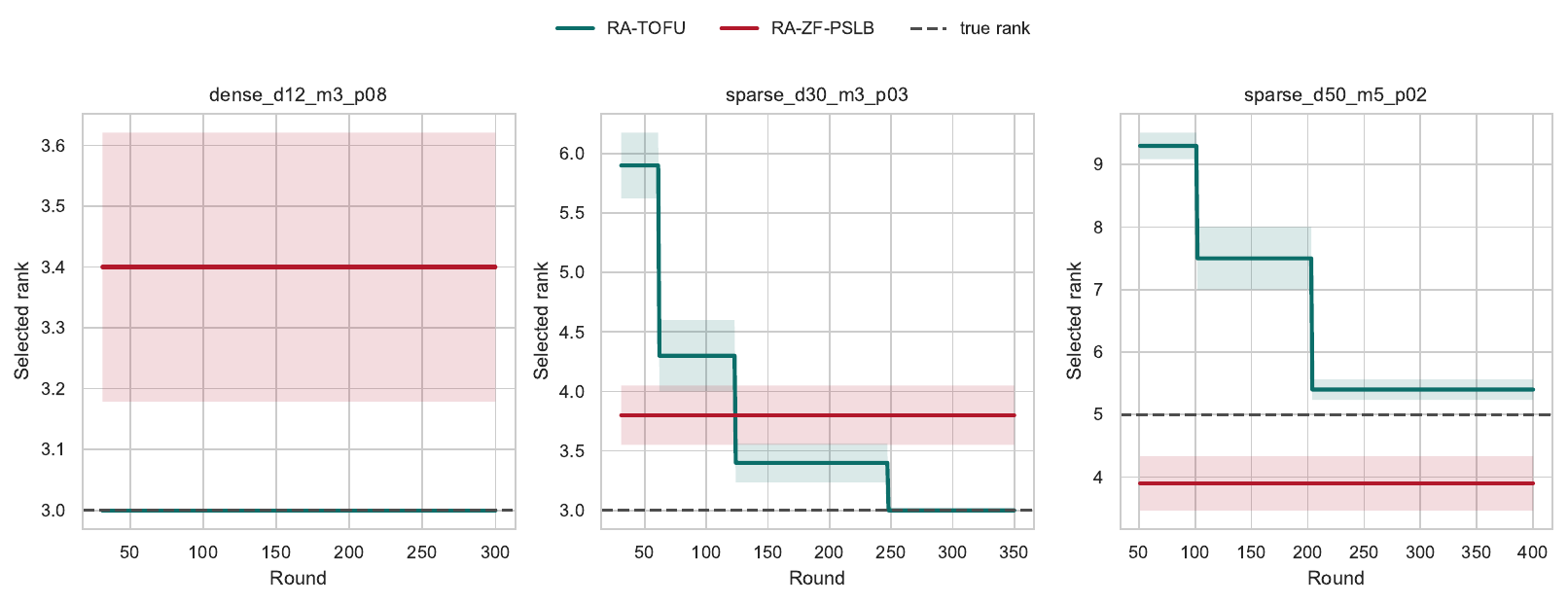}
        \caption{Rank recovery.}
        \label{fig:app_rank_recovery}
    \end{subfigure}
    \hfill
    \begin{subfigure}[b]{0.5\textwidth}
        \includegraphics[width=\textwidth]{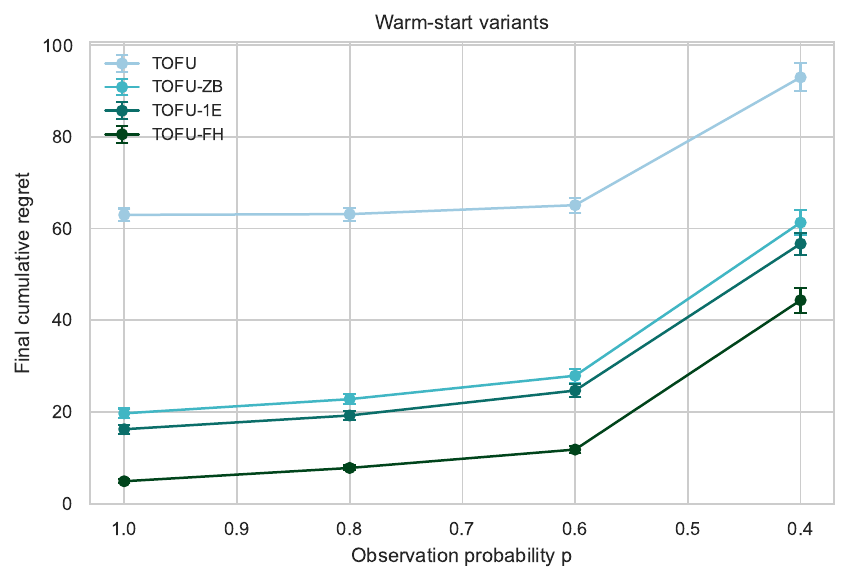}
        \caption{Warm-start variants.}
        \label{fig:app_warm_start}
    \end{subfigure}
    \caption{Synthetic diagnostics. Adaptive TOFU reduces sensitivity to rank choice, the corrected-covariance spectrum identifies the relevant rank, and warm-starting improves finite-sample performance.}
    \label{fig:app_rank_and_warm_start}
\end{figure*}

Figure~\ref{fig:app_cnn_diagnostics} gives additional diagnostics for the MNIST product-context experiment in Figure~\ref{fig:mnist_cnn_product_context}. This experiment starts from a supervised image model rather than a synthetic latent distribution. We train a small CNN on MNIST~\citep{lecun1998gradient}, freeze it, and use its \(m=4\)-dimensional penultimate representation \(h(x)\) together with the final classification head. If \(w_k\in\mathbb R^4\) is the class-\(k\) weight vector, the bandit arm for label \(k\) is the product context \(h(x)\odot w_k\), so the linear bandit reward preserves the classifier score through \(\langle h(x)\odot w_k,\mathbf 1\rangle=w_k^\top h(x)\). These ten class arms are then lifted into \(\mathbb R^{100}\) by a fixed orthonormal embedding and masked coordinatewise. Thus the experiment uses real image-derived representations, but the low-rank bandit geometry is known by construction. The diagnostics check that this construction is behaving as intended: the adaptive-rank estimates concentrate near the construction rank \(m=4\), and the fixed-rank validation sweep shows that the fixed-rank baselines used in the main comparison were chosen on held-out validation seeds rather than tuned on the reporting seeds.

\begin{figure}[!htb]
    \centering
    \begin{subfigure}[b]{0.48\textwidth}
        \includegraphics[width=\textwidth]{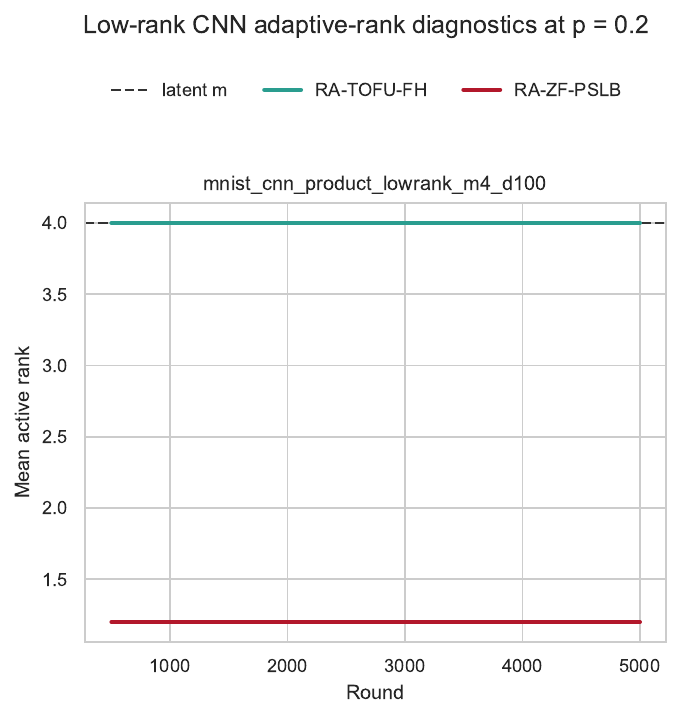}
        \caption{Adaptive-rank diagnostics.}
        \label{fig:app_cnn_rank_diagnostics}
    \end{subfigure}
    \hfill
    \begin{subfigure}[b]{0.48\textwidth}
        \includegraphics[width=\textwidth]{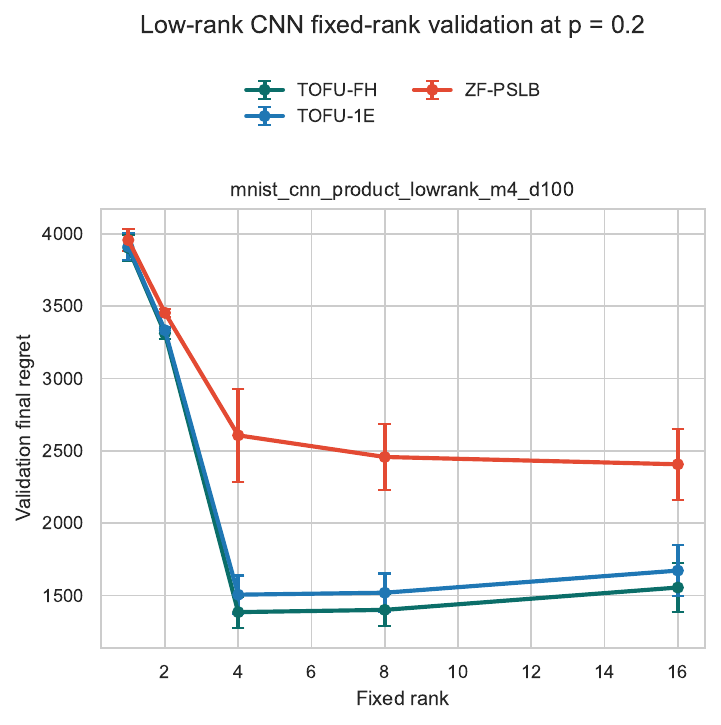}
        \caption{Fixed-rank validation.}
        \label{fig:app_cnn_fixed_rank_validation}
    \end{subfigure}
    \caption{MNIST product-context diagnostics. The rank behavior is consistent with the constructed low-rank feature map, and fixed-rank comparisons use validation choices disjoint from the reporting seeds.}
    \label{fig:app_cnn_diagnostics}
\end{figure}

\subsection{Text Product-Context Experiment}
\label{app:text_product_context}

As a second real-data problem, we construct a text product-context bandit from a four-class 20 Newsgroups classification task~\citep{lang1995newsweeder}. TF-IDF features are compressed by TruncatedSVD and fit with a no-intercept multinomial logistic-regression classifier. For document \(x\) and class \(k\), with document embedding \(h(x)\in\mathbb R^m\) and class weight \(w_k\in\mathbb R^m\), the bandit arm is the coordinatewise product context \(X_k(x)=h(x)\odot w_k\). This preserves the classifier score since \(\langle X_k(x),\mathbf 1_m\rangle=w_k^\top h(x)\). The low-dimensional product-context arms are lifted into ambient dimension \(d=1000\), and we add a reward-irrelevant orthogonal nuisance tail whose top empirical eigenvalue is \(0.25\) times the smallest retained signal eigenvalue. Thus the instance is approximately low-rank: the reward-relevant subspace is recoverable, but ambient methods must learn through many irrelevant masked coordinates.

The experiment uses \(m=20\), \(K=4\), horizon \(T=8000\), five reporting seeds, and observation probabilities \(p\in\{0.4,0.3,0.2\}\). The underlying text classifier has held-out accuracy about \(0.860\). Figure~\ref{fig:app_text_product_context} and Table~\ref{tab:text_product_context} show the same qualitative behavior as the image product-context experiment. Fixed-rank full-history \algname{} has the lowest final regret at all tested missingness levels, adaptive \algname{} remains close, and masked PSLB is worse, especially at \(p=0.2\). The gains over zero-imputed OFUL are smaller than in MNIST but consistent across the sweep, giving a second real-data modality in which the corrected low-rank representation helps under coordinate missingness.

\begin{figure*}[!htb]
    \centering
    \begin{subfigure}[b]{0.49\textwidth}
        \includegraphics[width=\textwidth]{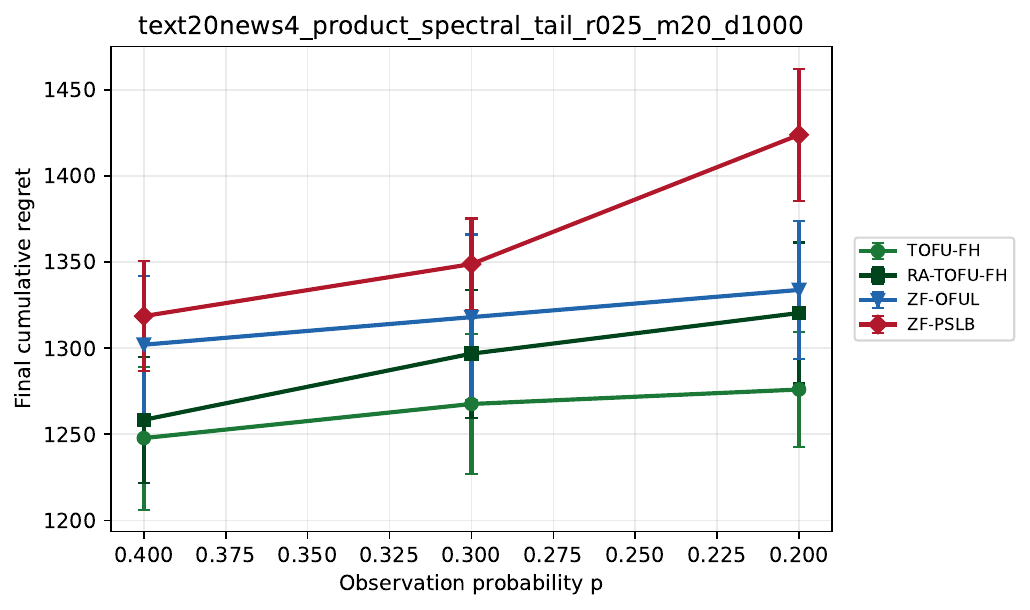}
        \caption{Final regret vs. \(p\).}
        \label{fig:app_text_product_missingness}
    \end{subfigure}
    \hfill
    \begin{subfigure}[b]{0.49\textwidth}
        \includegraphics[width=\textwidth]{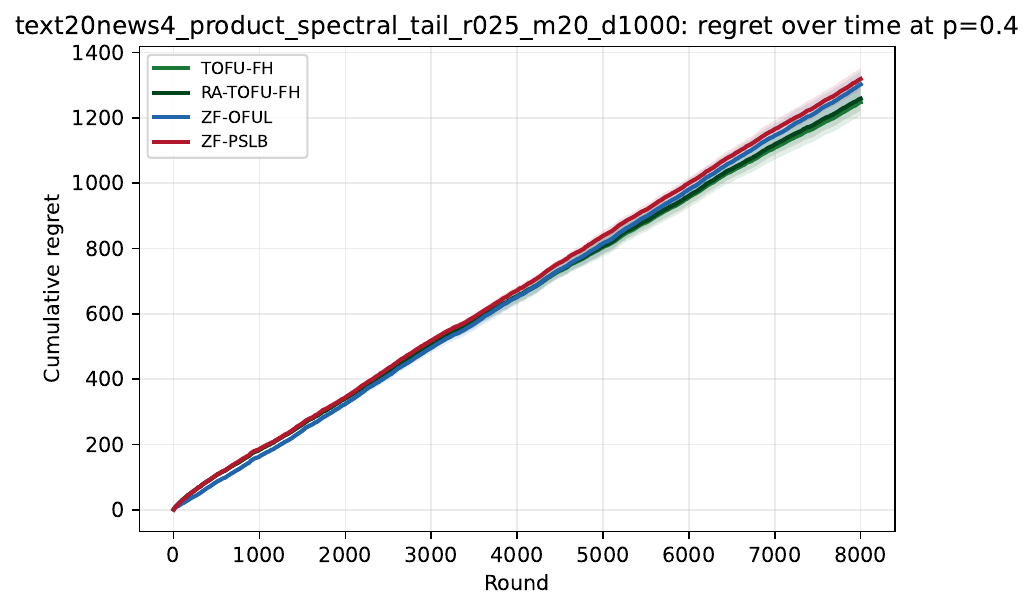}
        \caption{Regret trajectories at \(p=0.4\).}
        \label{fig:app_text_product_regret_p04}
    \end{subfigure}
    \hfill
    \begin{subfigure}[b]{0.52\textwidth}
        \includegraphics[width=\textwidth]{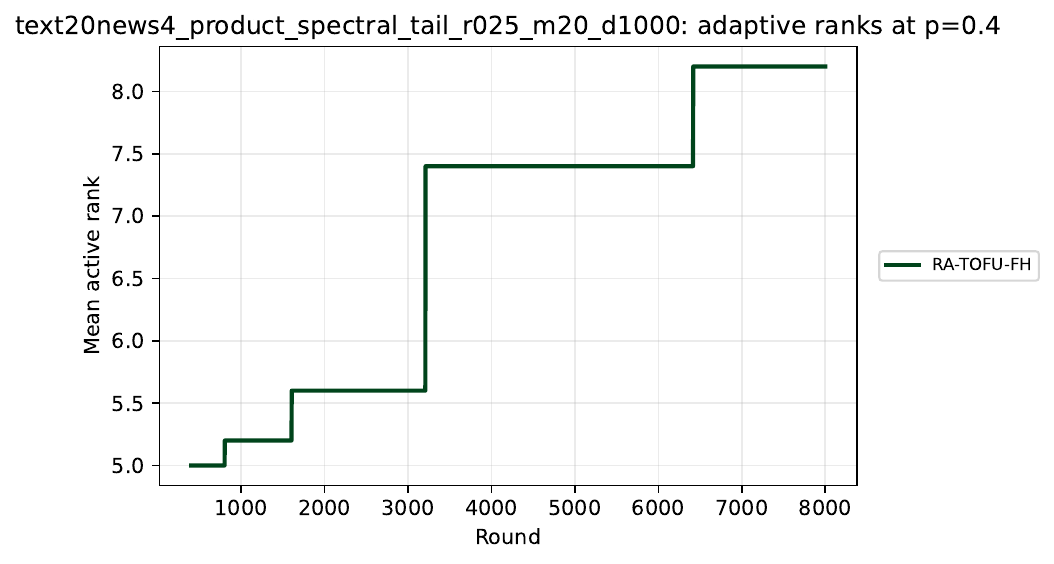}
        \caption{Adaptive-rank diagnostics.}
        \label{fig:app_text_product_rank}
    \end{subfigure}
    \caption{Text product-context experiment from four-class 20 Newsgroups. Fixed-rank full-history \algname{} is consistently best, while adaptive \algname{} remains competitive and masked PSLB is worse at heavier missingness.}
    \label{fig:app_text_product_context}
\end{figure*}

\begin{table}[t]
\centering
\small
\caption{Final cumulative regret in the text product-context experiment. Entries are mean \(\pm\) standard error over five reporting seeds.}
\label{tab:text_product_context}
\begin{tabular}{c c c c c}
\toprule
\(p\) &
Fixed-rank \algname{} &
Adaptive \algname{} &
Zero-imputed OFUL &
Masked PSLB \\
\midrule
\(0.4\) & \(1247.8\pm41.5\) & \(1258.4\pm36.6\) & \(1302.0\pm40.0\) & \(1318.6\pm31.9\) \\
\(0.3\) & \(1267.6\pm40.5\) & \(1296.8\pm37.2\) & \(1318.0\pm48.0\) & \(1348.8\pm26.4\) \\
\(0.2\) & \(1276.0\pm33.1\) & \(1320.4\pm40.9\) & \(1333.8\pm40.1\) & \(1423.8\pm38.5\) \\
\bottomrule
\end{tabular}
\end{table}

The rank diagnostic explains the small gap between fixed-rank and adaptive \algname{}. The fixed-rank methods use the construction rank \(m=20\), while adaptive \algname{} selects mean final ranks \(8.2,7.4,\) and \(5.6\) for \(p=0.4,0.3,\) and \(0.2\), respectively. Thus the adaptive selector is conservative on this approximately low-rank text instance, especially when missingness is heavier. Even with this lower selected rank, adaptive \algname{} remains close to the fixed-rank method and improves over the zero-imputed and masked-PSLB baselines, which is the main point of the diagnostic.

\subsection{Computational Resources}
The reported experiments were run on an Exxact TensorEX 2U rackmount machine with two AMD EPYC Rome 7542 processors (32 cores and 64 threads each), 1 TB DDR4 ECC memory, four NVIDIA A100 SXM4 GPUs with 40 GB memory each, a 2 TB NVMe OS drive, a 15.36 TB NVMe data drive, and Ubuntu 18.04. Some development and pilot runs were also executed on a 2021 Apple M1 Macbook with 64GB RAM, but the server configuration above is the conservative compute environment for reproducing the reported results. We do not report exact runtimes, since the experiments are small-scale validations of the theory rather than exhaustive compute benchmarks.